\documentclass{article}
\usepackage[margin=1in]{geometry}

\usepackage[utf8]{inputenc} 
\usepackage[T1]{fontenc}    
\usepackage[hidelinks]{hyperref}       
\usepackage{url}            
\usepackage{booktabs}       
\usepackage{amsfonts}       
\usepackage{nicefrac}       
\usepackage{microtype}      
\usepackage{xcolor}         
\usepackage{natbib}
\usepackage{graphicx}
\usepackage{mathrsfs}
\usepackage{amsfonts, amsmath,dsfont,amssymb,amsthm,stmaryrd,bbm}
\usepackage{mathtools}
\usepackage{caption}
\usepackage{subcaption}
\usepackage{enumitem}
\usepackage{wrapfig}

\usepackage{algorithm}
\usepackage{algorithmic}
\usepackage{float}

\usepackage[textsize=tiny]{todonotes}

\newtheorem{conj}{Conjecture}
\newtheorem{thm}{Theorem}

\newtheorem{rmk}{Remark}

\newtheorem{lemma}{Lemma}
\newtheorem*{lemmau}{Lemma}

\newtheorem{defn}[conj]{Definition}
\newtheorem{coro}{Corollary}

\newtheorem{assumption}{Assumption}

\newcommand{\f}[1]{\boldsymbol{#1}}
\newcommand{\bb}[1]{\mathbb{#1}}
\newcommand{\fl}[1]{\mathbf{#1}}
\newcommand{\ca}[1]{\mathcal{#1}}

\newcommand{\s}[1]{\mathsf{#1}}

\newcommand{\lr}[1]{\left|\left|#1\right|\right|}
\newcommand{\probuser}{P_{\s{N}}}

\newcommand{\remove}[1]{}

\title{Optimal Algorithms for Latent Bandits with Cluster Structure}

\author{
~~~Soumyabrata~Pal\footnote{S. Pal is with Google Research at Bangalore, INDIA (email: \texttt{soumyabrata@google.com}).}
~~~Arun~Sai~Suggala\footnote{A. Suggala is with Google Research at Bangalore, INDIA (email: \texttt{arunss@google.com}).}
~~~Karthikeyan~Shanmugam\footnote{K. Shanmugam is with Google Research at Bangalore, INDIA (email: \texttt{karthikeyanvs@google.com}).}
~~~Prateek Jain\footnote{P.Jain is with Google Research at Bangalore, INDIA (email: \texttt{prajain@google.com}).} }

\begin{document}

\maketitle
%

%

\begin{abstract}
We consider the problem of latent bandits with cluster structure where there are multiple users, each with an associated multi-armed bandit problem. These users are grouped into \emph{latent} clusters such that the mean reward vectors of users within the same cluster are identical. At each round, a user, selected uniformly at random, pulls an arm and observes a corresponding noisy reward. The goal of the users is to maximize their cumulative rewards. This problem is central to practical recommendation systems and has received wide attention of late \cite{gentile2014online, maillard2014latent}. Now, if each user acts independently, then they would have to explore each arm independently and a regret of $\Omega(\sqrt{\s{MNT}})$ is unavoidable, where  $\s{M}, \s{N}$ are the number of arms and users, respectively. Instead, we propose LATTICE (Latent bAndiTs via maTrIx ComplEtion) which allows exploitation of the latent cluster structure to provide the minimax optimal regret of $\widetilde{O}(\sqrt{(\s{M}+\s{N})\s{T}})$, when the number of clusters is $\Tilde{O}(1)$. This is the first algorithm to guarantee such strong regret bound. LATTICE is based on a careful exploitation of arm information within a cluster while simultaneously clustering users. Furthermore, it is  computationally efficient and requires only $O(\log{\s{T}})$ calls to an offline matrix completion oracle across all $\s{T}$ rounds. 
\end{abstract}

\section{INTRODUCTION}

Bandit optimization is a very general framework for sequential decision making when the dynamics of the underlying environment are unknown a priori. It has been well studied over the past few decades, and has shown great empirical success in areas including ad placement, clinical trials~\citep{lattimore2020bandit, mate2022field}. Such standard bandit methods often assume the decision maker has access to the entire user context or user state. However, this assumption rarely holds in practice. For example, in movie-recommendation scenarios, users can be clustered according to their taste in movies, but the observed features like user's demographic information might only be a noisy indicator of their taste. Similarly, in educational settings, the true cognitive state of a user is unknown. Instead we only get to observe a noisy estimate of the cognitive state through assessments. 
This shows that in practice, we need to optimize reward in the presence of observed as well as unobserved/latent context. Naturally, one option is to completely ignore the latent context as it can be thought of as part of the reward structure itself, but that generally leads to a significant increase in the sample complexity compared to the scenario when the latent structure is known a priori.

Several recent works have considered bandit optimization frameworks that explicitly factor in the latent data, and designed algorithms to maximize the cumulative rewards provided by the environment~\citep{maillard2014latent, hong2020latent, zhou2016latent}. However, as detailed below, even for simple latent structure like cluster of users, these papers either require strong assumptions or require additional side information, both of which are impractical. 

In this work, we consider the problem of multi-user multi-armed bandits with latent clusters (MAB-LC). This is a simple yet powerful setting that captures several practically important multi-user scenarios like recommendation systems, and was introduced in \cite{maillard2014latent}. Let there be $\s{N}$ users, $\s{M}$ arms and $\s{T}$ rounds ($\s{N}, \s{M} \approx 10^6$  in recommendation systems such as Youtube). The $\s{N}$ users can be partitioned into $\s{C}$ latent clusters where users in the same cluster have identical reward distributions; in other words, users in the same cluster have similar preferences for arms. In every round, one of the users, sampled uniformly at random, pulls one of the arms and obtains certain feedback. The goal of the decision maker is to maximize the cumulative reward of all the users. This problem was first introduced and studied in \cite{maillard2014latent} who provided theoretical guarantees for certain special settings (such as known cluster rewards or known cluster assignments) but not for the general problem. \cite{gentile2014online,li2016collaborative,li2019improved,gentile2017context,qi2022neural} considered a contextual bandit variant of MAB-LC, which is a generalization of our setting. But, most of the existing methods either provide sub-optimal regret bounds, or require strong assumptions on context vectors that might not hold in practice. \cite{hong2020latent} considered more general latent structures than the cluster structure we consider in this work, but required access to offline data for estimating the latent states. 
In another line of related work, \cite{jain2022online} studied the online low rank matrix completion  problem (a generalization of our setting). \cite{jain2022online} could only obtain minimax optimal regret bounds in the  special case of rank-$1$ setting. Obtaining optimal regret for the general rank-$r$ problem is still an open problem~\citep{jain2022online}. To summarize, while the MAB-LC problem  has been widely studied, to the best of our knowledge, designing an efficient method with nearly optimal regret bound is open. 



Before moving ahead, it is instructive to consider two hypothetical scenarios that illustrate the complexity of MAB-LC. If the cluster assignment (\emph{i.e.,} mapping between users and clusters) is known, we could have treated  users within a cluster as a single super-user and solved a separate multi-armed bandit for each super-user. This leads to a regret of $\widetilde{O}(\sqrt{\s{MCT}})$, which is minimax optimal. On the other hand, suppose the cluster assignment is unknown but the reward distributions of arms in a cluster are known. Then we could have played a separate multi-armed bandit problem for each user with the best arm from each cluster as a candidate arm. This leads to a regret of $\widetilde{O}(\sqrt{\s{NCT}})$, which is minimax optimal. However, in MAB-LC, both cluster assignment and reward distributions are unknown. Consequently, the \textit{reference} regret guarantee that one can hope to achieve is $O(\sqrt{(\s{M}+\s{N})\s{CT}})$. 

In this work, we propose a novel algorithm (LATTICE) for the problem of MAB-LC that achieves the above  reference regret bound. The key challenge  in solving this problem is to simultaneously cluster users, and quickly identify the optimal arms within each cluster. LATTICE addresses these problems using two key insights: (a) \textbf{(user clustering)} it uses low-rank matrix completion as an algorithmic tool to cluster users, and (b) \textbf{(arm elimination)} within each identified cluster, it discards sub-optimal arms by a careful exploitation of the  accrued arm information. LATTICE runs in phases and performs both user clustering and arm elimination in each phase. Computationally, our algorithm is efficient and requires only $O(\log{\s{T}})$ calls to an offline matrix completion oracle across all $\s{T}$ rounds. 
Furthermore, under certain incoherence assumptions on the user-arm matrix, we show that our algorithm achieves the minimax optimal regret of $\widetilde{O}(\sqrt{(\s{M}+\s{N})\s{T}})$, when $\s{C} = \widetilde{O}(1).$ 
We note the incoherence assumptions seems unavoidable for statistical recovery with partial observations. In addition to minimax optimal bounds, we also derive distribution-dependent regret bounds for our algorithm that inversely depend on the minimum gap between mean rewards of arms thus obtaining the optimal scaling with respect to gaps. 

We also consider a more general and practical setting where we relax the cluster definition as follows: (a) for any two users in the same cluster, we let their reward vectors be $\nu$ entry-wise close to each other, and (b) for any two users from different clusters, their respective best arm rewards are separated by more than $c \nu$, for some $c>1$. Note that, we don't require large separation between users in the mean rewards of sub-optimal arms across clusters.
We show that a modification of LATTICE obtains similar regret bounds as before in this general setting.

\subsection{Other Related Work}

\noindent \textbf{Contextual Bandits with Latent Structure.}
An extensive line of work \citep{gentile2014online,li2019improved,gentile2017context,li2016collaborative,qi2022neural} studies a variant of MAB-LC where every arm is associated with a context vector of dimension $d$ and expected arm reward for any user in a fixed cluster is a unknown linear function (depending only on the cluster and has unit norm) of the context vector. Importantly, in our setting, the arm contexts are not observed i.e. the context is hidden. In theory, one could apply the results in these works to MAB-LC by associating standard basis vectors to the arms and converting it into a contextual bandit problem. However, such a trivial conversion results in a  highly sub-optimal regret of $\tilde{O}(\sqrt{\s{M}^2\s{CT}} + \s{M}^3\s{N})$ due to a strong singular value assumption (see Appendix \ref{sec:comparison} for a detailed discussion). In other words,
the guarantees in this line of work is only useful when the dimension $d$ is much smaller than the number of arms. Furthermore, these papers also assume that the unknown parameter vectors corresponding to the clusters are significantly separated - this makes clustering easy with a few initial rounds. Importantly, our results/algorithm do not need such a condition - Assumption \ref{assum:matrix} allows the unknown cluster parameters to be as close as possible.
In a similar line of work, \cite{zhou2016latent} proposed an explore-then-commit style algorithm, but with sub-optimal regret bound, which in some cases is linear in $\s{T}$. 

\noindent \textbf{Online Low rank Matrix Completion (O-LRMC).} This is a more general problem than MAB-LC, but the existing results are significantly sub-optimal. As mentioned earlier, \cite{jain2022online}'s result applies to only rank-$1$ case. For the general rank-$r$ case, which corresponds to our $\s{C}$ clusters, the algorithm can be significantly suboptimal in terms of dependence on $\s{T}$. In a related work, \cite{sen2017contextual} studied an epsilon-greedy algorithm, and derived sub-optimal distribution-dependent bounds scaling inversely in the square of the gap between mean rewards.  In addition, their distribution-free regret bounds have sub-optimal dependence in $\s{T}$; $\s{T}^{2/3}$ instead of $\sqrt{\s{T}}$ provided by our method. \cite{dadkhahi2018alternating} provided an online alternating minimization heuristic for the general rank-$r$ problem, but do not provide any regret bounds. 
 In a separate line of work, \citep{kveton2017stochastic,trinh2020solving,katariya2017bernoulli,hao2020low,jun2019bilinear,huang2020online,lu2021low} study a similar low-rank reward matrix setting but they consider a significantly easier objective of identifying the largest entry in the entire reward matrix/tensor instead of finding the most rewarding arms for each user/agent. 
 
\noindent \textbf{Online Collaborative Filtering.} A closely related line of work studies the user-based online Collaborative Filtering (CF) \cite{Bresler:2014,Bresler:2016,Heckel:2017,Mina2019,huleihel2021learning}. 
These works study the MAB-LC problem with the additional constraint that the same arm cannot be pulled by any particular user more than once. While this model is strictly more restricted than MAB-LC, no theoretical bounds are known on the regret in this setting. Instead, these works minimize pseudo-regret: assuming the mean rewards or arms lie in $[0,1]$, these works aim to maximize the number of arms pulled with reward more than $1/2$. We note that this is a simpler metric than cumulative regret because maximizing the latter requires identifying the best arm, whereas maximizing the former only requires identifying arms with reward more than $1/2$. 

\section{PROBLEM FORMULATION}\label{sec:problem}

\noindent \textbf{Notations:} We write $[n]$ to denote the set $\{1,2,\dots,n\}$.
For a matrix $\fl{A}\in \bb{R}^{m \times n}$, we will write $\fl{A}_i,\fl{A}_{\mid i}$ to denote the $i^{\s{th}}$ row and column of matrix $\fl{A}$ respectively. We will write $\fl{A}_{ij}$ to denote the entry of the matrix $\fl{A}$ in the $i^{\s{th}}$ row and $j^{\s{th}}$ column. We will write $\lr{\fl{A}}_{\infty}=\max_{ij}\left|\fl{A}_{ij}\right|$ to denote the largest entry of the matrix $\fl{A}$. For a subset $\ca{S}\subseteq [m]$ of indices, we will write $\fl{A}_{\ca{S}}$ and $\fl{A}_{\mid \ca{S}}$ to denote the sub-matrix of $\fl{A}$ restricted to the rows in $\ca{S}$ and columns in $\ca{S}$ respectively. Extending the above notations, $\fl{A}_{i \mid \ca{S}}$ denotes the $i^{\s{th}}$ row of $\fl{A}$ restricted to the columns in $\ca{S}$; $\fl{A}_{\ca{S},\ca{S}'}$ denotes the sub-matrix of $\fl{A}$ restricted to the rows in $\ca{S}$ and columns in $\ca{S}'$. We write $\fl{e}_i$ to denote the $i^{\s{th}}$ standard basis vector that is zero everywhere except in the $i^{\s{th}}$ position where it has a \texttt{1}. $\widetilde{O}(\cdot)$ notation hides logarithmic factors.

 Consider a multi-user multi-armed bandit (MAB) problem where we have a set of $\s{M}$ arms (denoted by the set $[\s{M}]$), $\s{N}$ users (denoted by the set $[\s{N}]$) and $\s{T}$ rounds. In each round, a user $u(t)$ is sampled independently from a distribution $\probuser$ over $[\s{N}]$ (for much of the paper, we assume $\probuser$ is the uniform distribution). The sampled user $u(t)$ pulls an arm $\rho(t)$ from the set $[\s{M}]$ and receives a reward $\fl{R}^{(t)}$, s.t., 
\begin{align}\label{eq:obs}
    \fl{R}^{(t)} = \fl{P}_{u(t)\rho(t)}+\eta^{(t)}
\end{align}
where $\eta^{(t)}$ denotes the additive noise that is added to each observation. We will assume that the noise sequence $\{\eta^{(t)}\}_{t \in [\s{T}]}$ is composed of i.i.d zero-mean sub-Gaussian
random variables with variance proxy at most $\sigma^2>0$. Also, $\fl{P}\in \bb{R}^{\s{N}\times \s{M}}$ is the user-arm reward matrix. We study the MAB-LC problem under two assumptions on the reward matrix $\fl{P}$. 

 \noindent \textbf{Cluster Structure ($\s{CS}$).} 
Here, we assume the set of $\s{N}$ users can be partitioned into $\s{C}$ unknown clusters $\ca{C}^{(1)},\ca{C}^{(2)},\dots,\ca{C}^{(\s{C}) }$. Furthermore, 
in a particular cluster $\ca{C}^{(i)}$ for any $i\in [\s{C}]$, each user $u\in \ca{C}^{(i)}$ has an identical reward vector $\fl{P}_u$. Let $\fl{X}\in \bb{R}^{\s{C}\times \s{M}}$ be the sub-matrix of $\fl{P}$ that has the distinct rows of $\fl{P}$ corresponding to each cluster. 
$\tau:=\max_{i,j\in [\s{C}]} |\ca{C}^{(i)}|/ |\ca{C}^{(j)}|$ denotes the ratio of the maximum and minimum cluster size. For each user $u$, $\pi_u:[\s{M}]\rightarrow [\s{M}]$ denotes a permutation that sorts the arms in descending order of their reward for user $u$, i.e., $\fl{P}_{u\pi_u(i)}\ge \fl{P}_{u\pi_u(j)}$ for $i\le j$. Also, $\pi_u(1)\triangleq \s{argmax}_{j \in [\s{M}]} \fl{P}_{uj}$ is the arm with the highest reward for user $u\in [\s{N}]$.

\noindent \textbf{Relaxed Cluster Structure ($\s{RCS}$):} Here, we relax the cluster definition so that the users in the same cluster might not have identical reward vectors. That is, the assumption is that the set of $\s{N}$ users can be partitioned into $\s{C}$ clusters $\ca{C}^{(1)},\ca{C}^{(2)},\dots,\ca{C}^{(\s{C})}$ s.t. the following holds for some known $\nu>0$: 1) For any two users $u,v$ in the {\em same cluster}, $\pi_u(1)=\pi_v(1)$ and $\left|\left|\fl{P}_{u}-\fl{P}_{v}\right|\right|_{\infty} \le \nu$, 2) For any two users $u,v$ in different clusters, we will have either
    $\left|\fl{P}_{u\pi_u(1)}-\fl{P}_{v\pi_u(1)}\right| > 20\nu$ or $\left|\fl{P}_{u\pi_v(1)}-\fl{P}_{v\pi_v(1)}\right| > 20\nu$. Note that the $\s{CS}$ structure is a special case of $\s{RCS}$ with $\nu=0$.  

\begin{rmk}
The constant $20$ in $\s{RCS}$ model formulation is arbitrary and can be replaced by any constant $>1$.
\end{rmk}

Thus, in the $\s{RCS}$ model formulation, users in the same cluster have the same best arm and the reward vectors are entry-wise close; users in different cluster have rewards corresponding to one of the best arms to be separated.

Now, the goal is to minimize the regret assuming either $\s{CS}$ or $\s{RCS}$ structure on the reward matrix $\fl{P}$: 
\begin{align}
\s{Reg}(\s{T})\triangleq \bb{E} \Big(\sum_{t\in [\s{T}]}\fl{P}_{u(t)\pi_{u(t)}(1)}- \sum_{t\in[\s{T}]}\fl{P}_{u(t)\rho(t)}\Big).\label{eqn:regretExpect}
\end{align}
Here the expectation is over the randomness in the algorithm and the sampled users. 

\begin{rmk}\label{rmk:trivial}
Note that a trivial approach is to treat each user as a separate multi-armed bandit problem. Such a strategy does not utilize the low rank structure and leads to a regret of $O(\sqrt{\s{MNT}})$ assuming $\lr{\fl{P}}_{\infty},\sigma=O(1)$. Another trivial approach is to recommend random arms to each user (exploration) and subsequently use low rank matrix completion guarantees to estimate $\fl{P}$ and exploit. This will lead to a regret guarantee of $O(\s{(M+N)}^{1/3}\s{T}^{2/3})$ \citep{jain2022online}.
The goal is to obtain a significantly smaller regret guarantee of $O(\sqrt{\s{(M+N)T}})$ with $\s{C}=O(1)$. 
\end{rmk}

\section{PRELIMINARIES}\label{sec:prelims}\label{subsec:mc}



As mentioned earlier, the key algorithmic tool that we use is low rank matrix completion - a statistical estimation problem where the goal is to recover a low rank matrix from partially observed randomly sampled entries. Since the reward matrix $\fl{P}$ in MAB-LC is low rank, our strategy is to call the offline matrix completion oracle for relevant sub-matrices of $\fl{P}$ after we accumulate a sufficient number of random observations in each sub-matrix. In this work, we are interested in low rank matrix completion with non-trivial entry-wise guarantees that has been studied recently in \cite{chen2019noisy,abbe2020entrywise}. Below, we state a low rank matrix completion result that is adapted from \cite{jain2022online} which is in turn obtained with minor modifications from \cite{chen2019noisy}[Theorem 1] and is more suited to our setting:

\begin{lemma}[Lemma 2  in \cite{jain2022online}]\label{lem:min_acc}
Consider rank $\s{C}=O(1)$ reward matrix $\fl{P}\in \bb{R}^{\s{N} \times \s{M}}$ with SVD decomposition $\fl{P}=\fl{\bar{U}}\f{\Sigma}\fl{\bar{V}}^{\s{T}}$ satisfying $\|\fl{\bar{U}}\|_{2,\infty}\le \sqrt{\mu \s{C}/\s{N}}, \|\fl{\bar{V}}\|_{2,\infty}\le \sqrt{\mu \s{C}/\s{M}}$ and condition number $\kappa = O(1)$. Let $d_1=\max(\s{N},\s{M})$, $d_2=\min(\s{N},\s{M})$, and let $p$ be such that  
 $1 \ge p \ge c\mu^2d_2^{-1} \log^3 d_1$ for some constant $c>0$. 
For any positive integer $s>0$ satisfying $\frac{\sigma}{\sqrt{s}}=O\Big(\sqrt{\frac{pd_2}{\mu^3\log d_2}}\|\fl{P}\|_{\infty}\Big)$, Algorithm \ref{algo:estimate} with input $s,p,\sigma$ that uses  $m=\widetilde{O}\Big(sp\s{MN}\Big)$ observations to output $\widehat{\fl{P}}\in \bb{R}^{\s{N} \times \s{M}}$ for which, with probability at least $1-\widetilde{O}(\delta)$, we have 
\begin{align*}
\tiny
   \| \fl{P} - \widehat{\fl{P}}\|_{\infty} = O\left(\frac{\sigma  \sqrt{\mu^3\log d_1}}{\sqrt{spd_2}}\right)=\widetilde{O}\left(\frac{\sigma\sqrt{\s{MN}\mu^3}}{\sqrt{md_2}}\right).
\end{align*}
\end{lemma}




We now introduce a definition characterizing a \textit{nice} subset of users that we often use in the analysis.

\begin{defn}
A subset of users $\ca{S}\subseteq [\s{N}]$ will be called ``nice" if $\ca{S} \equiv \bigcup_{j \in \ca{A}} \ca{C}^{(j)}$ for some $\ca{A}\subseteq [\s{C}]$. In other words, $\ca{S}$ can be represented as the union of some subset of clusters.  
\end{defn}

\section{LATTICE ALGORITHM FOR $\s{CS}$}\label{sec:lattice}

\subsection{Algorithm and Proof Overview}



LATTICE runs in phases of exponentially increasing length.
In each phase, the goal is to divide the set of users into {\em nice} subsets. 
Moreover, for each subset of users, we have an active subset of arms that must contain the best arm for all users in the corresponding subset.   
That is, at the start of $\ell^{\s{th}}$ phase, we aim to create a list (of size $a_{\ell}\le \s{C}$) of \textit{nice} subsets of users $\ca{M}^{(\ell)}\equiv\{\ca{M}^{(\ell,i)}\}_{i \in [a_{\ell}]}$ and the corresponding subsets of arms $\ca{N}^{(\ell)}\equiv\{\ca{N}^{(\ell,i)}\}_{i\in [a_{\ell}]}$, s.t.  $\cup_{i \in [a_{\ell}]} \ca{M}^{(\ell,i)}= [\s{N}]$,   
      $\ca{N}^{(\ell,i)} \supseteq \{\pi_u(1) \mid u \in \ca{M}^{(\ell,i)}\}$, and 
\begin{align}\label{eq:gap_main}
    \left|\fl{P}_{u\pi_u(1)}-\min_{j\in \ca{N}^{(\ell,i)}}\fl{P}_{uj}\right| \le \epsilon_{\ell},\  \forall\ u\in \ca{M}^{(\ell,i)}. 
\end{align}
Above, $\epsilon_{\ell}$ is a fixed exponentially decreasing sequence in $\ell$. As we eliminate arms at each phase, the number of user subsets with more than $\gamma\s{C}$ active arms goes on shrinking with each phase. 
Since LATTICE is random, we define event $\ca{E}^{(\ell)}$ to be true if LATTICE maintains a list of user subsets and arm subsets satisfying the above properties.

Now, in round $t$, the sampled user $u(t)$ pulls arm $\rho(t)$ where $\rho(t)$ is sampled from $\ca{N}^{(\ell, i)}$ assuming $|\ca{N}^{(\ell, i)}|\geq \gamma \s{C}$, where $i$ is the index of subset $\ca{M}^{(\ell,i)}$ to which $u(t)$ belongs. If, $|\ca{N}^{(\ell, i)}|< \gamma \s{C}$, then the cluster  structure is ignored and arm $\rho(t)$ is selected from the active set of arms ($\ca{N}^{(\ell, i)}$) as determined by the Upper Confidence Bound (UCB) algorithm. Conditioned on $\ca{E}^{(\ell)}$ being true, our goal is to ensure $\ca{E}^{(\ell+1)}$ with high probability. Due to the arm pull strategy described above, for each subset of users $\ca{M}^{(\ell,i)}$ in $\ca{M}^{(\ell)}$ and their corresponding subset of active arms $\ca{N}^{(\ell,i)}$ (such that $
|\ca{N}^{(\ell,i)}| \ge \gamma\s{C}$) we observe random noisy entries from the sub-matrix $\fl{P}_{\ca{M}^{(\ell,i)},\ca{N}^{(\ell,i)}}$. Subsequently, we use low rank matrix completion  (Step 4 in Alg. \ref{algo:phased_elim}) and Lemma \ref{lem:min_acc} to obtain $ \widetilde{\fl{P}}\in \bb{R}^{\s{N}\times \s{M}}$ such that $\widetilde{\fl{P}}_{\ca{M}^{(\ell,i)},\ca{N}^{(\ell,i)}}$ is an entry-wise good estimate of $\fl{P}_{\ca{M}^{(\ell,i)},\ca{N}^{(\ell,i)}}$, i.e., 
\begin{align}\label{eq:entrywise}
    \lr{\widetilde{\fl{P}}_{\ca{M}^{(\ell,i)},\ca{N}^{(\ell,i)}}-\fl{P}_{\ca{M}^{(\ell,i)},\ca{N}^{(\ell,i)}}}_{\infty} \le \Delta_{\ell+1}
\end{align}
with high probability where $\Delta_{\ell+1}=\epsilon_{\ell+1}/40\s{C}$. We define the event $\ca{E}_2^{(\ell)}$ to be true if eq. (\ref{eq:entrywise}) is satisfied for all relevant sub-matrices. 

Next, conditioning on $\ca{E}_2^{(\ell)}$, consider a subset of users $\ca{M}^{(\ell,i)}$ for which the corresponding subset of active arms is large $|\ca{N}^{(\ell,i)}|\ge \gamma\s{C}$.  For next phase, the intuitive goal is to further partition $\ca{M}^{(\ell,i)}$ into subsets $\{\ca{M}^{(\ell,i,j)}\}_{j}$, each of which is \textit{nice} and find a list of corresponding subsets of active arms $\{\ca{N}^{(\ell,i,j)}\}_{j} \subseteq \ca{N}^{(\ell,i)}$ such that all arms in $\ca{N}^{(\ell,i,j)}$ have high reward (as in eq. \ref{eq:gap_main}) for all users in $\ca{M}^{(\ell,i,j)}$. To do so, for each user in the set $\ca{M}^{(\ell,i)}$, we find a subset of \textit{good arms} among the active arms  $\ca{T}_u^{(\ell)}\subseteq \ca{N}^{(\ell,i)}$ such that 
\small
\begin{align}\label{eq:good_arms}
    \ca{T}^{(\ell)}_u \equiv \{j \in \ca{N}^{(\ell,i)} \mid \max_{j'\in \ca{N}^{(\ell,i)}}\widetilde{\fl{P}}^{(\ell)}_{uj'}-\widetilde{\fl{P}}^{(\ell)}_{uj} \le 2\Delta_{\ell+1}\}
\end{align}
\normalsize
i.e. arms which have a high estimated reward for user $u$.

Subsequently, we design a graph whose nodes are users in $\ca{M}^{(\ell,i)}$ and an edge is drawn between two users $u,v\in \ca{M}^{(\ell,i)}$ if the following conditions are satisfied:
\small
\begin{align}\label{eq:edge}
     &\ca{T}^{(\ell)}_u \cap \ca{T}^{(\ell)}_v \neq \Phi, \left|\fl{\widetilde{P}}^{(\ell)}_{ux}-\fl{\widetilde{P}}^{(\ell)}_{vx}\right|\le 2\Delta_{\ell+1} \forall x\in \ca{N}^{(\ell,i)}
\end{align}
\normalsize
In other words, eq. (\ref{eq:edge}) defines an edge between two users in the same subset if reward estimates of active arms for the two users are close; secondly, there are common arms in their respective set of \textit{good arms} as defined in eq. (\ref{eq:good_arms}). We partition the set of users $\ca{M}^{(\ell,i)}$ into smaller \textit{nice} sets $\{\ca{M}^{(\ell,i,j)}\}$  by considering the connected components of the aforementioned graph and for users in each component $\ca{M}^{(\ell,i,j)}$, the updated trimmed common subset of arms
\begin{align}\label{eq:common_set}
\ca{N}^{(\ell,i,j)}\equiv \bigcup_{u \in \ca{M}^{(\ell,i,j)}} \ca{T}_u^{(\ell)}    
\end{align}
 with high reward is the union of set of good arms for all users in the connected component (see Step 8 in Alg. \ref{algo:phased_elim}). 
 We can show the following crucial and interesting lemma:
 \begin{lemmau}
Fix any $i\in [a_{\ell}]$ such that $\left|\ca{N}^{(\ell,i)}\right|\ge \gamma\s{C}$. Consider two users $u,v\in \ca{M}^{(\ell,i)}$ having a path in the constructed graph. Conditioned on $\ca{E}^{(\ell)},\ca{E}_2^{(\ell)}$, we have 
\begin{align*}
&\max_{x\in \ca{T}^{(\ell)}_u,y\in \ca{T}^{(\ell)}_v}\left|\fl{P}_{ux}-\fl{P}_{uy}\right|\le 32\s{C}\Delta_{\ell+1} \\
&\text{ and } \max_{x\in \ca{T}^{(\ell)}_u,y\in \ca{T}^{(\ell)}_v}\left|\fl{P}_{vx}-\fl{P}_{vy}\right|\le 32\s{C}\Delta_{\ell+1}.   
\end{align*}
\end{lemmau}
This lemma shows that good arms for one user is good for another if they are connected by a path. If the number of active arms for a subset of users become less than $\gamma\s{C}$, then we start  UCB (Upper Confidence Bound \cite{lattimore2020bandit}[Ch. 7]) for each user in that subset with the active arms until end of algorithm. Therefore, conditioned on $\ca{E}^{(\ell)},\ca{E}^{(\ell)}_2$, the event $\ca{E}^{(\ell+1)}$ is  true w.h.p. Hence, conditioned on $\ca{E}^{(\ell)}$, we can bound the regret in each round of the $\ell^{\s{th}}$ phase by $\epsilon_{\ell}$; roughly speaking, the number of rounds in the $\ell^{\s{th}}$ phase is  $1/\epsilon^2_{\ell}$ and therefore the regret is $1/\epsilon_{\ell}$. By setting $\Delta_{\ell}$ as in Step 3 of Alg. \ref{algo:phased_elim} (and $\epsilon_{\ell}=\Delta_{\ell}/64\s{C}$), we can bound the regret of LATTICE and achieve the guarantee in Theorems \ref{thm:main_LBM} and \ref{thm:main_LBM2}.
\begin{rmk}
The low rank matrix completion oracle is obtained with slight modifications from \cite{jain2022online}. In line 6 of Algorithm~\ref{algo:phased_elim}, we require the matrix completion oracle (Algorithm~\ref{algo:estimate}) to be stateful (i.e., we want it to wait until appropriate data arrives).  This is because vanilla low rank matrix completion results work under the assumption of Bernoulli sampling i.e. each entry in the matrix is observed once with some probability $p$. So to mimic Bernoulli sampling, we require a stateful version of the algorithm. We create a Bernoulli mask at the beginning of the phase and  pull arms to observe only the masked entries in sequence. We also make multiple observations corresponding to the same mask and take the average in each of the mask indices to reduce the variance; similarly, we also compute several estimates of the same matrix with independently sampled mask and take an entry-wise median to reduce error probability. Using these tricks appropriately (see Appendix \ref{app:mc} for a detailed proof and discussion) can allow us to obtain the guarantee in Lemma \ref{lem:min_acc}.  
In between subsequent arrivals of users belonging to a cluster, the algorithm remains stateful and waits at line 8.
\end{rmk}
 
\subsection{Theoretical Guarantees}\label{subsec:theory}

\begin{algorithm*}[t]
\caption{LATTICE (Latent bAndiTs via maTrIx ComplEtion )
\label{algo:phased_elim}}
\begin{algorithmic}[1]
\REQUIRE Number of users $\s{N}$, arms $\s{M}$, clusters $\s{C}$, rounds $\s{T}$, noise $\sigma^2>0$,  $\gamma\ge 1$ and constant $C'\ge 0$. 
\STATE Set $\ca{M}^{(1)} \equiv [\ca{M}^{(1,1)}]$ where $\ca{M}^{(1,1)}=[\s{N}]$ and $\ca{N}^{(1)} \equiv [\ca{N}^{(1,1)}]$ where $\ca{N}^{(1,1)}=[\s{M}]$. Set round index $t$ to be a global parameter; initialize $t=0$. 

\FOR{$\ell=1,2,\dots,$}

\STATE Set $\Delta_{\ell+1}=C'2^{-\ell}$ for some appropriate $C'>0$. Initialize $\ca{M}^{(\ell+1)}=[]$ and $\ca{N}^{(\ell+1)}=[]$ to be empty lists.

\STATE Collect data by running Alg. \ref{algo:compute_estimate} i.e. Alg. \textsc{Collect Data and Compute Estimate}($\s{C},\gamma,\sigma^2,\Delta_{\ell+1},\ca{M}^{(\ell)},\ca{N}^{(\ell)},\s{T}$) for the $\ell^{\s{th}}$ phase. Subsequently, for $\ca{M}^{(\ell,i)}\in \ca{M}^{(\ell)}$ such that $\left|\ca{N}^{(\ell,i)}\right|\ge \gamma\s{C}$, compute an estimate $\widetilde{\fl{P}}_{\ca{M}^{(\ell,i)},\ca{N}^{(\ell,i)}}$ of matrix $\fl{P}_{\ca{M}^{(\ell,i)},\ca{N}^{(\ell,i)}}$ from the data such that with high probability $\lr{\widetilde{\fl{P}}_{\ca{M}^{(\ell,i)},\ca{N}^{(\ell,i)}}-\fl{P}_{\ca{M}^{(\ell,i)},\ca{N}^{(\ell,i)}}}_{\infty} \le \Delta_{\ell+1}$.

\FOR {$i: | \ca{N}^{(\ell,i)} | \ge \gamma \s{C}$}

\STATE For every user $u\in \ca{M}^{(\ell,i)}$, compute $\ca{T}^{(\ell)}_u \equiv \{j \in \ca{N}^{(\ell,i)} \mid \max_{j'\in \ca{N}^{(\ell,i)}}\widetilde{\fl{P}}^{(\ell)}_{uj'}-\widetilde{\fl{P}}^{(\ell)}_{uj} \le 2\Delta_{\ell+1}\}$.

\STATE Construct graph $\ca{G}^{(\ell,i)}$ whose nodes are users in $\ca{M}^{(\ell,i)}$ and an edge exists between two users $u,v \in \ca{M}^{(\ell,i)}$ if $\ca{T}^{(\ell)}_u \cap \ca{T}^{(\ell)}_v \neq \Phi$ and $\left|\fl{\widetilde{P}}^{(\ell)}_{ux}-\fl{\widetilde{P}}^{(\ell)}_{vx}\right|\le 2\Delta_{\ell+1}$ for all arms $x\in \ca{N}^{(\ell,i)}$.

\STATE For each connected component $\ca{M}^{(\ell,i,j)}$ ($\cup_j \ca{M}^{(\ell,i,j)}\equiv \ca{M}^{(\ell,i)}$), compute $\ca{N}^{(\ell,i,j)}\equiv \cup_{u \in \ca{M}^{(\ell,i,j)}}\ca{T}_u^{(\ell)}$. Append $\ca{M}^{(\ell,i,j)}$ into $\ca{M}^{(\ell+1)}$ and $\ca{N}^{(\ell,i,j)}$ into $\ca{N}^{(\ell+1)}$.  
\ENDFOR

\STATE For each pair of sets $(\ca{M}^{(\ell,i)},\ca{N}^{(\ell,i)})$ such that $|\ca{N}^{(\ell,i)}|\le \gamma\s{C}$
, append $\ca{M}^{(\ell,i)}$ to $\ca{M}^{(\ell+1)}$ and $\ca{N}^{(\ell,i)}$ to $\ca{N}^{(\ell+1)}$.


\ENDFOR

\end{algorithmic}
\end{algorithm*}

\begin{algorithm*}[t]
\caption{\textsc{Collect Data and Compute Estimate}   \label{algo:compute_estimate}}
\begin{algorithmic}[1]

\REQUIRE Number of clusters $\s{C}$, parameter $\gamma>0$, noise variance $\sigma^2>0$, desired error guarantee $\Delta_{\ell+1}>0$, partition of users $\ca{M}^{(\ell)}$ comprising \textit{nice} subsets of users, corresponding list of sets of arms $\ca{N}^{(\ell)}$, rounds $\s{T}$.

\FOR{the $i^{\s{th}}$ set of users $\ca{M}^{(\ell,i)}\in \ca{M}^{(\ell)}$ and $i^{\s{th}}$ set of arms $\ca{N}^{(\ell,i)}\in \ca{N}^{(\ell)}$}
\IF{$\ca{N}^{(\ell,i)}\ge \gamma\s{C}$}
\STATE Create a stateful instance of Algorithm \textsc{Low Rank Matrix Estimate} with parameters - \{ $\ca{V}=\ca{N}^{(\ell,i)}$ for users, $\ca{U}=\ca{M}^{(\ell,i)}$ (whenever a user $u\in \ca{U}$ is sampled to pull an arm), desired error $\Delta_{\ell+1}$ in estimate, noise $\sigma^2$ and total rounds $\s{T}$\}. 
\COMMENT{All instances of Algorithm \ref{algo:estimate} will wait until entire data is collected (Line 5 in Alg. \ref{algo:data_collection_matrix_estimation}) or finish.}
\ELSE
 \STATE For \textit{each} $u \in \ca{M}^{(\ell,i)}$, if not created already, create a stateful instance of Algorithm \textsc{UCB}$(u)$  with arm set $\ca{N}^{(\ell,i)}$ 
\ENDIF
\ENDFOR

\WHILE {There exists an instance of Algorithm \ref{algo:estimate} that waits (i.e. Alg. \ref{algo:data_collection_matrix_estimation} invoked by Alg. \ref{algo:estimate} waits at Line $5$.)}
\STATE Sample a user $u(t)$ from the environment. Determine $i_t \leftarrow \{i: u(t) \in \ca{M}^{(\ell,i)}\}$.

\IF {Algorithm \ref{algo:estimate}'s instance with $\ca{U} = \ca{M}^{(\ell,i_t)} $ is still waiting}
 \STATE Execute one step of the \textsc{DataCollectionSubRoutine} (Line 5 of Alg. \ref{algo:data_collection_matrix_estimation} with $u(t)$) invoked by Algorithm \ref{algo:estimate}'s instance (with $\ca{U} = \ca{M}^{(\ell,i_t)}$) 
\ELSIF {Instance of Algorithm \textsc{UCB($u(t)$)} (i.e. Alg. \ref{algo:estimate_ucb}) exists}
 \STATE Invoke one time step of Algorithm  \textsc{UCB($u(t)$)}.
 \ELSE
 \STATE pull a random arm from $\ca{N}^{(\ell,i_t)}$.
\ENDIF
\STATE $t \leftarrow t+1$. Stop the Algorithm when $t=\s{T}$.
\ENDWHILE

\FOR{the $i^{\s{th}}$ set of users $\ca{M}^{(\ell,i)}\in \ca{M}^{(\ell)}$ and $i^{\s{th}}$ set of arms $\ca{N}^{(\ell,i)}\in \ca{N}^{(\ell)}$}
\IF{$\ca{N}^{(\ell,i)}\ge \gamma\s{C}$}
\STATE Store estimate $\widetilde{\fl{P}}_{\ca{M}^{(\ell,i)},\ca{N}^{(\ell,i)}}$ obtained as output from Algorithm \textsc{Low Rank Matrix Estimate} for users $\ca{U}=\ca{M}^{(\ell,i)}$ and arms $\ca{V}=\ca{N}^{(\ell,i)}$. 
\ENDIF

\ENDFOR
\STATE Return all stored estimates $\widetilde{\fl{P}}_{\ca{M}^{(\ell,i)},\ca{N}^{(\ell,i)}}$ for $i:\left|\ca{N}^{(\ell,i)}\right|\ge \gamma\s{C}$.

\end{algorithmic}
\end{algorithm*}

To obtain regret bounds, we first make the following assumptions on the matrix $\fl{X}\in \bb{R}^{\s{C}\times \s{M}}$ whose rows correspond to cluster reward vectors in the $\s{CS}$ setting:

\begin{assumption}[Assumptions on $\fl{X}$]\label{assum:matrix}
Let  $\fl{X}=\fl{U}\f{\Sigma}\fl{V}^{\s{T}}$ be the SVD of $\fl{X}$. Also, let $\fl{X}$ satisfy the following: 1) Condition number: $\fl{X}$ is full-rank and has non zero singular values $\lambda_1> \dots > \lambda_{\s{C}}$ with condition number $\lambda_1/\lambda_{\s{C}}=O(1)$, 2) $\mu$-incoherence:  $\lr{\fl{V}}_{2,\infty}\le \sqrt{\mu \s{C}/\s{M}}$, 3) Subset Strong Convexity (SSC): For some $\alpha$ satisfying $\alpha \log \s{M} = \Omega(1)$, $\gamma = \widetilde{O}(1)$, for all subset of indices $\ca{S}\subseteq [\s{M}], |\ca{S}| \ge \gamma\s{C}$, the minimum non-zero singular value of $\fl{V}_{\ca{S}}$ must be at least $\sqrt{\alpha\left|\ca{S}\right|/\s{M}}$. 
\end{assumption}


{\em Feasibility of Assumption~\ref{assum:matrix}.} The first two parts of Assumption \ref{assum:matrix} on  condition number and $\mu$-incoherence are fairly mild, and are satisfied by a variety of matrices. For example Gaussian random matrices are $\mu$-incoherent with $\mu =O(\log \s{N})$ \citep{candes2009exact}. However, the third part of Assumption \ref{assum:matrix}  on subset strong convexity (SSC) is relatively strong. It says that the minimum non-zero singular value of all reasonably sized sub-matrices of $\fl{X}$ must be large. This is helpful in showing that the sub-matrices estimated in Line 4 of Algorithm~\ref{algo:phased_elim} are incoherent (which in turn provides matrix-completion guarantees).
For $\s{C}=1$, this assumption is satisfied by any matrix whose entries are slight perturbations of a positive constant. For general $\s{C}>1$, interestingly, the matrices that satisfy this assumption are related to maximally erasure-robust frames (\cite{fickus2012numerically,wang2018random}). But, it is an open problem to identify such matrices for $\s{M}\gg \s{C}$.   Despite this, we note that the third part of Assumption~\ref{assum:matrix} can be significantly relaxed. We do not actually need the SSC condition on all the subsets of $[\s{M}]$. We only need it on sub-matrices that are formed by Algorithm~\ref{algo:phased_elim}. Interestingly, we show that the number of such sub-matrices is only $\tilde{O}(1)$, and consequently popular random matrices satisfy this condition (see Appendix \ref{sec:feasibility} for both empirical and theoretical evidence).  However, to simplify the analysis and presentation in the paper, we go with the condition stated in Assumption~\ref{assum:matrix}, and \emph{not} the more refined condition above.

{\em Justification of Assumption~\ref{assum:matrix}.} Assumption \ref{assum:matrix} is similar to the assumptions required by standard low-rank matrix completion methods \citep{candes2009exact,bhojanapalli2014universal}. The main purpose of it is to guarantee that the sub-matrices of the reward matrix $\fl{P}$ estimated in Line 4 in Algorithm \ref{algo:phased_elim} are incoherent and have low condition numbers - conditions necessary to invoke standard low rank matrix completion guarantees (Lemma \ref{lem:min_acc}) for the respective sub-matrices. 
Intuitively, the incoherence condition seems important to obtain small regret because to get small regret we require an arm pull of $i$-th user to provide good information for $j$-th user. That is, the matrix $\fl{X}$ should have information "well-spread" out instead of information being concentrated in a few entries or in a few directions only. To see this, consider an extreme example (when these assumptions are not satisfied) when $\fl{X}=[\fl{I}_{\s{C}\times \s{C}} \; \f{0}]$. In that case, most of the arms will give no information when pulled; all the arms need to be sampled for all users to get a good estimate of the reward matrix. Further exploration into necessity of Assumption~\ref{assum:matrix} is left for future work. 

\begin{assumption}\label{assum:cluster_ratio}
We will assume that $\tau,\s{C}=O(1)$ and does not scale with the number of rounds $\s{T}$.
\end{assumption}
Note that the above assumption is just for simplicity of exposition. Our algorithm is indeed polynomial in $\s{C}$ and $\tau$, so we can incorporate more general $\tau$ and $\s{C}$. But for simplicity, we ignore these factors by assuming them to be constants. 

Next, we characterize some properties namely the condition number and incoherence of sub-matrices of $\fl{P}$ restricted to a \textit{nice} subset of users in the $\s{CS}$ setting 

 \begin{lemma}\label{lem:condition_num}
Suppose Assumption \ref{assum:matrix} is true.
 Consider a sub-matrix $\fl{P}_{\s{sub}}$ of $\fl{P}$ having non-zero singular values $\lambda'_1>\dots> \lambda'_{\s{C}'}$ (for $\s{C}'\le \s{C}$). Then, if the rows of $\fl{P}_{\s{sub}}$ correspond to a nice subset of users, we have $\frac{\lambda'_1}{\lambda'_{\s{C}'}} \le \frac{\lambda_1}{\lambda_{\s{C}}}\sqrt{\tau}$.
 \end{lemma}
 
 \begin{lemma}\label{lem:incoherence}
Suppose Assumption \ref{assum:matrix} is true. Consider a sub-matrix $\fl{P}_{\s{sub}}\in \bb{R}^{\s{N}'\times \s{M}'}$ (with SVD decomposition $\fl{P}_{\s{sub}}=\widetilde{\fl{U}}\widetilde{\f{\Sigma}}\widetilde{\fl{V}}$) of $\fl{P}$ whose rows correspond to a nice subset of users. Then, provided the number of columns in $\fl{P}_{\s{sub}}$ is larger than $\gamma\s{C}$, we must have $\lr{\widetilde{\fl{U}}}_{2,\infty} \le \sqrt{\frac{\s{C}\tau}{\s{N}'}}$ and $\lr{\widetilde{\fl{V}}}_{2,\infty} \le \sqrt{\frac{\mu \s{C}}{\alpha\s{M}'}}$. 
\end{lemma}

Lemmas \ref{lem:condition_num} and \ref{lem:incoherence} allow us to apply low rank matrix completion  (Lemma \ref{lem:min_acc}) to relevant sub-matrices of the reward matrix $\fl{P}$. Now, we are ready to present our main theorem:


\begin{thm}\label{thm:main_LBM}
Consider the MAB-LC problem in $\s{CS}$ framework with $\s{M}$ arms, $\s{N}$ users, $\s{C}$ clusters and $\s{T}$ rounds such that at every round $t\in [\s{T}]$, we observe reward $\fl{R}^{(t)}$ as defined in eq. (\ref{eq:obs}) with noise variance proxy $\sigma^2>0$. Let $\fl{P}\in \bb{R}^{\s{N}\times \s{M}}$ be the expected reward matrix and $\fl{X}\in \bb{R}^{\s{C}\times \s{M}}$ be the sub-matrix of $\fl{P}$ with distinct rows. Suppose  Assumption \ref{assum:matrix} is satisfied by $\fl{X}$ and Assumption \ref{assum:cluster_ratio} is true. Then Alg. \ref{algo:phased_elim} with $C'=c\s{C}^{-1}\min\Big(\|\fl{P}\|_{\infty},\frac{\sigma\sqrt{\mu}}{\log \s{M}}\Big)$ for some appropriate constant $c>0$ guarantees the regret $\s{Reg}(\s{T})$ to be:
\begin{align*}
 \widetilde{O}(\sigma\sqrt{\mu^3\s{T}(\s{M+N})}) +\sigma\sqrt{\s{NT}}).
\end{align*}
\end{thm}

To better understand the theorem, let's remove the scaling factors. Dividing the regret by $\sigma$ gives us a scale-free regret of $\sqrt{\s{T}(\s{M}+\s{N})}$. Now, even if we know the clustering structure apriori, the regret would be $\sqrt{\s{T}\s{M}}$, so we are only paying an additive factor of $\sqrt{\s{N}\s{T}}$ for the latent cluster structure, which is tight. 

Also, dividing the the regret by $\s{N}$ we get per-user regret of $\sqrt{(1+\frac{\s{M}}{\s{N}})} \sqrt{\s{T}'}$ where $\s{T}'=\s{T}/\s{N}$ is the average number of arm-pulls for each user. That is, when $\s{N}\gg \s{M}$, we need only $\log (\s{M}+\s{N})$ arm pulls to get reasonable estimate. Hence, per user, the number of arms that needs to be pulled decreases exponentially from $\Omega(\s{M})$ to $\log (\s{M}+\s{N})$. On the other hand, when the number of users is small, each user has to explore at least $\s{M}/\s{N}$ arms to collaboratively provide information about all the arms. This also matches the intuition, especially when the number of users is $1$ where the bound matches the standard single-user MAB bound.



We would like to add the following two remarks: 
\begin{rmk}[Generalization]\label{rmk:gen_1}
Our results can be generalized to the setting when the users are sampled according to a known non-uniform distribution in  different ways. We can simulate the uniform distribution in each phase of the Alg. \ref{algo:phased_elim} by ignoring several observations; this approach is disadvantageous since users with very low probability of getting sampled will increase the number of sufficient observations significantly. Another approach is to partition the set of users into disjoint buckets such that the probabilities of getting sampled for users in the same bucket are within a factor of $2$ of each other. Now, in each bucket, we can run Alg. \ref{algo:phased_elim} separately and simulate the uniform distribution in each phase. Since the number of buckets will be logarithmic in $[\s{N}]$, the regret remains same up to logarithmic factors.   
\end{rmk}

\begin{rmk}[Generalization Continued]\label{rmk:gen_2}
 We can generalize our results to the setting where $\tau,\s{C},\kappa=\lambda_1/\lambda_{\s{C}}$ scales with the number of rounds $\s{T}$ by modifying Thm. 2 in \cite{chen2019noisy} appropriately. This will lead to the first term of regret guarantee in Thm. \ref{thm:main_LBM} being $\sqrt{\s{VT}}$ where $\s{V}=\widetilde{O}\Big(\s{poly}(r,\tau,\s{C},\kappa) \Big(\s{N}+\s{M}\Big)\Big)$; hence we will have a $\s{poly}(r,\tau,\s{C},\kappa)$ additional multiplicative factor in the regret. See Appendix \ref{app:general_proof} for details on this generalization.  
\end{rmk}

In the $\s{CS}$ framework, we can also provide instance-dependent regret bounds that are sharper than worst case guarantees in Theorem \ref{thm:main_LBM}. Let us introduce some definitions: for every cluster $c\in[\s{C}]$, define the subset of arms $\ca{G}_{c,\ell}$ for all users $u \in \ca{C}^{(c)}$ and for all $\ell>1$ as 
\begin{align*}
\ca{G}_{c,\ell} \equiv \{j \in [\s{M}] \mid \epsilon_{\ell} \le \left|\fl{P}_{uj} - \fl{P}_{u\pi_u(1)} \right| \le \epsilon_{\ell-1}\}    
\end{align*}
 and $\ca{G}_{c,1} \equiv \{j \in [\s{M}] \mid \epsilon_{1} \le \left|\fl{P}_{uj} - \fl{P}_{u\pi_u(1)} \right|\}$ for $\ell=1$; $\ca{G}_{c,\ell}$ ($\ca{G}_{c,1}$) corresponds to the subset of arms having a sub-optimality gap that is between $\epsilon_{\ell-1}$ and $\epsilon_{\ell}$ (greater than $\epsilon_1$ respectively) for all users belonging to the cluster $\ca{C}^{(c)}$. There is no ambiguity in the definition since all users in the same cluster $\ca{C}^{(c)}$ have the same mean rewards over all arms. 
Let us also define 
$
\ca{H}_c \equiv \bigcup \limits_{\ell> 1} \s{argmin}_{j\in \ca{G}_{c,\ell}} \left|\fl{P}_{uj}-\fl{P}_{u\pi_u(1)}\right|     
$ with the understanding that whenever  $\ca{G}_{c,\ell} = \Phi$, there is no $\s{argmin}$ to be counted in the set.
For brevity of notation, let $\Psi_{c,a} \triangleq \fl{P}_{u\pi_u(1)}-\fl{P}_{ua}$ be the sub-optimality gap in the reward of arm $a$ for any user $u$ in cluster $c$.

\begin{thm}\label{thm:main_LBM2}
Consider the setting in Theorem \ref{thm:main_LBM} and the sets $\{\ca{G}_{c,\ell},\ca{H}_c\}_{c,\ell}$ as defined above. Then Alg. \ref{algo:phased_elim} with $C'=c\s{C}^{-1}\min\Big(\|\fl{P}\|_{\infty},\frac{\sigma\sqrt{\mu}}{\log \s{M}}\Big)$ for some appropriate constant $c>0$ guarantees the regret $\s{Reg}(\s{T})$ to be 
\begin{align*}
\tiny
    & \widetilde{O}\Big(\frac{\lr{\fl{P}}_{\infty}\s{V}}{C'^2}\fl{1}[\ca{G}_{c,1} \neq \emptyset]+ \sum_{c\in [\s{C}], a\in \ca{H}_c} \frac{\Psi_{c,a}\s{T}^{-2}}{\s{C}}+\frac{\s{C}\s{V}}{\Psi_{c,a}}\Big) \\
    &+\s{N}^{-1}\widetilde{O}\Big(\sum_{\substack{c\in [\s{C}] \\ a \in \{\pi_c(s)\}_{s=1}^{\gamma\s{C}}}}
    \left|\ca{C}^{(c)}\right|\Big(\frac{\sigma}{\Psi_{c,a}}+3\Psi_{c,a}\Big)\Big).
\end{align*}
where $\s{V}=\widetilde{O}\Big(\sigma^2 \mu^3 \Big(\s{N}+\s{M}\Big)\Big)$, $\{\pi_c(s)\}_{s=1}^{\gamma\s{C}}$ are the $\gamma\s{C}$ arms for cluster $c$ with the smallest sub-optimality gap. 
\end{thm}

Loosely speaking, the regret bound in Thm. \ref{thm:main_LBM2} scales as $\widetilde{O}((\s{M}+\s{N})/\Psi)$ where $\Psi$ is the minimum sub-optimality gap across all the $\s{N}$ users involved.
This is because arms with large sub-optimality gaps are quickly eliminated by LATTICE in the initial phases itself; therefore if most competing arms for a user has large sub-optimality gaps, then the user will end up pulling high reward arms more often .
Again, this guarantee improves over the $\widetilde{O}((\s{MN})/\Psi)$ regret trivially obtained without collaboration across users.

\subsubsection{Lower Bounds}
\label{sec:lower_bounds}
\begin{thm}[Distribution-free]
\label{thm:dist_free_lower_bounds}
Let $\s{C} \leq \min\{\s{M}, \s{N}\}$. Suppose the distributions of arm rewards are Bernoulli and suppose the user at round $t$ is sampled independently from  a distribution $\probuser$. Moreover, suppose the weighted fraction of users in the $i^{th}$ cluster is $\tau_i$. Let $\sup$ be the supremum over all problem instances and $\inf$ be the infimum over all algorithms with knowledge of $\s{M}, \s{N}, \s{C}$. Then
\[
\inf \sup \s{Reg}(\s{T}) \geq 0.02(R_1+R_2),
\]
where $R_1  = \sum_{n\in [\s{N}]} \mathbb{E}\left[\min\{\sqrt{\s{C}\s{T}_n}, \s{T}_n\}\right],$  $R_2  = \sum_{c\in [\s{C}]} \mathbb{E}\left[\min\{\sqrt{\s{M}\s{T}_c}, \s{T}_c\}\right]$. Here,
 $\s{T}_n\sim \text{Bin}(\s{T}, \probuser(n)),$ $\s{T}_c \sim \text{Bin}(\s{T}, \tau_c)$  are binomial random variables.
\end{thm}
We now specialize the above result to the case where $\probuser$ is a uniform distribution, and the cluster sizes are uniform.
\begin{coro}
\label{cor:dist_free_lower_bounds_uniform}
Consider the setting of Theorem~\ref{thm:dist_free_lower_bounds}. Suppose $\s{T} \geq 10(\s{M}+\s{N})\s{C},$ $\probuser$ is the uniform distribution, and suppose each cluster has the same number of users. Then $\inf \sup \s{Reg}(\s{T}) = \Omega(\sqrt{(\s{M}+\s{N})\s{CT}})$
\end{coro}
Together with Theorem~\ref{thm:main_LBM}, the above result shows that LATTICE achieves minimax optimal regret when $\s{C} = O(1)$, and the reward matrix $\fl{X}$ satisfies the incoherence condition.
\begin{thm}[Distribution-dependent]
\label{thm:dist_dep_lower_bounds}
Consider the setting of Theorem~\ref{thm:dist_free_lower_bounds}. Suppose there is a unique best arm for each cluster. Moreover, suppose our algorithm is uniformly efficient, \textit{i.e.,} for any sub-optimal arm $a$ of any user $u$, $\mathbb{E}[N_{a,u}(T)] = o(T^{\alpha})$ for all $\alpha \in (0,1)$. Then
\[
\lim_{\s{T}\to\infty}\frac{\s{Reg}(\s{T})}{\log{\s{T}}} \geq \sum_{c\in[\s{C}]} \sum_{a} \frac{\fl{P}_{u(c)\pi_{u(c)}(1)}(1-\fl{P}_{u(c)\pi_{u(c)}(1)})}{\Psi_{c,a}},
\]
where the inner summation is over the set of all sub-optimal arms in cluster $c$. Here, $u(c)$ is any user in cluster $c$, and $\fl{P}_{u(c)\pi_{u(c)}(1)}$ is the mean reward of the best arm in cluster $c$, and  $\Psi_{c,a} = \fl{P}_{u(c)\pi_{u(c)}(1)} - \fl{P}_{u(c)a}$.
\end{thm}
The lower bound in Theorem~\ref{thm:dist_dep_lower_bounds} can be tightened a bit more, albeit at the expense of readability. We provide this improved bound in the Appendix.

\section{LATTICE ALGORITHM FOR $\s{RCS}$}\label{sec:GLBM}

\begin{algorithm}[t]
\caption{\textsc{LATTICE Algorithm for GCS}   \label{algo:phased_elim_gen}}
\begin{algorithmic}[1]
\REQUIRE Number of users $\s{N}$, arms $\s{M}$, clusters $\s{C}$, rounds $\s{T}$, noise $\sigma^2$, separation $\nu>0$, parameters $\gamma \ge 1$, $\s{C}'\ge 0$.

\STATE Initialize as in Step 1 of Alg. \ref{algo:phased_elim}.
\FOR{$\ell=1,2,\dots,$}

\STATE Run Steps 3-4 as in
Algorithm \ref{algo:phased_elim}.

\IF{$\Delta_{\ell+1} \ge 2\nu$ and $\left|\ca{M}^{(\ell)}\right| < \s{C}$}

\STATE Run Steps 5-9 as in Algorithm \ref{algo:phased_elim}

\ELSE 

\STATE For each pair of sets $(\ca{M}^{(\ell,i)},\ca{N}^{(\ell,i)})$ s.t. $|\ca{N}^{(\ell,i)}|\ge \gamma\s{C}$, append $\ca{M}^{(\ell,i)}$ to $\ca{M}^{(\ell+1)}$ and its corresponding set of arms $\bigcap_{u\in \ca{M}^{(\ell,i)}} \ca{T}_u^{(\ell)}$ to $\ca{N}^{(\ell+1)}$. 

\ENDIF

\STATE Run Step 10 in Algorithm \ref{algo:phased_elim}.
 

\ENDFOR

\end{algorithmic}
\end{algorithm}

LATTICE for $\s{RCS}$ (Alg. \ref{algo:phased_elim_gen}) is very similar to Alg. \ref{algo:phased_elim} with the main novelty being cluster-wise elimination of arms in Steps $7$ that needs a more aggressive approach. In essence, Alg. \ref{algo:phased_elim_gen} has three components:

\begin{itemize}[leftmargin=*,noitemsep, nolistsep]

    \item \textbf{Joint Arm Elimination: } As in Algorithm \ref{algo:phased_elim}, we run a phased algorithm where in the $\ell^{\s{th}}$ phase, we maintain a partition of users $\ca{M}^{(\ell)}$ and a family of subsets of active arms $\ca{N}^{(\ell)}$ having a one-to-one mapping. For any set of users in $\ca{M}^{(\ell)}$ that has more than $\gamma\s{C}$ active arms, we use Matrix Completion techniques to jointly shrink their set of active arms and partition them even further. We stop this component if we end up with $\s{C}$ groups of users for the first time or if $\Delta_{\ell+1}\le 2\nu$. In essence, in each phase, we eliminate arms for multiple clusters of users together.
    
    \item \textbf{Cluster-wise Arm Elimination:} In the second part, we no longer seek to partition each subset of users any further since users in the same subset provably correspond to the same cluster. Here, for elimination of bad arms, we pursue an intersection-based approach of good arms over all users in the same subset (Step 7 in Alg. \ref{algo:phased_elim_gen}); this is more aggressive elimination as compared to the union-based approach (Step 8 in Alg. \ref{algo:phased_elim}) that was pursued in the previous component.
    
    \item \textbf{Upper Confidence Bound:} If number of active arms $\left|\ca{N}^{(\ell,i)}\right|$ for users in a subset $\ca{M}^{(\ell,i)}\in \ca{M}^{(\ell)}$ falls below $\gamma\s{C}$, then we start/continue the Upper Confidence Bound (UCB) algorithm for each user in $\ca{M}^{(\ell,i)}$ separately with their subset of active arms $\ca{N}^{(\ell,i)}$ (Step 10 in Alg. \ref{algo:phased_elim}).
\end{itemize}

\noindent \textbf{Theoretical guarantees:} We make similar assumptions on the reward matrix $\fl{P}\in \bb{R}^{\s{N}\times \s{M}}$ as in the $\s{CS}$ framework:

\begin{assumption}[Assumptions on reward matrix $\fl{P}$]\label{assum:matrix2}
We assume that $\fl{P}$ with SVD decomposition $\fl{P}=\fl{U}\f{\Sigma}\fl{V}^{\s{T}}$ satisfies the following  properties 1) (Condition Number) $\fl{P}$ has rank $\s{C}$ and has non zero singular values $\lambda_1>\lambda_2 > \dots > \lambda_{\s{C}}$ with $\lambda_1/\lambda_{\s{C}}=O(1)$ 2) ($\mu$-incoherence)  $\lr{\fl{U}}_{2,\infty}\le \sqrt{\mu \s{C}/\s{N}}$ and $\lr{\fl{V}}_{2,\infty}\le \sqrt{\mu \s{C}/\s{M}}$. 3) (Subset Strong Convexity (a)) For some constant $\beta>0$ and for any subset of indices $\ca{S}\subseteq [\s{N}], \ca{S}= \ca{C}^{(j)}$ (corresponding to some cluster of users), we must have $\fl{x}^{\s{T}}\fl{U}_{\ca{S}}^{\s{T}}\fl{U}_{\ca{S}}\fl{x} \ge \beta\tau /\s{C}$ for all unit norm vectors $\fl{x}\in \bb{R}^{\s{C}}$.
4) (Subset Strong Convexity (b)) 
For some $\alpha$ satisfying $\alpha \log \s{M} = \Omega(1)$, $\gamma = \widetilde{O}(1)$, for all subset of indices $\ca{S}\subseteq [\s{M}], |\ca{S}| \ge \gamma\s{C}$, the minimum non-zero singular value of $\fl{V}_{\ca{S}}$ must be at least $\sqrt{\alpha\left|\ca{S}\right|/\s{M}}$.
\end{assumption}

\begin{rmk}
Note that the Subset Strong Convexity (a) of Assumption \ref{assum:matrix2} (used for proving incoherence guarantees of relevant sub-matrices of $\fl{P}$-Lemma \ref{lem:incoherence2}) will be satisfied only if the separation $\nu$ is bounded from below (since for $\ca{S}=\ca{C}^{(j)}$, $\fl{U}_{\ca{S}}$ loses rank when $\nu=0$). However, when $\nu=0$, $\s{RCS}$ reduces to the $\s{CS}$ framework; here, we do not need (Subset Strong Convexity (a)) since we have a different analysis for proving incoherence guarantees (Lemma \ref{lem:incoherence}) of relevant sub-matrices. For extremely small $\nu$, we can combine the two analyses to obtain similar sufficient guarantees (by using triangle inequality for instance).
\end{rmk}

\begin{figure*}[t]
  \begin{subfigure}[t]{0.33\textwidth}
    \centering 
    \includegraphics[scale = 0.35]{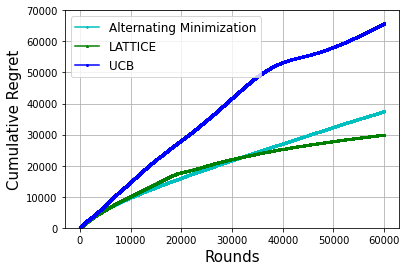}\vspace*{-5pt}
    \caption{\small Movielens dataset}
 ~\label{fig:gaussian2}
  \end{subfigure}
 \hfill
\begin{subfigure}[t]{0.33\textwidth}
\centering
  \includegraphics[scale = 0.35]{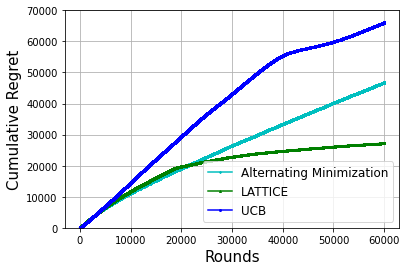}\vspace*{-5pt}
  \caption{\small Netflix dataset}
      ~\label{fig:uniform2}
 \end{subfigure}%
 \hfill
\begin{subfigure}[t]{0.33\textwidth}
\centering
  \includegraphics[scale = 0.35]{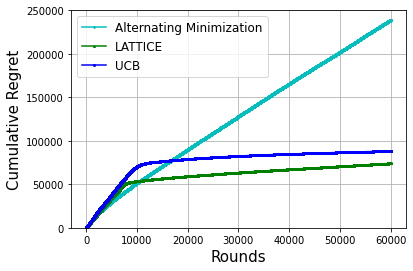}\vspace*{-5pt}
  \caption{\small Jester dataset}
      ~\label{fig:uniform3}
 \end{subfigure}%
 \vspace{-10pt}
 \caption{
 \small Cumulative Regret of the Alternating Minimization (AM) algorithm in \cite{dadkhahi2018alternating}, LATTICE (Alg. 3) and Upper Confidence Bound (UCB) algorithm with  $\s{T}=60000$ rounds for 3 datasets a) \textbf{Movielens 10m dataset:} $\s{N}=200$ users and $\s{M}=200$ arms b) \textbf{Netflix dataset:} $\s{N}=200$ users and $\s{M}=200$ arms c) \textbf{Jester dataset:} $\s{N}=100$ users and $\s{M}=100$ arms. }
\vspace{-15pt}
\end{figure*}

As before, we characterize the condition number and the incoherence of the relevant sub-matrices of $\fl{P}$ that will allow us to apply low rank matrix completion techniques and provide theoretical guarantees (see Lemma \ref{lem:min_acc}).

\begin{lemma}\label{lem:condition_num2}
Suppose Assumption \ref{assum:matrix2} is true.
  Consider a sub-matrix $\fl{P}_{\s{sub}}$ of $\fl{P}$ having non-zero singular values $\lambda'_1>\dots> \lambda'_{\s{C}'}$ (for $\s{C}'\le \s{C}$). Then, provided $\fl{P}_{\s{sub}}$ is non-zero, we have $\frac{\lambda'_1}{\lambda'_{\s{C}'}} \le \frac{\lambda_1}{\lambda_{\s{C}}}$.
 \end{lemma}
 
 \begin{lemma}\label{lem:incoherence2}
Suppose Assumption \ref{assum:matrix2} is true. Consider a sub-matrix $\fl{P}_{\s{sub}}\in \bb{R}^{\s{N}'\times \s{M}'}$ (with SVD decomposition $\fl{P}_{\s{sub}}=\widetilde{\fl{U}}\widetilde{\f{\Sigma}}\widetilde{\fl{V}}$) of $\fl{P}$ whose rows correspond to a nice subset of users. Then, provided the number of columns in $\fl{P}_{\s{sub}}$ is larger than $\gamma\s{C}$, we must have $\lr{\widetilde{\fl{U}}}_{2,\infty} \le \sqrt{\frac{\s{C}\tau}{\s{N}'}}$ and $\lr{\widetilde{\fl{V}}}_{2,\infty} \le \sqrt{\frac{\mu \s{C}}{\alpha\s{M}'}}$. 
\end{lemma}

Now, we are ready to state our main theorems

\begin{thm}\label{thm:main_GLBM}
Consider the MAB-LC problem in $\s{RCS}$ framework with $\s{M}$ arms, $\s{N}$ users, $\s{C}$ clusters and $\s{T}$ rounds such that at every round $t\in [\s{T}]$, we observe reward $\fl{R}^{(t)}$ as defined in eq. (\ref{eq:obs}) with noise variance proxy $\sigma^2>0$. Let $\fl{P}\in \bb{R}^{\s{N}\times \s{M}}$ be the expected reward matrix such that Assumption \ref{assum:matrix2} is satisfied by $\fl{P}$. Moreover, suppose Assumption \ref{assum:cluster_ratio} is true. Then Alg. \ref{algo:phased_elim_gen} with $C'=c\s{C}^{-1}\min\Big(\|\fl{P}\|_{\infty},\frac{\sigma\sqrt{\mu}}{\log \s{M}}\Big)$ for some appropriate constant $c>0$ guarantees the regret $\s{Reg}(\s{T})$ to be:
\begin{align}\label{eq:regret2}
    \widetilde{O}(\sigma\sqrt{\mu^3\s{T}(\s{M+N})}) +\sigma \sqrt{\s{NT}}).
\end{align}
\end{thm}

Note that the regret bound above is similar to that of Theorem~\ref{thm:main_LBM}, despite the stricter setting. Here again, the "scale-free" regret is $\sqrt{\s{T}(\s{M}+\s{N})}$ which as discussed in remarks below Theorem~\ref{thm:main_LBM}, is intuitive, is practical in realistic regimes, and is nearly optimal. 

\begin{rmk}
Note that the generalization remarks \ref{rmk:gen_1},\ref{rmk:gen_2} in Sec. \ref{subsec:theory} also extend to Theorem \ref{thm:main_GLBM}.
Also, recall the definitions of $\ca{H}_c,\ca{G}_{c,\ell}$ for clusters $c\in [\s{C}]$ and phases indexed by $\ell \ge 1$ depending on the sub-optimality gap from Section \ref{subsec:theory}. With equivalent definitions for the $\s{RCS}$ setting, the gap dependent regret bounds in Theorem \ref{thm:main_LBM2} can be achieved by Algorithm \ref{algo:phased_elim_gen} as well provided Assumptions  \ref{assum:cluster_ratio} and \ref{assum:matrix2}  are true.
\end{rmk}

\section{EXPERIMENTS}

We have provided detailed experiments on synthetic datasets (deferred to Appendix \ref{sec:experiments}) and popular real world recommendation data-sets namely 1) Movielens 10m dataset 2) Netflix dataset and 3) Jester dataset. For simplicity, we implement a significantly simplified version of our algorithm (Alg. \ref{algo:simple_lattice} in Appendix \ref{sec:experiments}).  
In addition, we have also compared with a highly competitive heuristic - the Alternating Minimization (AM) algorithm described in \cite{dadkhahi2018alternating} and the standard Upper Confidence Bound (see \cite{lattimore2020bandit}) algorithm individually for each user. However, we stress that the AM algorithm does not have any theoretical guarantees.
 For the Movielens dataset, we restricted ourselves to the $200$ users ($\s{N}$) who have rated most movies and $200$ movies ($\s{M}$) that have been rated the most. For Netflix and Jester, with a similar pre-processing, the values of $\s{N},\s{M}$ are $(200,200)$ and $(100,100)$ respectively. We compared the performance of our algorithm LATTICE (for $\s{GCS}$ - i.e. after a few phases, we run UCB individually for each user with their active items) with the AM algorithm in \cite{dadkhahi2018alternating} (with the hyper-parameters provided in \cite{dadkhahi2018alternating} for Movielens and Jester datasets; for Netflix dataset, we used the hyperparameters provided for Movielens).
 In Figures \ref{fig:gaussian2},\ref{fig:uniform2} and \ref{fig:uniform3}, we have shown the cumulative regret of the three algorithms -clearly, LATTICE outperforms the other baselines empirically as well. In particular, LATTICE successfully removes large chunks of bad items for large groups of users jointly in few rounds. Further details about implementation are deferred to Appendix \ref{sec:experiments}.

\section{CONCLUSION}

For the multi-user multi-armed latent bandit problem introduced in \cite{maillard2014latent} we provided a novel, computationally efficient algorithm LATTICE. 
Ours is the first algorithm to obtain $\widetilde{O}(\sqrt{(\s{M}+\s{N})\s{T}})$  regret guarantee in this challenging and practically important setting, as latent cluster structure in users/agents is commonplace and is a standard modeling tool for practitioners. 
Our work also resolves open problems posed in \cite{jain2022online} and \cite{sen2017contextual} for online low rank matrix completion in certain special case. Finally, it would be interesting to optimize the regret dependence on other factors such as the number of clusters ($\s{C}$), $\s{RCS}$-gap ($\nu$), as well as other parameters incoherence and condition number. 


\bibliographystyle{unsrtnat}
\bibliography{references.bib}

\newpage
\onecolumn
\appendix

\begin{algorithm*}[!htbp]
\caption{Simplified LATTICE
\label{algo:simple_lattice}}
\begin{algorithmic}[1]
\REQUIRE Number of users $\s{N}$, arms $\s{M}$, clusters $\s{C}$, rounds $\s{T}$, noise $\sigma^2>0$,  phase lengths $\{\Delta_{\ell}\}_{\ell \ge 1}$ satisfying $\sum_{\ell} \Delta_{\ell} = \s{T}$. Gap Parameters $\{\nu_{\ell}\}_{\ell \ge 1}$. Parameter $\lambda \ge 0$ for nuclear norm minimization. Phase parameter $\s{L} \ge 0$ and robust intersection parameter $0 \le \rho \le 1$.
\STATE Partition entire time period into phases $[1,\Delta_{1}],[\Delta_{1}+1,\Delta_1+\Delta_2],\dots$. 
\STATE Set $\ca{M}^{(1)} \equiv [\ca{M}^{(1,1)}]$ where $\ca{M}^{(1,1)}=[\s{N}]$ and $\ca{N}^{(1)} \equiv [\ca{N}^{(1,1)}]$ where $\ca{N}^{(1,1)}=[\s{M}]$.
\FOR{$t = 1,2,\dots,\s{T}$}
\STATE Sample user $u(t)$ from $[\s{N}]$, determine phase $\ell$ in which round $t$ belongs
\STATE Determine the set $\ca{M}^{(\ell,i)}$ in $\ca{M}^{(\ell)}$ in which $u(t)$ belongs. 
\STATE User $u(t)$ pulls an arm $\rho(t)$ uniformly at random from $\ca{N}^{(\ell,i)}$ and observes feedback $\fl{R}^{(t}$.
\IF{last round of phase and $\ell \le \s{L}$}
\STATE Initialize $\ca{M}^{(\ell+1)} = [],\ca{N}^{(\ell+1)} = []$
\FOR{each set $\ca{M}^{(\ell,i)}\in \ca{M}^{(\ell)}$} 
 \STATE Consider Matrix $\fl{Q}\in \bb{R}^{\s{N}\times \s{M}}$. Fill entries in the sub-matrix $\fl{Q}_{\ca{M}^{(\ell,i)},\ca{N}^{(\ell,i)}}$ by assigning $\fl{Q}_{ij}=\text{mean}\{\fl{R}^{(t)}\mid t \in \text{ phase }\ell, i=u(t),j = \rho(t), i \in \ca{M}^{(\ell,i)}\}$.
 \STATE Let $\Omega$ be the filled entries in $\fl{Q}_{\ca{M}^{(\ell,i)},\ca{N}^{(\ell,i)}}$. Complete the matrix $\fl{Q}_{\ca{M}^{(\ell,i)},\ca{N}^{(\ell,i)}}$ by solving the convex program 
\begin{align}\label{eq:convex2}
    \min_{\fl{T}\in \bb{R}^{\s{N}\times \s{M}}} \frac{1}{2}\sum_{(i,j)\in \Omega}\Big(\fl{Q}_{ij}-\fl{T}_{ij}\Big)^2+\lambda\|\fl{T}_{\ca{M}^{(\ell,i)},\ca{N}^{(\ell,i)}}\|_{\star},
\end{align}
\IF{$\ell \le \s{L}$}
\STATE Solve $k$-means for users in $\ca{M}^{(\ell,i)}$ using the vector embedding formed by the rows in $\fl{T}_{\ca{M}^{(\ell,i)},\ca{N}^{(\ell,i)}}$. Choose best $k\le \s{C}$ by using ELBOW method. Denote the cluster of users by $\{\ca{M}^{(\ell,i,j)}\}_j$.
\FOR{each cluster of users $\ca{M}^{(\ell,i,j)}$} 
\STATE Compute $\ca{N}^{(\ell,i,j)}$ as $\{s \in \ca{N}^{(\ell,i)} \mid |\fl{T}_{us}-\max_{s'\in \ca{N}^{(\ell,i)}}\fl{T}_{us'}| \le \nu_{\ell} \text{ for some } u \in \ca{M}^{(\ell,i,j)}\}$.
\STATE Append $\ca{M}^{(\ell,i,j)}$ to $\ca{M}^{(\ell+1)}$  and $\ca{N}^{(\ell,i,j)}$ to $\ca{N}^{(\ell+1)}$.
\ENDFOR
 \ELSE
 \STATE Compute set of active arms $\ca{N}^{(\ell,i)}$ as $\{s \in \ca{N}^{(\ell,i)} \mid |\fl{T}_{us}-\max_{s'\in \ca{N}^{(\ell,i)}}\fl{T}_{us'}| \le \nu_{\ell} \text{ for at least $\rho$-fraction of users in }  \ca{M}^{(\ell,i)}\}$.
 \STATE Append $\ca{M}^{(\ell,i)}$ to $\ca{M}^{(\ell+1)}$  and $\ca{N}^{(\ell,i)}$ to $\ca{N}^{(\ell+1)}$. \textit{\#Instead of Steps 20,21, we can also start running UCB individually for each user in $\ca{M}^{(\ell,i)}$ with the set of active items $\ca{N}^{(\ell,i)}$ for the remaining rounds. This can be more practical since cluster structures are not always satisfied exactly.}
\ENDIF

\ENDFOR

\ENDIF
\ENDFOR 

\end{algorithmic}
\end{algorithm*}

\paragraph{Organization:} The Appendix is organized as follows: in Section \ref{sec:experiments}, we provide detailed synthetic experiments with a simplified version of the LATTICE algorithm. In Section \ref{sec:comparison}, we provide a more detailed comparison with the online clustering line of work. In Section \ref{sec:feasibility}, we provide detailed proof for feasibility of Assumption \ref{assum:matrix}. In Section \ref{app:mc}, we provide details on proofs of results presented in Section \ref{sec:prelims}. In Section \ref{app:detailed_lbm}, we provide detailed proof of Theorems \ref{thm:main_LBM} and \ref{thm:main_LBM2}. In Section \ref{app:lower_bounds}, we provide detailed proofs of the lower bounds on cumulative regret. In Section \ref{app:glbm}, we provide detailed proof of Theorem \ref{thm:main_GLBM}. Finally in Section \ref{app:general_proof}, we provide a proof of a general version of Lemma \ref{lem:min_acc} and the regret guarantee claimed in Remark \ref{rmk:gen_2}.   

\section{Further Experiments}\label{sec:experiments}

\subsection{Synthetic Datasets}

\begin{figure*}[!htbp]
  \begin{subfigure}[t]{0.5\textwidth}
    \centering 
    \includegraphics[scale = 0.5]{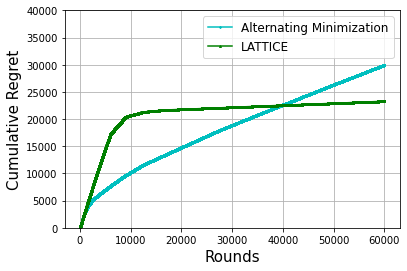}\vspace*{-5pt}
    \caption{\small $\fl{V}$ has entries distributed according to standard normal.}
 ~\label{fig:gaussian}
  \end{subfigure}
 \hfill
\begin{subfigure}[t]{0.5\textwidth}
\centering
  \includegraphics[scale = 0.5]{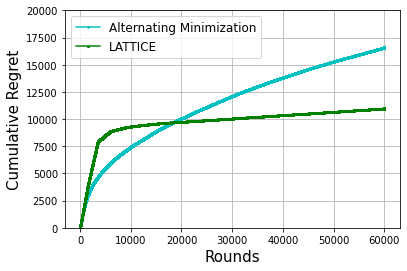}\vspace*{-5pt}
  \caption{\small $\fl{V}$ has entries generated according to a uniform distribution}
      ~\label{fig:uniform}
 \end{subfigure}%
 \caption{Cumulative Regret of the Alternating Minimization (AM) algorithm in \cite{dadkhahi2018alternating} and LATTICE (Alg. \ref{algo:simple_lattice}) with $\s{N}=200$ users, $\s{M}=200$ arms, $\s{C}=4$ clusters and $\s{T}=60000$ rounds. The reward matrix $\fl{P}=\fl{U}\fl{V}^{\s{T}}$ is generated in the following way: each row of $\fl{U}$ is a standard basis vector while each entry of $\fl{V}$ is sampled independently from a standard normal distribution $\ca{N}(0,1)$ in (a) and each entry of $\fl{V}$ is sampled independently from $\ca{U}[0,5]$ in (b). Notice that LATTICE becomes superior to AM as the number of rounds increase.}
\vspace{-10pt}
\end{figure*}

For experimentation, we will run Algorithm \ref{algo:simple_lattice} that involves the following  simplifications of a) \textit{for each matrix completion step in Algorithm \ref{algo:phased_elim}} - every user randomly pulls arms in the active set of arms (see Step 7 in Alg. \ref{algo:simple_lattice}) and subsequently, a single optimization problem with nuclear norm minimizer is solved at Step 12 b) \textit{the clustering step using graphs in Steps 7-8 of Alg. \ref{algo:phased_elim}} - we use $k$-means to cluster the users (see Step 14) where the vector embedding of each user is the row in the sub-matrix estimate that we computed by completing the sub-matrix corresponding to the subset of users (that the said user belongs to) and its active subset of arms. 

Next we perform detailed experiments with Algorithm \ref{algo:simple_lattice} on synthetic datasets that are generated as described below. Note that we compare against the Alternating Minimization (AM) algorithm presented in \cite{dadkhahi2018alternating} for solving the online multi-user multi-armed bandit problem when the reward matrix is of low-rank. Note that the AM algorithm is a very strong baseline in practice for our problem setting. In \cite{dadkhahi2018alternating}, it was experimentally demonstrated for both synthetic and real datasets that the AM algorithm outperforms previously designed algorithms in the literature that can be applied to our problem setting (\cite{sen2017contextual} and \cite{kawale2015efficient}) by a significant margin. 

\paragraph{Dataset Generation:} We take $\s{N}=200$ users, $\s{M}=200$ arms, $\s{C}=4$ clusters and the number of rounds $\s{T}=60000$. We generate the ground truth matrix $\fl{P}=\fl{U}\fl{V}^{\s{T}}$ where $\fl{U}\in \bb{R}^{\s{N}\times \s{C}}$, $\fl{V}\in \bb{R}^{\s{M}\times \s{C}}$ in the following manner: in the $i^{\s{th}}$ row of $\fl{U}$, the $(i\%\s{C})^{\s{th}}$ entry is set to be \texttt{1} while the other entries are $0$, each entry of $\fl{V}$ is sampled uniformly at random from (a) standard normal distribution $\ca{N}(0,1)$  (b) uniform distribution $[0,5]$. For (a), we assume that the noise added to each observed entry is sampled independently from $\ca{N}(0,0.5)$. For (b), we assume that the noise added is uniformly distributed in $[-0.5,0.5]$.

\paragraph{Algorithm Details:} We tune both the AM  and the LATTICE algorithm (Alg. \ref{algo:simple_lattice}). The AM algorithm \cite{dadkhahi2018alternating} has two hyper-parameters $\lambda_1,\lambda_2$ which are set to be $0.5$ and $0.01$ respectively for both data-sets generated according to (a) and (b). Alg. \ref{algo:simple_lattice} has several hyperparameters - we set the phase length $\Delta_{\ell}=1500+500*(\ell-1)$, the gap parameters $\nu_{\ell} = \lr{\fl{P}}_{\infty}/6\cdot 8^{\ell}$ and $\s{L}=5$. We also take $\lambda=5\sqrt{\Delta_{\ell}/200}$ for the convex relaxation problem in \ref{eq:convex2}.
In Step 14 of Alg. \ref{algo:simple_lattice}, we choose the best $k$ using the following heuristic: we go on increasing $k$ by $1$ if the objective function of $k$-means reduces by a factor of at least $0.6$; also if the objective is less than $100$, we do not split the cluster anymore.
Again, these hyperparameters remain same for both datasets generated according to (a) and (b).

\paragraph{Results and Insights:} The cumulative regret (averaged over $5$ independent runs) is plotted for both the AM  and  LATTICE algorithms (Alg. \ref{algo:simple_lattice}) in Figures \ref{fig:gaussian} (Gaussian) and \ref{fig:uniform} (Uniform) respectively. Notice that for both synthetic datasets (Gaussian and Uniform), in the initial periods, AM has a better performance while in latter stages LATTICE improves significantly and eventually beats it. The reason is that the AM algorithm starts creating confidence sets for arm pulls for every user from the first few rounds itself. However, in almost all the runs, the AM algorithm fails to converge to the best arm for many users (although it does converge to arms with very small sub-optimality gap for each user). On the other hand, LATTICE, in the initial few phases mimics pure exploration but it converges to the best arm for most users almost always. Therefore, the cumulative regret of LATTICE hardly increases after a certain number of rounds whereas the cumulative regret of AM goes on increasing. Therefore, we can conclude that in synthetic datasets where our assumptions namely the cluster structure is satisfied, LATTICE is not only competitive with the AM algorithm but also has superior performance when the number of rounds is large. One disadvantage of the AM algorithm that it is quite sensitive to the choice of hyperparameters - a slightly incorrect choice leads to diverging of the regret guarantees from the first few rounds itself; in comparison, LATTICE is much more stable with respect to the choice of hyperparameters.

\subsection{Real-world datasets (Implementation details)}

For all three datasets 1) Movielens 10m 2) Netflix and 3) Jester, we set the phase length in Algorithm \ref{algo:simple_lattice} to be $\Delta_{\ell}=2000+500*(\ell-1)$, the gap parameters $\nu_{\ell} = \lr{\fl{P}}_{\infty}/4\cdot 2^{\ell}$ and $\s{L}=5$ for Movielens and Netflix; for Jester dataset, we took $\nu_{\ell} = \lr{\fl{P}}_{\infty}/6\cdot 8^{\ell}$ and $\s{L}=5$. After five phases, instead of Steps 20,21 in Alg. \ref{algo:simple_lattice}, we start running standard UCB for each user with the set of active items for the remaining rounds. 
We also take $\lambda=5\sqrt{\Delta_{\ell}/200}$ for the convex relaxation problem in \ref{eq:convex2}.
In Step 14 of Alg. \ref{algo:simple_lattice}, we choose the best $k$ using the following heuristic: we go on increasing $k$ by $1$ if the objective function of $k$-means reduces by a factor of at least $0.6$; also if the objective is less than $50$, we do not split the cluster anymore. 

\section{Detailed comparison with Online Clustering}\label{sec:comparison}
\cite{gentile2014online,gentile2017context,li2019improved} study the contextual version of the MAB-LC problem considered in our work. In their set up, a random user $u \in {\cal U}$ arrives at time $t$, the online algorithm is presented with an action space ${\cal A}_t$ where each action $a \in {\cal A}_t$ has a feature vector $\mathbf{x}_a$. If the action chosen is $k_t$, then the mean rewards obtained is $\mathbf{x}^T_{k_t} \mathbf{\beta}_{c(u)}$ where $\beta_c$ are the model parameters for cluster $c$. All users $u$ such that $c(u) = c$ have an identical reward model. 

We can map our problem to this problem by presenting a fixed action set in every slot, i.e. $\mathbf{x}_i = \mathbf{e}_i$ for all $i\in[\s{M}]$ where $\{\mathbf{e}_i\}$ is the canonical basis in $\mathbb{R}^{\s{M}}$, and $\mathcal{A}_t = \{\mathbf{x}\}_{i=1}^{\s{M}}$ for all $t$. However, such a conversion results in highly sub-optimal regret of $\tilde{O}(\sqrt{\s{M}^2\s{CT}} + \s{M}^3\s{N})$. There are two main reasons for this. One is that this conversion leads to extremely high feature vector dimension of $\s{M}$. The other is that the algorithms in \cite{gentile2014online,gentile2017context,li2019improved} \textit{crucially} depend on the assumption that for a fixed $a$ at time $t$, the feature vector $\mathbf{x}_a$ is sampled i.i.d from a distribution on the unit sphere such that minimum singular value of $\mathbb{E}[\mathbf{x}_a \mathbf{x}_a^T]$ is at least a constant. Based on our conversion above, it is easy to see that $\mathbb{E}[\mathbf{x}_a \mathbf{x}_a^T] = \frac{1}{\s{M}}$  for MAB-LC. Since $\s{M}$ is very large, the minimum singular value in our setting is quite small which leads to poor regret. This assumption is crucial to the analysis of \cite{gentile2014online,gentile2017context,li2019improved}, and removing it is non-trivial. Consider a user $u$ in the system and the Gram matrix $S_{u,t} = \sum_{s<t: u_s=u} \mathbf{x}_{k_s} \mathbf{x}_{k_s}^T$ formed for user $u$ based on feature vectors $x_{k_s}$ of actions chosen at times when the user $u$ arrived in the system. Crucial property that is needed for online clustering to proceed in the works of \cite{gentile2014online,gentile2017context,li2019improved} is that the minimum singular value of $S_{u,t}$ is $\Omega(T_{u,t})$ where $T_{u,t}$ is the number of time slots user $u$ arrived till time $t$ with very high probability (see for instance Claim $1$ in \cite{gentile2014online}, Lemma $4$ in \cite{li2019improved}). This is ensured through the randomness assumption for $\mathbf{x}_a$. In our case with mapping to canonical basis vectors, minimum singular value of $S_{u,t}$ will scale sub-linearly ($o(T_{u,t})$ for large $T_{u,t}$ ) if the algorithm is doing well on user $u$ in terms of regret, i.e. focusing on arms close to the best arm.

Note that this can be seen as a motivation for our elimination style algorithm, since we only rely on the overlap in the set of `good arms' of every user for clustering, while in \cite{gentile2014online,gentile2017context,li2019improved}, the authors use the estimate of the entire mean reward vector to cluster. This requires the estimation error to be low in 'all directions' for the Gram matrix.


\section{Further Discussion on Feasibility of Assumptions}\label{sec:feasibility}

\subsection{$\s{C}=1$}

For the special case of $\s{C}=1$, Assumption \ref{assum:matrix} is satisfied by any $\fl{X}\in \bb{R}^{1\times \s{M}}$ that satisfies the following: for all $i\in [\s{M}]$, we have that $\fl{X}_i$ denoting the $i^{\s{th}}$ entry of $\fl{X}$ is bounded from below by $\nu>0$ and from above by $1$. In that case, the SVD of $\fl{X}$ is denoted by $u \sigma \fl{v}^{\s{T}}$ where $u=1$, $\sigma=\lr{\fl{X}}_2$ and $\fl{v}=\fl{X}/\lr{\fl{X}}_2$. Clearly, the condition number of $\fl{X}$ is $1$. Next, note that $\lr{\fl{v}}_{\infty} \le 1/\sqrt{\nu \s{M}}$. For any sub-set $\ca{J}\subseteq [\s{M}]$, we can have a similar conclusion on $\fl{X}$ restricted to the indices in $\ca{J}$. Hence Assumption \ref{assum:matrix} is satisfied for such $\fl{X}$. 

\subsection{Gaussian ensemble (Simulations)}

\begin{figure*}[!htbp]
  \begin{subfigure}[t]{0.4\textwidth}
    \centering 
    \includegraphics[scale = 0.5]{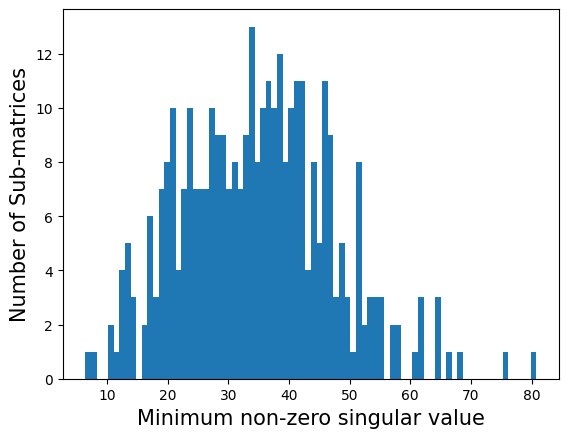}\vspace*{-5pt}
    \caption{\small $\fl{V}$ has entries distributed according to standard normal. We plot the minimum singular value of $\fl{V}$ restricted to the items corresponding to unique sub-matrices of users and items that we estimate in Step 11 of Alg. \ref{algo:simple_lattice}- Alg. \ref{algo:simple_lattice} is run $50$ times for a single sample of $\fl{V}$.}
 ~\label{fig:sim1}
  \end{subfigure}
 \hfill
\begin{subfigure}[t]{0.5\textwidth}
\centering
  \includegraphics[scale = 0.5]{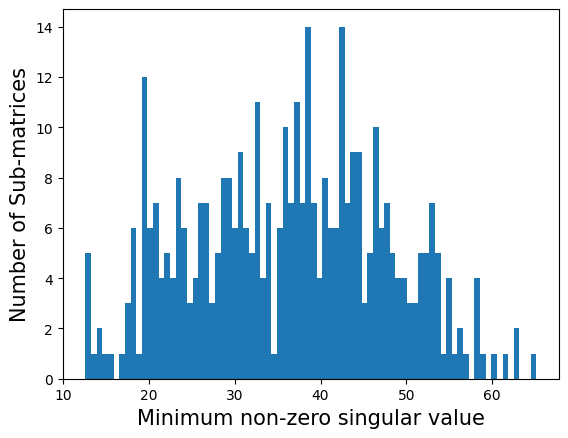}\vspace*{-5pt}
  \caption{\small $\fl{V}$ has entries distributed according to standard normal. We plot the minimum singular value of $\fl{V}$ restricted to the items corresponding to unique sub-matrices of users and items that we estimate in Step 11 of Alg. \ref{algo:simple_lattice}- Alg. \ref{algo:simple_lattice} is run $10$ times each for $10$ samples of $\fl{V}$.}
      ~\label{fig:sim2}
 \end{subfigure}%
 \caption{Minimum non-zero singular value via simulations on a Gaussian ensemble. }
\vspace{-10pt}
\end{figure*}

We consider the setting in Section \ref{sec:experiments} with $\s{M}=200$ users, $\s{N}=200$ items and $\s{C}=4$. Here, $\fl{P}=\fl{U}\fl{V}^{\s{T}}$ where $\fl{U}\in \bb{R}^{\s{N}\times \s{C}}$, $\fl{V}\in \bb{R}^{\s{M}\times \s{C}}$ in the following manner: in the $i^{\s{th}}$ row of $\fl{U}$, the $(i\%\s{C})^{\s{th}}$ entry is set to be \texttt{1} while the other entries are $0$, each entry of $\fl{V}$ is sampled uniformly at random from a standard normal distribution $\ca{N}(0,1)$. For all sub-matrices with $\ca{U}'\subseteq [\s{N}]$, $\ca{V}'\subseteq [\s{M}]$ that we need to estimate in Step 11 of Alg. \ref{algo:simple_lattice}, we report the histogram of minimum singular value of $\fl{V}_{\ca{V}'}$ for all unique $\fl{V}'$ in two sets of experiments 1) We take one sample of $\fl{V}$ where each entry of $\fl{V}$ is sampled from $\ca{N}(0,1)$ and run Algorithm \ref{algo:simple_lattice} $50$ times. 2) We take $10$ samples of $\fl{V}$ where each entry of $\fl{V}$ is sampled from $\ca{N}(0,1)$ and run Algorithm \ref{algo:simple_lattice} $10$ times for each of them.
In both case, we notice that the minimum singular value is sufficiently large - in particular, more than a large enough constant.

\subsection{Relaxing Subset Strong Convexity Assumption}\label{sec:ssc_relaxed}

In this section we are going to show that when the entries of the matrix $\fl{X}$ are independently generated according to $\ca{N}(0,1)$, we can ensure that in all phases indexed by $\ell$, for all sub-matrices $\fl{P}_{\ca{M}^{(\ell,i)},\ca{N}^{(\ell,i)}}$  (corresponding to a nice subset of users $\ca{M}^{(\ell,i)}$ and their active items $\ca{N}^{(\ell,i)}$) that are estimated in Line 4 of Algorithm \ref{algo:phased_elim}, we will have the following properties for the SVD of $\widetilde{\fl{U}}\widetilde{\Sigma}\widetilde{\fl{V}}$ of the sub-matrix $\fl{P}_{\ca{M}^{(\ell,i)},\ca{N}^{(\ell,i)}}$:

\begin{enumerate}
    \item \textbf{(P1:)} The condition number of the matrix $\fl{P}_{\ca{M}^{(\ell,i)},\ca{N}^{(\ell,i)}}$ that is, the ratio of the maximum and minimum non-zero singular value is bounded from above by a constant.
    \item \textbf{(P2:)} The orthonormal matrices $\widetilde{\fl{U}},\widetilde{\fl{V}}$ are incoherence i.e. we will have $\lr{\widetilde{\fl{U}}}_{2,\infty} =O(\sqrt{\mu /\left|\ca{M}^{(\ell,i)}\right|})$ and $\lr{\widetilde{\fl{V}}}_{2,\infty} =O(\sqrt{\mu /\left|\ca{N}^{(\ell,i)}\right|})$ for some small $\mu$.
\end{enumerate}

The only minor algorithmic modification that we need is to use $C'\le 1/3\sqrt{\log \s{M}}$ and $\Delta_{\ell+1}=C' \s{C}^{-\ell}$ in Step 3 of Algorithm \ref{algo:phased_elim}. The reasons for these minor modifications will become apparent in the analysis - however such a change will only lead to additional multiplicative logarithmic factors in the regret.

We start by showing the following lemmas (recall that $\fl{X}_{\mid \ca{S}}$ corresponds to the matrix $\fl{X}$ restricted to the columns in $\ca{S}$.)

\begin{lemma}\label{lem:sss}
If $\fl{x}^{\s{T}}\fl{X}_{\mid \ca{S}}\fl{X}_{\mid \ca{S}}^{\s{T}}\fl{x} \ge \alpha \gamma\s{C}\lambda_1^2/\s{M}$ for a subset $\ca{S}\subseteq [\s{M}], |\ca{S}|=\gamma\s{C}$ for all unit vectors $\fl{x}\in \bb{R}^{\s{C}}$, then the minimum eigenvalue of $\fl{V}_{\ca{S}}^{\s{T}}\fl{V}_{\ca{S}} \ge \alpha\gamma \s{C}/\s{M}$. In other words, Subset Strong Convexity (SSC) of $\fl{X}$ with SVD decomposition $\fl{X}=\fl{U}\f{\Sigma}\fl{V}^{\s{T}}$ implies SSC of $\fl{V}$. 
\end{lemma}

\begin{proof}
 If $\fl{x}^{\s{T}}\fl{X}_{\mid \ca{S}}\fl{X}_{\mid \ca{S}}^{\s{T}}\fl{x} \ge \alpha \gamma\s{C}\lambda_1^2/\s{M}$ for a subset $\ca{S}\subseteq [\s{M}], |\ca{S}|=\gamma\s{C}$ for all unit vectors $\fl{x}\in \bb{R}^{\s{C}}$, then the minimum eigenvalue of $\fl{V}_{\ca{S}}^{\s{T}}\fl{V}_{\ca{S}} \ge \alpha\gamma \s{C}/\s{M}$. To see this, note  $\fl{X}_{\mid \ca{S}}=\fl{U}\f{\Sigma}\fl{V}^{\s{T}}_{\ca{S}}$ implying that $\fl{V}^{\s{T}}_{ \ca{S}}=(\fl{U}\f{\Sigma})^{-1}\fl{X}_{\mid \ca{S}}$. Hence, $\fl{V}^{\s{T}}_{\ca{S}}\fl{V}_{\ca{S}} = (\fl{U}\f{\Sigma})^{-1}\fl{X}_{\mid \ca{S}}\fl{X}_{\mid \ca{S}}^{\s{T}}(\fl{U}\f{\Sigma})^{-\s{T}}$ implying that $(\fl{V}^{\s{T}}_{\ca{S}}\fl{V}_{\ca{S}})^{-1} = (\fl{U}\f{\Sigma})^{\s{T}}(\fl{X}_{\mid\ca{S}}\fl{X}_{\mid\ca{S}}^{\s{T}})^{-1}(\fl{U}\f{\Sigma})$. Taking the operator norm on both sides, we have  $\lambda_{\min}(\fl{V}^{\s{T}}_{\ca{S}}\fl{V}_{\ca{S}}) \ge \lambda_1^{-2} \lambda_{\min}(\fl{X}_{\mid \ca{S}}\fl{X}^{\mid \s{T}}_{\ca{S}})$ implying that $\fl{x}^{\s{T}}\fl{V}_{\ca{S}}^{\s{T}}\fl{V}_{\ca{S}}\fl{x} \ge \alpha\gamma\s{C}/\s{M}$.
\end{proof}

\begin{lemma}
Suppose the entries of $\fl{X}$ are generated independently according to $\ca{N}(0,1)$. Then, $\fl{X}$ with SVD decomposition $\fl{X}=\fl{U}\f{\Sigma}\fl{V}^{\s{T}}$ satisfies $\lr{\fl{V}}_{2,\infty} \le 16\sqrt{\frac{\s{C}\log \s{M}}{\s{M}}}$ with probability at least $1-O(\s{M}^{-1})$.
\end{lemma}

\begin{proof}
We must have $\sqrt{\s{M}}-\sqrt{\s{C}}-t\le \lambda_{\s{C}} \le \lambda_{1} \le \sqrt{\s{M}}+\sqrt{\s{C}}+t$ with probability at least $1-2e^{-t^2/2}$ implying that $\sqrt{\s{M}}/2 \le \lambda_{\s{C}} \le \lambda_1 \le 2\sqrt{\s{M}}$; hence we must have $\lambda_1/\lambda_{\s{C}}=O(1)$ with probability at least $1-O(e^{-\s{M}})$. Moreover, we have $\fl{X}^{\s{T}}\fl{X}=\fl{V}\f{\Sigma}^2\fl{V}^{\s{T}}$. Clearly, we must have $\lr{\fl{X}}_{\infty,2}\lambda_{\s{C}}^{-1} \le \lr{\fl{V}}_{2,\infty} \le \lr{\fl{X}}_{\infty,2}\lambda_1^{-1}$. For any column $\fl{X}_{\mid i}$, we have $\lr{\fl{\fl{X}_{\mid i}}}_2^2$ is a chi-squared random variable with $\s{C}$ degrees of freedom. Using standard  concentration inequalities for chi-squared random variables, we have $\lr{\fl{\fl{X}_{\mid i}}}_2 \le 8\sqrt{\s{C}\log \s{M}}$ w.p. at least $1-\s{M}^{-2}$. By taking a union bound over all $i\in [\s{M}]$, we have $\lr{\fl{X}}_{\infty,2} \le 8\sqrt{\s{C}\log \s{M}}$ w.p. at least $1-\s{M}^{-1}$. Hence $\lr{\fl{V}}_{2,\infty} \le 16\sqrt{\frac{\s{C}\log \s{M}}{\s{M}}}$.
\end{proof}

\begin{lemma}\label{lem:for_any}
Suppose the entries of $\fl{X}$ are generated independently according to $\ca{N}(0,1)$. Then for any subset of columns $\ca{S}\subseteq [\s{M}]$ satisfying $\ca{S}=\Omega(\s{C}\log \s{M})$, we must have that the minimum singular value of $\fl{X}_{\mid \ca{S}}\fl{X}_{\mid \ca{S}}^{\s{T}}$ is at least $\left|\ca{S}\right|/2$ with probability at least $1-2e^{-\Omega(\s{C}\log \s{M})}$.
\end{lemma}

\begin{proof}
On the other hand, for a subset $\ca{S}\subseteq [\s{M}], |\ca{S}|=\gamma\s{C}$, we must have the minimum singular value of $\fl{X}_{\mid \ca{S}}$ to be at least $\sqrt{\s{C}}(\sqrt{\gamma}-1)-t$ w.p. at least $1-2e^{-t^2/2}$. 
Taking $t=\sqrt{\s{C}\gamma}/2$, we must have the minimum singular value of $\fl{X}_{\mid \ca{S}}$ to be at least $\sqrt{\s{C}\gamma}/2$ with probability at least $1-2e^{-\Omega(\s{C}\log\s{M})}$. 
\end{proof}

Note that ideally, to handle the arbitrary sub-matrices $\fl{P}_{\s{sub}}$ that might arise in Line 4 of Algorithm \ref{algo:phased_elim}, we could have taken a union bound over all possible subsets of $[\s{M}]$ of size $\Omega(\s{C}\log \s{M})$ but the total number of subsets is too large. However, interestingly, the sub-matrices $\fl{P}_{\s{sub}}$ are not completely arbitrary as we show below:

\textit{Discussion:} Note that in Line 6 of Algorithm \ref{algo:phased_elim}, in the $\ell^{\s{th}}$ phase, we construct a set of good items $\ca{T}_u^{(\ell)}$ for users in a relevant nice subset $\ca{M}^{(\ell,i)}$ as follows:
we compute $\ca{T}^{(\ell)}_u \equiv \{j \in \ca{N}^{(\ell,i)} \mid \max_{j'\in \ca{N}^{(\ell,i)}}\widetilde{\fl{P}}^{(\ell)}_{uj'}-\widetilde{\fl{P}}^{(\ell)}_{uj} \le 2\Delta_{\ell+1}\}$ where $\ca{N}^{(\ell,i)}$ is the set of active items. In Line 8 of Algorithm \ref{algo:phased_elim}, at the end of the $\ell^{\s{th}}$ phase, we construct updated \textit{nice} subset of users  $\{\ca{M}^{(\ell,i,j)}\}$ as the connected components of a relevant graph. We also construct the corresponding active set of items $\{\ca{N}^{(\ell,i,j)}\}_j$ where $\ca{N}^{(\ell,i,j)} \equiv \cup_{u\in \ca{M}^{(\ell,i,j)}} \ca{T}_u^{(\ell)}$. In the $(\ell+1)^{\s{th}}$ phase, it is important that properties \textbf{P1} and \textbf{P2} are satisfied by each of the sub-matrices $\{\fl{P}_{\ca{M}^{(\ell,i,j)},\ca{N}^{(\ell,i,j)}}\}_j$.
In order to do so, we slightly expand the active set of items. Without loss of generality, consider $\ca{M}^{(\ell,i,1)}$ and denote it as $\ca{M}$ for brevity. Similarly, the set of active items constructed at the end of Line 8 for $\ca{M}^{(\ell,i,1)}$ is denoted by $\ca{N}$ for brevity.

We generalize the above definition of $\ca{T}_u^{(\ell)}$ (for a fixed user $u\in \ca{M}$, active set of items $\ca{N}^{(\ell,i)}$ at beginning of $\ell^{\s{th}}$ phase and parameter $\Delta_{\ell+1}$) in the following way:
\begin{align*}
    \ca{T}^{(\ell,a)}_u \equiv \{j \in \ca{N}^{(\ell,i)} \mid \max_{j'\in \ca{N}^{(\ell,i)}}\widetilde{\fl{P}}^{(\ell)}_{uj'}-\widetilde{\fl{P}}^{(\ell)}_{uj} \le a\Delta_{\ell+1}\}
\end{align*}
where $a>0$ is any positive constant. Furthermore, for a fixed $\Delta_{\ell+1}$, we also define 
\begin{align*}
    \ca{S}(u,a) \equiv \{j \in [\s{N}] \mid \max_{j'\in [\s{N}]}\fl{P}_{u\pi_u(1)}-\fl{P}_{uj} \le a\Delta_{\ell+1}\}
\end{align*}
Note that, conditioned on $\fl{P}_u$, the sets $\ca{S}(u,a)$ are deterministic sets. We can show the following lemmas: 

\begin{lemma}\label{lem:expec_max}[\cite{kamath2015bounds}]
Suppose we have $n$ independent random variables $x_1,x_2,\dots,x_n\sim \ca{N}(0,\sigma^2)$. In that case, we have 
\begin{align*}
    \frac{\sigma\sqrt{\log n}}{\sqrt{\pi\log 2}}\le \bb{E}[\max_{j\in [n]}x_j] \le \sigma\sqrt{2\log n}
\end{align*}
\end{lemma}

\begin{lemma}\label{lem:conc_max}[Borell-TIS inequality]
Suppose we have $n$ independent random variables $x_1,x_2,\dots,x_n\sim \ca{N}(0,\sigma^2)$. In that case, we have for all $t>0$
\begin{align*}
    \Pr\Big(\max_{j\in [n]}x_i-\bb{E}[\max_{j\in [n]}x_j] \ge t\Big) \le 2\exp(-t^2/2\sigma^2).
\end{align*}
\end{lemma}

\begin{lemma}\label{lem:minsing_feasible}
Assume that $\left|\ca{S}(u,a)\right|=\Omega(\log^4\s{M})$ and $a\Delta_{\ell+1}\le 1/2$. In that case, with probability at least $1-(\s{poly}(\s{M}))^{-1}$, the minimum singular value of $\fl{X}_{\mid \ca{S}(u,a)}$ is at least $c\left|\ca{S}(u,a)\right|$ for some constant $c>0$.
\end{lemma}
\begin{proof}
Denote the minimum singular value of $\fl{X}_{\mid \ca{S}(u,a)}$ by $\lambda_{\s{C}}$.
Recall that $\fl{X}_{\mid i}$ denotes the $i^{\s{th}}$ column of $\fl{X}$, $\fl{X}_{j,i}$ denotes the entry in the $j^{\s{th}}$ row and $i^{\s{th}}$ column. $\fl{X}_{\setminus j, i}$ denotes the $i^{\s{th}}$ column of $\fl{X}$ without the $j^{\s{th}}$ entry.
Without loss of generality, let us assume that the user $u$ belongs to the first cluster i.e. $\fl{P}_{ui}=\fl{X}_{1i}$ for all $i\in [\s{M}]$. Note that although $\fl{X}$ is a random matrix the set $\ca{S}(u,a)$ depends on the values of $\fl{X}$. Therefore, we condition on the first row of $\fl{X}$ and the event $\ca{E}$ that $\min_{u} \max_{j} \fl{P}_{uj} \ge \sqrt{\log {\s{M}}}/3$ with  probability at least $1-\s{poly}(\s{M})$ (by combining Lemmas \ref{lem:expec_max},\ref{lem:conc_max}).
In other words, we condition on a particular instance of the random variable $\fl{X}_{1,i}$ that is, in the following analysis we consider $\fl{X}_{1,i}$ to be fixed for all $i\in \ca{S}(u,a)$ - note that such a conditioning does not provide any information about the random variables $\fl{X}_{\setminus 1, \ca{S}(u,a)}$ (rows other than the first row in the matrix $\fl{X}$ restricted to columns in $\ca{S}(u,a)$).
By definition of the minimum singular value, we have 
\begin{align*}
    \lambda_{\s{C}}^2 &= \inf_{\fl{w}:\lr{\fl{w}}_2=1} \sum_{i \in \ca{S}(u,a)} \fl{w}^{\s{T}}\fl{X}_{\mid i}\fl{X}_{\mid i}^{\s{T}}\fl{w} \\
    &= \inf_{\fl{w}:\lr{\fl{w}}_2=1} \sum_{i \in \ca{S}(u,a)} (\fl{w}_{1}\fl{X}_{1,i}+\fl{w}_{\setminus 1}\fl{X}_{\setminus 1,i})^2 \\
    &= \inf_{\fl{w}:\lr{\fl{w}}_2=1} \sum_{i \in \ca{S}(u,a)} (\fl{w}_{1}\fl{X}_{1,i})^2+\inf_{\fl{w}:\lr{\fl{w}}_2=1} \sum_{i \in \ca{S}(u,a)} (\fl{w}_{\setminus 1}\fl{X}_{\setminus 1,i})^2+\inf_{\fl{w}:\lr{\fl{w}}_2=1} \sum_{i \in \ca{S}(u,a)} 2(\fl{w}_{1}\fl{X}_{1,i})(\fl{w}_{\setminus 1}\fl{X}_{\setminus 1,i})
\end{align*}
 Now, we consider each of the three terms above: we have 
\begin{align*}
    \inf_{\fl{w}:\lr{\fl{w}}_2=1} \sum_{i \in \ca{S}(u,a)} (\fl{w}_{1}\fl{X}_{1,i})^2 \overset{(a)}{\ge} \fl{w}_1^2 \left|\ca{S}(u,a)\right| \log \s{M} \cdot \frac{1}{36}.
\end{align*}
(a) This is because, conditioned on the event $\ca{E}$ and the fact $a\Delta_{\ell+1}<\frac{1}{2}$, we must have  $\min_{j \in \ca{S}_{(u,a)}} \fl{P}_{uj} \ge \sqrt{\log \s{M}}/6$. 

Next, we look at the second term which corresponds to the square of minimum singular value of the matrix $\fl{X}_{\setminus 1,\ca{S}(u,a)}$ which, normalized by $\lr{\fl{w}_{\setminus 1}}_2$ is a random Gaussian matrix of dimensions $\s{C}-1 \times \left|\ca{S}(u,a)\right|$.
Therefore, by using standard tools from random matrix theory, we must have that 
\begin{align*}
    \inf_{\fl{w}:\lr{\fl{w}}_2=1} \sum_{i \in \ca{S}(u,a)} (\fl{w}_{\setminus 1}\fl{X}_{\setminus 1,i})^2 = \inf_{\fl{w}:\lr{\fl{w}}_2=1} \lr{\fl{w}_{\setminus 1}}_2^2 \sum_{i \in \ca{S}(u,a)} (\frac{\fl{w}^{\s{T}}_{\setminus 1}}{\lr{\fl{w}_{\setminus 1}}_2}\fl{X}_{\setminus 1,i})^2 \ge \frac{\lr{\fl{w}_{\setminus 1}}_2^2\left|\ca{S}(u,a)\right|}{10}
\end{align*}
with probability $1-\s{poly}(\s{M})$ provided $\left|\ca{S}(u,a)\right|=\Omega(\log \s{M})$. Now, we consider the third term corresponding to the sums of inner products
\begin{align*}
    \ca{T}_w \triangleq  \sum_{i \in \ca{S}(u,a)} 2(\fl{w}_{1}\fl{X}_{1,i})(\fl{w}_{\setminus 1}^{\s{T}}\fl{X}_{\setminus 1,i}) = \sum_{j \in \s{C}\setminus \{1\}}2(\fl{w}_{1}\fl{w}_j)\underbrace{\sum_{i \in \ca{S}(u,a)}\fl{X}_{1,i}\fl{X}_{j,i}}_{\triangleq \ca{Y}_{j}} 
\end{align*}
Clearly, due to the randomness in $ \fl{X}_{\setminus 1,i}$, we have $\bb{E}\ca{Y}_j=0$.

Note that for any $j\in \s{C}\setminus\{1\}$, we will have $\sum_{i \in \ca{S}(u,a)}\fl{X}_{1,i}\fl{X}_{j,i}\sim \ca{N}(0,\sum_{i\in \ca{S}(u,a)}\fl{X}_{1,i}^2)$ and by standard Gaussian tail bounds, we have that $\sum_{i \in \ca{S}(u,a)}\fl{X}_{1,i}\fl{X}_{j,i}\le 10\sqrt{\sum_{i\in \ca{S}(u,a)}\fl{X}_{1,i}^2}\log \s{M}$ with probability at least $1-(\s{poly}(\s{M}))^{-1}$. Therefore by taking a union bound over all $j\in \s{C}\setminus \{1\}$, we have that 
\begin{align*}
    \left|\ca{T}_w\right| &\le 2\left|\fl{w}_1\right|(\sum_{j}\left|\fl{w}_j\right|)\max_j \left|\sum_{i \in \ca{S}(u,a)}\fl{X}_{1,i}\fl{X}_{j,i}\right| \le  10\left|\fl{w}_1\right|(\sum_{j}\left|\fl{w}_j\right|)\sqrt{\sum_{i\in \ca{S}(u,a)}\fl{X}_{1,i}^2}\log \s{M} \\
    &\le 20\sqrt{\s{C} \left|\ca{S}(u,a)\right|} |\fl{w}_1| \lr{\fl{w}_{\setminus 1}}_2 \log^2 (\s{MC}) \le 20\sqrt{\s{C} \left|\ca{S}(u,a)\right|}  \log^2 (\s{MC}).
\end{align*}
Here we used Cauchy Schwarz inequality to say that $\lr{\fl{w}_{\setminus 1}}_1 \le \sqrt{\s{C}} \lr{\fl{w}_{\setminus 1}}_2$. Furthermore, we also used that with probability at least $1-\s{poly}(\s{M})$, we have that $\max_{1,i} \fl{X}_{1,i}^2 \le 4\log \s{MC}$ - hence, this implies $\sum_{i\in \ca{S}(u,a)}\fl{X}_{1,i}^2 \le \sqrt{2}\left|\ca{S}(u,a)\right|\log (\s{MC}) $. Note that the above inequality holds for any unit norm vector $\fl{w}$.
Thus, combining all of these, we can conclude that provided $\left|\ca{S}(u,a)\right|=\Omega(\log ^4 \s{M})$, we must have for some constant $c>0$.
\begin{align*}
    \lambda_{\s{C}}^2 \ge c\left|\ca{S}(u,a)\right| \text{ with probability at least } 1-(\s{poly}(\s{M}))^{-1}.
\end{align*}
with probability at least $1-\s{poly}(\s{M})$. Note that the above statement is taken after conditioning on the first row of $\fl{X}$ restricted to the columns in $\ca{S}(u,a)$ and by invoking a union bound on the maximum value of the entire matrix $\fl{X}$ and the minimum singular value of the matrix $\fl{X}_{\setminus 1, \ca{S}(u,a)}$. Therefore the lower bound on $\lambda_{\s{C}}$ holds for all possible realizations of the first row of $\fl{X}$ provided the high probability events involving the sub-matrix of $\fl{X}$ restricted to columns in $\ca{S}(u,a)$ hold true. Finally, we do the same analysis for all rows of $\fl{X}$ - that is, take another union bound over all clusters $\s{C}$ to arrive at the statement of the lemma.
\end{proof}

Next, we have the following tail bounds for a Gaussian random variable $x\sim \ca{N}(0,1)$:

\begin{align}\label{eq:gaussian_tail}
\frac{1}{\sqrt{\Delta^2}+\sqrt{\Delta^2+16}}\sqrt{\frac{8}{\pi}}\exp\Big(-\frac{\Delta^2}{8}\Big) \le \Pr(x\ge \frac{\Delta}{2}) \le \frac{1}{\sqrt{\Delta^2}+\sqrt{\Delta^2+\frac{32}{\pi}}}\sqrt{\frac{8}{\pi}}\exp\Big(-\frac{\Delta^2}{8}\Big)    
\end{align}

For simplicity, for any $x$, we will use the notation $x\approx t$ to imply that $x\in [c_1 t, c_2 t]$ for some constants $c_1,c_2$. From Lemmas \ref{lem:expec_max} and \ref{lem:conc_max}, we can conclude the following lemma:

\begin{lemma}\label{lem:gaussian_max}
For all $j\in [\s{C}]$, we must have that $\max_{i\in [\s{M}]}\fl{X}_{j,i} \approx \sqrt{\log \s{M}}$ that is,  $\max_{i\in [\s{M}]}\fl{X}_{j,i}\in [c_1\sqrt{\log \s{M}},c_2\sqrt{\log \s{M}}]$ for some constants $c_1\le c_2$ with probability at least $1-\s{poly}(\s{M})$.
\end{lemma}

Next, we will show the following result:

\begin{lemma}\label{lem:lb_feasible}
Suppose $1/\s{M} \le a\Delta_{\ell+1}<1/3\sqrt{c_2^2\log \s{M}}$. For all users $u \in [\s{N}]$, with probability at least $1-\s{poly}(\s{M})$, we have $\left|\ca{S}(u,a\s{C})\right|=O(\log {\s{M}})\left|\ca{S}(u,a)\right|$.
\end{lemma}

\begin{proof}
Let us fix a user $u\in [\s{N}]$. Again, without loss of generality, let us assume that the user $u$ belongs to the first cluster i.e. $\fl{P}_{ui}=\fl{X}_{1,i}$ for all $i\in [\s{M}]$.
To prove the statement of the lemma, we will discretize the interval $[c_1\sqrt{\log \s{M}},\sqrt{(c_2^2+1)\log \s{M}}]$ into a grid $\ca{F}$ with equally spaced points with spacing $a\Delta_{\ell+1}/4$  where $c_1,c_2$ is defined in Lemma \ref{lem:gaussian_max}. For any point $\sqrt{t\log \s{M}}\in \ca{F}$ (for some constant $t\in [c_1^2,c_2^2+1]$), we must have for a gaussian random variable $x\in \ca{N}(0,1)$ (see equation \ref{eq:gaussian_tail}),
\begin{align*}
    \Pr(x \ge \sqrt{t\log \s{M}}) \approx \frac{1}{t\s{M}^{t/2}\log \s{M}}.
\end{align*}
Furthermore, we will also have for some constant $c'$
\begin{align*}
\sqrt{t\log \s{M}-c'}  \le \sqrt{t\log \s{M}-2a\Delta_{\ell+1}\sqrt{t\log \s{M}}+(a\Delta_{\ell+1})^2} = \sqrt{t\log \s{M}}-a\Delta_{\ell+1}.    
\end{align*}
Therefore, we will also have (since $\s{C}$ is a constant) - see equation \ref{eq:gaussian_tail}
\begin{align}\label{eq:probs}
   & \Pr(x \ge \sqrt{t\log \s{M}-a\Delta_{\ell+1}}) \approx  \frac{1}{t\s{M}^{t/2}\log \s{M}} \\
   &\text{ and similarly }\Pr(\sqrt{t\log \s{M}} \ge x \ge \sqrt{t\log \s{M}-2a\s{C}\Delta_{\ell+1}}) \approx  \frac{1}{t\s{M}^{t/2}\log \s{M}}.
\end{align}
Therefore, for any constant $t\in [c_1^2,c_2^2+1]$, we will have that 
\begin{align*}
    \Pr\Big(\sqrt{t\log \s{M}}\le x \le \sqrt{t \log \s{M}-2a\s{C}\Delta_{\ell+1}}\Big) \approx \Pr\Big(\sqrt{t\log \s{M}}\le x \le \sqrt{t \log \s{M}-a\Delta_{\ell+1}}\Big).
\end{align*}

Now consider, $\s{M}$ independent gaussian random variables $x_1,x_2,\dots,x_{\s{M}}\sim \ca{N}(0,1)$.
At this point, we can use the multiplicative version of the Chernoff bound which says the following: for independent random variables $x_1,x_2,\dots,x_n$ that take values in $\{0,1\}$, we have for any $\delta>0$, 
\begin{align*}
    \Pr(\sum_i x_i -\bb{E}\sum_i x_i \ge \delta \bb{E}\sum_i x_i) \le 2\exp\Big(-\delta^2(\bb{E}\sum_i x_i)/3\Big).
\end{align*}
By using the multiplicative Chernoff bound (substituting $\delta=\sqrt{\log \s{M}}$ for $\bb{E}\sum_i x_i=O(\sqrt{\log \s{M}})$ and $\delta=1/2$ otherwise), we can conclude that with probability $1-\s{poly}(\s{M})$, we have the following (we also use the fact that if the expected sum in the multiplicative chernoff inequality is $o(1/\s{M})$, then it is dominated by a set of independent random variables $x_1',x_2'\dots,x_\s{M}'\in \{0,1\}$ such that $\bb{E}\sum_i x_i'= \Theta(1)$):
\begin{align}\label{eq:bound_1}
    &\text{ for all $t\in \ca{F}$, }\sum_{i\in [\s{M}]} \mathds{1}[x_i \in [\sqrt{t\log \s{M}},\sqrt{t\log \s{M}-2a\s{C}\Delta_{\ell+1}}]] \\
    &= O(\sqrt{\log \s{M}}) \bb{E}\sum_{i\in [\s{M}]}\mathds{1}[x_i \in [\sqrt{t\log \s{M}},\sqrt{t\log \s{M}-2a\s{C}\Delta_{\ell+1}}]] \nonumber
\end{align}
and similarly, when $\bb{E}\sum_{i\in [\s{M}]}\mathds{1}[x_i \in [\sqrt{t\log \s{M}},\sqrt{t\log \s{M}-a\Delta_{\ell+1}}]]=\Omega(\sqrt{\log\s{M}})$, we have
\begin{align}\label{eq:bound_2}
    &\text{ for all $t\in \ca{F}$, }\sum_{i\in [\s{M}]} \mathds{1}[x_i \in [\sqrt{t\log \s{M}},\sqrt{t\log \s{M}-a\Delta_{\ell+1}}]] \\
    &= \Omega(1) \bb{E}\sum_{i\in [\s{M}]}\mathds{1}[x_i \in [\sqrt{t\log \s{M}},\sqrt{t\log \s{M}-a\Delta_{\ell+1}}]] \nonumber
\end{align}
Let us define the event $\ca{G}$ when equations \ref{eq:bound_1} and \ref{eq:bound_2} are true for all $t\in \ca{F}$.
In that case, condition on events that $\max_{i}\fl{X}_{1,i}\in [c_1\sqrt{\log \s{M}},c_2\sqrt{\log \s{M}}]$ and the event $\ca{G}$ related to the intervals formed by discretizing the range of the max reward value is true. 
Suppose $t^{\star}\in \ca{F}=\s{argmin}_{t \le \ca{F}} \mathds{1}[t \ge \max_i \fl{X}_{1,i}]$ is the smallest value in the grid $\ca{F}$ larger than the maximum value in the first row of $\fl{X}$.
In that case, note that by definition, $\ca{S}(u,a)$ must have the following property
\begin{align*}
    \ca{S}(u,a)\triangleq \{j \in [\s{M}] \mid \max_{j'\in [\s{N}]}\fl{X}_{1,j'}-\fl{X}_{1,j} \le a\Delta_{\ell+1}\} \supseteq \{j\in [\s{M}]\mid \fl{X}_{1,j} \in \{t^{\star},t^{\star}-a\Delta_{\ell+1}\}\}.
\end{align*}
This is because, by definition, $t^{\star}$ lies to the right of $\max_{1,i}\fl{X}_{1,i}\equiv \fl{P}_{u\pi_u(1)}$. Similarly, we will have 
\begin{align*}
     \ca{S}(u,a\s{C})\subseteq \{j\in [\s{M}]\mid \fl{X}_{1,j} \in \{t^{\star},t^{\star}-2a\s{C}\Delta_{\ell+1}\}\}.
\end{align*}
since $\left|t^{\star}-\max_{1,i}\fl{X}_{1,i}\right|\le a\Delta_{\ell+1}/4$ due to the construction of our grid. Note that the sets $\ca{S}(u,a)$ and $\{j\in [\s{M}]\mid \fl{X}_{1,j} \in \{t^{\star},t^{\star}-a\Delta_{\ell+1}\}\}$ should have a size of at least $1$ - since the grid spacing is $a\Delta_{\ell+1}/4$ implying that both sets must contain the element $\max_{j'\in [\s{N}]}$. We use this fact in the special case when $\bb{E}\sum_{i\in [\s{M}]}\mathds{1}[x_i \in [\sqrt{t^{\star}\log \s{M}},\sqrt{t^{\star}\log \s{M}-a\Delta_{\ell+1}}]]=O(\sqrt{\log\s{M}})$ and equation \ref{eq:bound_2} does not hold. However, in this special case, we will have $\left|\{j\in [\s{M}]\mid \fl{X}_{1,j} \in \{t^{\star},t^{\star}-2a\s{C}\Delta_{\ell+1}\}\}\right|=O(\log \s{M})$ (by using equations \ref{eq:probs} and \ref{eq:bound_1}). 
Since, otherwise 
\begin{align*}
    \left|\{j\in [\s{M}]\mid \fl{X}_{1,j} \in \{t^{\star},t^{\star}-2a\s{C}\Delta_{\ell+1}\}\}\right| \approx O(\sqrt{\log \s{M}}) \left|\{j\in [\s{M}]\mid \fl{X}_{1,j}\in \{t^{\star},t^{\star}-a\Delta_{\ell+1}\}\}\right|
\end{align*}
we must have that $
\left|\ca{S}(u,a\s{C})\right|= O(\log \s{M})\left|\ca{S}(u,a)\right|$. The failure probability for this event is $1/\s{poly}(\s{M})$.

\end{proof}

\begin{coro}\label{coro:feasibility}
Fix any user $u\in [\s{N}]$. Assume that $\left|\ca{S}(u,a)\right|=\Omega(\log^4\s{M})$ and $a\Delta_{\ell+1}\le 1/3\sqrt{\log \s{M}}$.
Consider any subset of columns $\ca{J}\subseteq [\s{M}]$ such that $\ca{S}(u,a\s{C}) \supseteq \ca{J}\supseteq \ca{S}(u,a)$. In that case, with probability at least $1-(\s{poly}(\s{M}))^{-1}$, we will have the minimum singular value of $\fl{X}_{\mid \ca{J}}$ to be $\Omega(\left|\ca{J}\right|/\sqrt{\log \s{M}})$.
\end{coro}

\begin{proof}
The proof follows from Lemmas \ref{lem:lb_feasible}, \ref{lem:minsing_feasible} and the fact that the minimum singular value of $\fl{X}_{\mid \ca{J}}$ is larger than the minimum singular of $\fl{X}_{\mid \ca{S}(u,a)}$ by definition as $\ca{J}\supseteq \ca{S}(u,a)$. 
\end{proof}

Finally, we take a union bound over the event in Lemma \ref{lem:lb_feasible} over all the phases (at most $O(\log \s{T})$). Note that the smallest error tolerance remains above $1/\s{M}$ when $\s{T}\ll \s{M}$. Hence, there will be no issue in applying Lemma \ref{lem:lb_feasible}.

\paragraph{With high probability, we only estimate sub-matrices of $\fl{P}$ in Step 4 of Algorithm \ref{algo:phased_elim}:}
Consider any user $u\in \ca{M} \equiv \ca{M}^{(\ell,i)}$ with active items $\ca{N} $ (equivalently $\ca{N}^{(\ell,i)}$) and any phase $\ell$.
Condition on the events $\ca{E}^{(\ell)},\ca{E}_2^{(\ell)}$ (defined rigorously in Section \ref{sec:detailed_proof} and implying that the matrix completion steps in every phase of Algorithm \ref{algo:phased_elim} has been successful till the end of phase $\ell$) and  $\ca{N}^{(\ell,i)}\supseteq \ca{S}(u,a+2)$. In that case, we have $\ca{T}_u^{(\ell,a)}\supseteq \ca{S}(u,a)$. Furthermore, we will also have that $\ca{T}_u^{(\ell,a)} \subseteq \ca{N}\subseteq 
\ca{S}(u,a\s{C})$. 
The previous statement follows from standard applications of the triangle inequality (see for instance Lemma \ref{lem:interesting5} and its proof). We choose $a$ to be $3$ for all phases and in Corollary \ref{coro:feasibility}, we take a union bound over all the phases (and corresponding fixed $\Delta_{\ell+1}$) and possible sub-matrices. Also, note that the error tolerance in  phase $\ell+1$ is decreased by a factor of $\s{C}$ - since the set $\ca{N}$ contains $\ca{S}(u,a)$ for all users $u\in \ca{M}$, the subsequent phase can also have a similar property.

Note that with high probability, in each phase, each sub-matrix of the reward matrix that we need to estimate corresponds only to a \textit{nice subset} of users - therefore the total number of matrices that we need to take a union bound on is at most $O(2^{\s{C}}\cdot \log \s{T})$.
Hence, with high probability in phase $\ell$, all sub-matrices (to be estimated) restricted to the set of users $\ca{M}$ and $\ca{N}$ satisfies the following:
\begin{enumerate}
        \item \textbf{(P1:)} The condition number of the matrix $\fl{P}_{\ca{M},\ca{N}}$ that is, the ratio of the maximum and minimum non-zero singular value is bounded from above by a constant. This is proved by using Assumption \ref{assum:cluster_ratio} and Lemma \ref{lem:condition_num}.
    \item \textbf{(P2:)} The orthonormal matrices $\widetilde{\fl{U}},\widetilde{\fl{V}}$ are incoherent i.e. we will have $\lr{\widetilde{\fl{U}}}_{2,\infty} =O(\sqrt{\mu /\left|\ca{M}\right|})$ and $\lr{\widetilde{\fl{V}}}_{2,\infty} =O(\sqrt{\mu /\left|\ca{N}\right|})$ with $\mu=\Omega(1/\log \s{M})$. This follows from using Corollary \ref{coro:feasibility} and the proof of Lemma \ref{lem:incoherence}. 
\end{enumerate}

Hence we can proceed with the rest of the analysis as presented in Section \ref{sec:detailed_proof}.

\section{Missing Details in Section \ref{subsec:mc}}\label{app:mc}
\paragraph{Low Rank Matrix Completion algorithm.} Algorithm~\ref{algo:estimate} describes the low-rank matrix completion algorithm we use in our work. This algorithm is adapted from \cite{jain2022online}, with minor modifications that are necessary for our setting. At its core, the algorithm solves a nuclear norm regularized convex objective to complete the matrix (equation~\eqref{eq:convex}). This procedure is repeated $f = O(\log{\s{NMT}})$ times, and the final estimate of the matrix is computed as the entry-wise median of the $f$ solutions. 

Algorithm~\ref{algo:data_collection_matrix_estimation} collects the data needed for matrix completion in Equation~\eqref{eq:convex}. At a high level, this subroutine randomly selects entries in the matrix (each entry is selected with probability $p$). It then computes an estimate of each of the selected entries by pulling the arm corresponding to the entry multiple times ($b$ times) and taking an average of the obtained rewards. The collected data is then shared with Algorithm~\ref{algo:estimate} for matrix completion. The main difficulty in implementing this algorithm is that in MAB-LC, the users arrive randomly in each iteration. Consequently, one has to wait for the required users to arrive to collect the necessary data. The question now is, how long does the algorithm wait to collect all the necessary data? By mapping this problem to the popular \emph{Coupon Collector Problem}, it can be show that the the sample complexity of the algorithm increases atmost by $\log$ factors.
\begin{algorithm*}[!t]
\caption{\textsc{Low Rank Matrix Estimate} (Adapted with slight modifications from \cite{jain2022online})   \label{algo:estimate}}
\begin{algorithmic}[1]

\REQUIRE users $\ca{U}\subseteq[\s{N}]$, arms $\ca{V}\subseteq[\s{M}]$, rank $r$, incoherence $\mu$ of reward matrix $\fl{P}_{\ca{U},\ca{V}}$, total rounds $\s{T}$, noise variance $\sigma^2$, desired error in estimate $\zeta$.

\STATE Set $d_2=\min(\left|\ca{U}\right|,\left|\ca{V}\right|)$, sampling probability $p=C\mu^2 d_2^{-1}\log^3 d_2$, variance reduction factor $b=\Big\lceil \Big(\frac{c\sigma r \sqrt{\mu}}{\zeta\log d_2}\Big)^2 \Big\rceil$, number of repetitions of algorithm $f=O(\log \s{MNT})$, regularization parameter $\lambda = C_{\lambda}\sigma \sqrt{d_2 p}$ for suitable constants $c,C, C_{\lambda}>0$.
\FOR{$k=0,1,2,\dots,f$} 
 \STATE $\fl{Z} \leftarrow$ \textsc{DataCollectionSubRoutine}$(\ca{U}, \ca{V}, p, b)$

\STATE Without loss of generality, assume $|\ca{U}| \le |\ca{V}|$ (if $|\ca{U}| \ge |\ca{V}|$, we simply swap rows and columns). Randomly partition the columns into $k=\lceil|\ca{V}|/|\ca{U}|\rceil$ sets. Precisely, for each $i\in \ca{V}$, independently set $\delta_i$ to be a value in the set $\{1,2 \dots k\}$ uniformly at random. Partition indices in $\ca{V}$ into $\ca{V}^{(1)},\ca{V}^{(2)},\dots,\ca{V}^{(k)}$ where $\ca{V}^{(q)}=\{i\in \ca{V}\mid \delta_i =q\}$ for each $q\in [k]$. Set $\Omega^{(q)}\leftarrow \Omega \cap (\ca{U}\times \ca{V}^{(q)})$ for all $q\in [k]$. 

\FOR{$q\in [k]$}
\STATE Solve convex program \vspace*{-15pt}
\begin{align}\label{eq:convex}
    \min_{\fl{Q}^{(q)}\in \bb{R}^{|\ca{U}|\times|\ca{V}^{(q)}|}} \frac{1}{2}\sum_{(i,j)\in \Omega^{(q)}}\Big(\fl{Q}^{(q)}_{i\pi(j)}-\fl{Z}_{ij}\Big)^2+\lambda\|\fl{Q}^{(q)}\|_{\star},
\end{align}
where $\|\fl{Q}^{(q)}\|_{\star}$ denotes nuclear norm of matrix $\fl{Q}^{(q)}$ and $\pi(j)$ is index of $j$ in set $\ca{V}^{(q)}$. 
\ENDFOR
\STATE Compute matrix $\widetilde{\fl{P}}^{(k)}\in \bb{R}^{\s{N}\times \s{M}}$ such that $\widetilde{\fl{P}}^{(k)}_{\ca{U},\ca{V}^{(q)}}=\fl{Q}^{(q)}$ for all $q\in [k]$ and for every index $(i,j)\not \in \ca{U}\times \ca{V}$, $\widetilde{\fl{P}}^{(k)}_{ij}=0$. \#\textit{We combine the estimates of each of the smaller matrices to form an estimate of the larger matrix i.e. without partitioning the columns. Moreover, $\{\widetilde{\fl{P}}_{\ca{U},\ca{V}}^{(k)}\}_{k=1}^{f}$ correspond to $f$ independent estimates of $\fl{P}_{\ca{U},\ca{V}}$.}
\ENDFOR
\RETURN entry-wise median of $\{\widehat{\fl{P}}^{(1)},\widehat{\fl{P}}^{(2)},\dots,\widehat{\fl{P}}^{(f)}\}$.

\end{algorithmic}
\end{algorithm*}

\begin{algorithm*}[!t]
\caption{\textsc{DataCollectionSubRoutine}   \label{algo:data_collection_matrix_estimation}}
\begin{algorithmic}[1]
\REQUIRE users $\ca{U}\subseteq[\s{N}]$, arms $\ca{V}\subseteq[\s{M}]$, sampling probability $p$, variance reduction factor $b$.
\STATE For each tuple $(u,v)\in \ca{U}\times \ca{V}$, independently set $\delta_{uv}=1$ with probability $p$ and $\delta_{uv}=0$ with probability $1-p$. Let $\Omega\subseteq \ca{U}\times \ca{V}$ to be the set of indices with $\delta_{uv}=1$. 
\FOR{$\ell=1,2,\dots,b$} 
\STATE For all $(i,j)\in \Omega$, set $\s{Mask}_{ij}=0$.
\WHILE{there exists $(i,j)\in \Omega$ such that $\s{Mask}_{ij}=0$}
\STATE WAIT UNTIL next round $t$  such that some $u(t) \in \ca{U}$ is sampled.
\STATE For user $u(t)\in \ca{U}$, pull arm $\rho(t)$ in $\{j \in \ca{V}\mid (u(t),j)\in \Omega, \s{Mask}_{u(t)j}=0\}$ and set $\s{Mask}_{u(t)\rho(t)}=1$. 
If not possible then pull any arm $\rho(t)$ in $\ca{V}$ such that $(u(t),\rho(t))\not \in \Omega$. Observe $\fl{R}^{(t)}$. 
\ENDWHILE
\ENDFOR

\STATE Initialize a matrix $\fl{Z}$ of size $|\ca{U}|\times |\ca{V}|$ to 0's.
\STATE 
For each tuple $(i,j)\in \Omega$, set $\fl{Z}_{ij}$ to be average of  rewards collected for that entry.
\RETURN $\fl{Z}$
\end{algorithmic}
\end{algorithm*}

\begin{algorithm*}[!t]
\caption{\textsc{Upper Confidence Bound (UCB)} \label{algo:estimate_ucb} \cite{bubeck2012regret}}
\begin{algorithmic}[1]

\REQUIRE user $u\in [\s{N}]$, set of arms $\ca{V}\subseteq[\s{M}]$, error $\zeta$, noise $\sigma^2$, total rounds $\s{T}$.

\STATE WAIT UNTIL round $t$ when user $u(t)$ is sampled. 
\STATE Choose arm $j_t \in \ca{V}$ according to $j_t = \s{argmax}_{j \in \ca{V}} \s{UCB}_j(s-1,\s{T}^{-3})$ where $s-1$ is the number of rounds ran in this instantiation of the UCB algorithm so far and (suppose $t_j(s-1)$ is the number of times arm $j\in \ca{V}$ has been pulled by $u$ in the previous $s-1$ rounds and $\widehat{\mu}_j(s-1)$ is the empirical mean of the arm $j$ due to the previous $t_j(s-1)$ observations)
\begin{align*}
    \s{UCB}_j(s-1,\s{T}^{-3}) = 
    \begin{cases}
    \infty  \quad \text{ if arm $j$ has not been played before} \\
    \widehat{\mu}_j(s-1)+\sigma\sqrt{\frac{6\log \s{T}}{t_j(s-1)}}
    \end{cases}
\end{align*}

\end{algorithmic}
\end{algorithm*}


\subsection{Proof of Lemma~\ref{lem:incoherence}}
\begin{proof}[Proof of Lemma \ref{lem:incoherence}]

Suppose the reward matrix $\fl{P}$ has the SVD decomposition $\fl{U}\f{\Sigma}\fl{V}^{\s{T}}$. We are looking at a sub-matrix of $\fl{P}$ denoted by $\fl{P}_{\s{sub}}\in \bb{R}^{\s{N}'\times \s{M}'}$ which can be represented as $\fl{U}_{\s{sub}}\f{\Sigma}\fl{V}_{\s{sub}}^{\s{T}}$ where $\fl{U}_{\s{sub}},\fl{V}_{\s{sub}}$ are sub-matrices of $\fl{U},\fl{V}$ respectively and are not necessarily orthogonal. Here, the rows in $\fl{P}_{\s{sub}}$ corresponds to some cluster of users $\s{C}^{(a)}$ for $a\in \s{C}$ and the columns in $\fl{P}_{\s{sub}}$ corresponds to the trimmed set of arms in $[\s{M}]$.
Suppose the rows of the sub-matrix $\fl{P}_{\s{sub}}$ correspond to the users in $\s{C}'\le \s{C}$ clusters. Note that $\fl{U}_{\s{sub}}$ can be further represented as $\fl{U}^{(1)}_{\s{sub}}\fl{U}^{(2)}_{\s{sub}}$ where $\fl{U}^{(1)}_{\s{sub}}\in \bb{R}^{\s{N}'\times \s{C}'}$ is a binary matrix with orthonormal columns and $1$-sparse rows (the non-zero index with value $1/\sqrt{\text{cluster size}}$ in the row indicates the cluster of the user). $\fl{U}^{(2)}_{\s{sub}}\in \bb{R}^{\s{C}'\times \s{C}'}$ indicate the distinct rows of $\fl{U}_{\s{sub}}$ corresponding to each of the $\s{C}$' clusters (multiplied by $\sqrt{\text{cluster size}}$). 

Hence we can write (provided $\fl{V}_{\s{sub}}$ is invertible) 
\begin{align*}
    \fl{P}_{\s{sub}} = \fl{U}^{(1)}_{\s{sub}}\fl{U}^{(2)}_{\s{sub}}\f{\Sigma}(\fl{V}_{\s{sub}}^{\s{T}}\fl{V}_{\s{sub}})^{1/2}(\fl{V}_{\s{sub}}^{\s{T}}\fl{V}_{\s{sub}})^{-1/2}\fl{V}_{\s{sub}}^{\s{T}} = \fl{U}^{(1)}_{\s{sub}}\widehat{\fl{U}}\widehat{\fl{\Sigma}}\widehat{\fl{V}}^{\s{T}}(\fl{V}_{\s{sub}}^{\s{T}}\fl{V}_{\s{sub}})^{-1/2}\fl{V}_{\s{sub}}^{\s{T}} 
\end{align*}
where $\widehat{\fl{U}}\widehat{\fl{\Sigma}}\widehat{\fl{V}}^{\s{T}}$ is the SVD of the matrix $\fl{U}^{(2)}_{\s{sub}}\f{\Sigma}(\fl{V}_{\s{sub}}^{\s{T}}\fl{V}_{\s{sub}})^{1/2}$. 
Since $\widehat{\fl{U}}$ is orthogonal, $\fl{U}^{(1)}_{\s{sub}}\widehat{\fl{U}}$ is orthogonal as well. 
Similarly, $\widehat{\fl{V}}^{\s{T}}(\fl{V}_{\s{sub}}^{\s{T}}\fl{V}_{\s{sub}})^{-1/2}\fl{V}_{\s{sub}}^{\s{T}}$ is orthogonal as well whereas $\widehat{\f{\Sigma}}$ is diagonal. Hence $\fl{U}^{(1)}_{\s{sub}}\widehat{\fl{U}}\widehat{\fl{\Sigma}}\widehat{\fl{V}}^{\s{T}}(\fl{V}_{\s{sub}}^{\s{T}}\fl{V}_{\s{sub}})^{-1/2}\fl{V}_{\s{sub}}^{\s{T}}$ indeed corresponds to the SVD of $\fl{P}_{\s{sub}}$ and we only need to argue about the incoherence of $\fl{U}^{(1)}_{\s{sub}}\widehat{\fl{U}}$ and $(\widehat{\fl{V}}^{\s{T}}(\fl{V}_{\s{sub}}^{\s{T}}\fl{V}_{\s{sub}})^{-1/2}\fl{V}_{\s{sub}}^{\s{T}})^{\s{T}}$. Notice that $\max_i \|(\fl{U}^{(1)}_{\s{sub}}\widehat{\fl{U}})^{\s{T}}\fl{e}_i\|_2 \le \|(\fl{U}^{(1)}_{\s{sub}})^{\s{T}}\fl{e}_i\|_2 \le 1/\sqrt{\s{J}_{\min}}$ where $\s{J}_{\min}$ is the minimum cluster size. Since $\tau=\frac{\s{J}_{\max}}{\s{J}_{\min}}$ and $\s{C}\s{J}_{\max} \ge \s{N}'$, we must have $\s{J}_{\min} \ge \frac{\s{N}'}{\s{C}\tau}$. Hence $\max_{i}\|(\fl{U}^{(1)}_{\s{sub}})^{\s{T}}\fl{e}_i\|_2 \le \sqrt{\frac{\s{C}\tau}{\s{N}'}}$.
Similarly, we have
\begin{align*}
    &\max_i \|\widehat{\fl{V}}(\fl{V}_{\s{sub}}^{\s{T}}\fl{V}_{\s{sub}})^{-1/2}\fl{V}_{\s{sub}}^{\s{T}}\fl{e}_i\| \le \max_i \|(\fl{V}_{\s{sub}}^{\s{T}}\fl{V}_{\s{sub}})^{-1/2}\fl{V}_{\s{sub}}^{\s{T}}\fl{e}_i\| \\
    &\le \frac{\|\fl{V}_{\s{sub}}\|_{2,\infty}}{\sqrt{\lambda_{\min}(\fl{V}_{\s{sub}}^{\s{T}}\fl{V}_{\s{sub}})}} \le \frac{\lr{\fl{V}}_{2,\infty}}{\sqrt{\lambda_{\min}(\fl{V}_{\s{sub}}^{\s{T}}\fl{V}_{\s{sub}})}} \le \sqrt{\frac{\mu \s{C}}{\alpha \s{M}'}}
\end{align*}
where the last line follows from the fact that $\min_{\fl{x}\in \bb{R}^{\s{C}}} \fl{x}^{\s{T}}\fl{V}_{\s{sub}}^{\s{T}}\fl{V}_{\s{sub}}\fl{x} = \min_{\fl{x}\in \bb{R}^{\s{C}}} \sum_{i \in \ca{S}} \fl{x}^{\s{T}}\fl{V}_{i}\fl{V}_{i}^{\s{T}}\fl{x}$ where $\ca{S},|\ca{S}|=\s{M}'$ is the set of rows in $\fl{V}$ represented in $\fl{V}_{\s{sub}}$. Here, we use Assumption \ref{assum:matrix} to conclude that $\lambda_{\min}(\fl{V}_{\s{sub}}^{\s{T}}\fl{V}_{\s{sub}}) \ge \alpha \s{M}'/\s{M}$.


\end{proof}

\subsection{Proof of Lemma~\ref{lem:condition_num}}
\begin{proof}[Proof of Lemma \ref{lem:condition_num}]
Suppose $\fl{P}_{\s{sub}}$ has dimensions $m' \times n'$ such that its rows correspond to the set of indices $\ca{S}$ and columns correspond to the set of indices $\ca{S}'$.
We have that $\lambda_1^2 = \sup_{\fl{x} \in \bb{R}^n\mid \lr{\fl{x}}_2=1} \fl{x}^{\s{T}}\fl{P}^{\s{T}}\fl{P}\fl{x}$ and $(\lambda'_1)^2 = \sup_{\fl{x} \in \bb{R}^{n'}\mid \lr{\fl{x}}_2=1} \fl{x}^{\s{T}}\fl{P}_{\s{sub}}^{\s{T}}\fl{P}_{\s{sub}}\fl{x}$. Similarly, we can write $\lambda'_1 = \sup_{\fl{x} \in \bb{R}^{n'}\mid \lr{\fl{x}}_2=1} \fl{x}^{\s{T}}\sum_{i\in \ca{S}}(\fl{P}_{i\mid \ca{S}'}^{\s{T}}\fl{P}_{i\mid \ca{S}'})\fl{x}$ where $\fl{P}_{i\mid \ca{S}'}$ is the $i^{\s{th}}$ row of the matrix $\fl{P}$ restricted to the column indices in $\ca{S}'$. Since $\fl{x}^{\s{T}}(\fl{b}\fl{b}^{\s{T}})\fl{x}>0$ for any vector $\fl{x}$, we have 
\begin{align*}
    &\s{J}_{\max} \lambda^2_{\max} \ge \s{J}_{\max}\sup_{\fl{x} \in \bb{R}^n\mid \lr{\fl{x}}_2=1} \fl{x}^{\s{T}}\sum_{i\in [\s{C}]}(\fl{X}_{i}^{\s{T}}\fl{X}_{i})\fl{x} \overset{(a)}{\ge} \sup_{\fl{x} \in \bb{R}^n\mid \lr{\fl{x}}_2=1} \fl{x}^{\s{T}}\sum_{i\in [n]}(\fl{P}_{i} ^{\s{T}}\fl{P}_{i})\fl{x} \\
    & \ge \sup_{\fl{x} \in \bb{R}^{n}\mid \lr{\fl{x}}_2=1} \fl{x}^{\s{T}}\sum_{i\in \ca{S}}(\fl{P}_i^{\s{T}} \fl{P}_i)\fl{x} 
    \ge \sup_{\fl{x} \in \bb{R}^{n}\mid \lr{\fl{x}}_2=1,\fl{x}_{\mid [n]\setminus \ca{S}'}=\fl{0}} \fl{x}^{\s{T}}\sum_{i\in \ca{S}}(\fl{P}_i^{\s{T}}\fl{P}_i)\fl{x} \ge \sup_{\fl{x} \in \bb{R}^{n'}\mid \lr{\fl{x}}_2=1} \fl{x}^{\s{T}}\sum_{i\in \ca{S}}(\fl{P}_{i\mid \ca{S}'}^{\s{T}} \fl{P}_{i\mid \ca{S}'})\fl{x}
\end{align*}
implying that $\sqrt{\s{J}_{\max}} \lambda_{\max} \ge \lambda_1'$. The inequality (a) follows from the fact that $\s{J}_{\max} \fl{X}_i^{\s{T}}\fl{X}_i \ge \sum_{u \in [n]\mid u \in \ca{C}^{(i)}}\fl{P}_i^{\s{T}}\fl{P}_i$. In order to prove the inequality on $\lambda_r'$, we need to do some more work. 
Let us denote the SVD of $\fl{X}=\fl{U}\f{\Sigma}\fl{V}^{\s{T}}$ where $\fl{U}\in \bb{R}^{\s{C}\times \s{C}}, \fl{V}^{\s{T}}\in \bb{R}^{\s{C}\times \s{M}}$. Consider the matrix $\fl{P}$ restricted to the rows in the set $\ca{S}$ denoted by $\fl{P}_{\ca{S}}$. Notice that $\fl{P}_{\ca{S}}^{\s{T}}\fl{P}_{\ca{S}}=\fl{V}\f{\Sigma}\fl{U}^{\s{T}}\fl{D}\fl{U}\f{\Sigma}\fl{V}^{\s{T}}$ where $\fl{D}$ is a diagonal matrix whose $i^{\s{th}}$ entry corresponds to the number of users in the set $\ca{S}$ belonging to the $i^{\s{th}}$ cluster. If the set $\ca{S}$ is a union of clusters, then the minimum diagonal entry in $\fl{D}$ must be $\s{J}_{\min}$. 
Let us denote the nullspace of the matrix $\fl{P}_{\ca{S}}$ by $\ca{K}$.
We can write the minimum non-zero eigenvalue of $\fl{P}_{\ca{S}}^{\s{T}}\fl{P}_{\ca{S}}$ as $\min_{\fl{x}\in \ca{K}^c \mid \lr{\fl{x}}_2=1} \fl{x}^{\s{T}}\fl{V}\f{\Sigma}\fl{U}^{\s{T}}\fl{D}\fl{U}\f{\Sigma}\fl{V}^{\s{T}}\fl{x}$. 
Let us write the vector $\fl{z}=\f{\Sigma}\fl{V}^{\s{T}}\fl{x}$; $\fl{z}$ must belong to the sub-space $\ca{T}$ spanned by the rows of $\fl{U}$ corresponding to the non-zero diagonal indices of $\fl{D}$ as otherwise $\fl{x}$ will have non-zero projection on the null-space of $\fl{P}_{\ca{S}}$. Note that the rows of $\fl{U}$ are orthonormal as well (since $\fl{U}$ is a square matrix i.e. $\fl{U}^{\s{T}}=\fl{U}^{-1}$). Hence, $\min_{\fl{z}\in \ca{T}}\fl{z}^{\s{T}}\fl{U}^{\s{T}}\fl{D}\fl{U}\fl{z}=\s{J}_{\min}\lr{\fl{z}}_2^2 \ge \s{J}_{\min} \s{\lambda}^2_{\min}$.  
Next, we have that 
\begin{align*}
    & \inf_{\fl{x} \in  \ca{K}^c \mid \lr{\fl{x}}_2=1} \fl{x}^{\s{T}}\sum_{i\in \ca{S}}(\fl{P}_{i}^{\s{T}}\fl{P}_{i})\fl{x} \\
    & \overset{(b)}{\le} \inf_{\fl{x} \in  \ca{K}^c \mid \lr{\fl{x}}_2=1,\fl{x}_{\mid [n]\setminus \ca{S}'}=\fl{0}} \fl{x}^{\s{T}}\sum_{i\in \ca{S}}(\fl{P}_i^{\s{T}} \fl{P}_i)\fl{x}
     \le  \inf_{\fl{x} \in  (\ca{K}')^c \cap \ca{K}^c \mid \lr{\fl{x}}_2=1} \fl{x}^{\s{T}}\sum_{i\in \ca{S}}(\fl{P}_{i\mid \ca{S}'}^{\s{T}}\fl{P}_{i\mid \ca{S}'})\fl{x}.              
\end{align*}
where $\ca{K}'$ is the null-space of the matrix $\fl{P}$ restricted to the columns in $\ca{S}'$. Note that the step (b) follows as long as the there exists a vector with non-zero entries only on ${\cal S}'$ in the row space which implies that the sub-matrix is non-zero.

Therefore $\lambda'_{|\s{C}|'} \ge \sqrt{\s{J}_{\min}}\lambda_{\min}$. Using the fact that $\tau=\s{J}_{\max}/\s{J}_{\min}$, the lemma is proved.

\end{proof}

\section{Proofs of Theorems \ref{thm:main_LBM} and \ref{thm:main_LBM2}}\label{app:detailed_lbm}

\subsection{Proof Overview}\label{subsec:proof_overview}

For any phase indexed by $\ell$, we are going to prove that conditioned on the events $\ca{E}^{(\ell)}$, the event $\ca{E}^{(\ell+1)}$ is also going to be true with high probability with proper choice of $\Delta_{\ell+1},\epsilon_{\ell+1}$. First, inspired by low rank matrix completion techniques, conditioned on $\ca{E}^{(\ell)}$ and by using Lemmas \ref{lem:condition_num}, \ref{lem:incoherence} along with the fact that each set of users in $\ca{M}^{(\ell)}$ is \textit{nice}, we can show with Lemma \ref{lem:min_acc} that in phase $\ell$, by using $m_{\ell}=O(\s{V}/\Delta^2_{\ell+1})$ (where $\s{V}=\widetilde{O}(\sigma^2 \mu^3 (\s{N}+\s{M}))$) rounds, $\ca{E}_2^{(\ell)}$ is true with high probability (see Alg. \ref{algo:estimate} in Appendix \ref{app:mc}). Next, we can show the following series of lemmas (let $\ca{B}^{(\ell)} =\bigcup_{i \in [a_{\ell}] \mid \left| \ca{N}^{(\ell,i)}\right| \ge \gamma\s{C}} \ca{M}^{(\ell,i)}$ denote the set of users having more than $\gamma\s{C}$ active arms) regarding the sets $\ca{T}_u^{(\ell)}$ of good arms obtained from the estimates of the relevant reward sub-matrices (Step 6 in Alg. \ref{algo:estimate}):

\begin{lemma}\label{lem:interesting2}
Conditioned on  $\ca{E}^{(\ell)},\ca{E}_2^{(\ell)}$, for every $u\in \ca{B}^{(\ell)}$,  $\pi_u(1)\in \ca{T}_u^{(\ell)}$ and $\max_{s,s'\in \ca{T}_u^{(\ell)}}\left|\fl{P}_{us}-\fl{P}_{us'}\right|\le 4\Delta_{\ell+1}$.
\end{lemma}

Lemma \ref{lem:interesting2} states that for every relevant user, the best arm always belongs to the set of good arms and characterizes how good the remaining arms are. 

\begin{lemma}\label{lem:interesting_clique}
Fix any $i\in [a_{\ell}]$ such that $\left|\ca{N}^{(\ell,i)}\right|\ge \gamma\s{C}$. Conditioned on the events $\ca{E}^{(\ell)},\ca{E}_2^{(\ell)}$, nodes in $\ca{M}^{(\ell,i)}$ corresponding to the same cluster form a clique.
Also, users in each connected component of the graph $\ca{G}^{(\ell,i)}$ form a nice subset.
\end{lemma}

Recall that we draw a graph $\ca{G}^{(\ell,i)}$ with nodes corresponding to users in $\ca{M}^{(\ell,i)}$ and edges drawn according to eq. (\ref{eq:edge}). Lemma \ref{lem:interesting_clique} says that users in same cluster always form a clique; however, this does not rule out inter-cluster edges. Nevertheless, if two users have an edge, then the next lemma shows that good arms for one are good for the other:

\begin{lemma}\label{lem:interesting4}
Fix any $i\in [a_{\ell}]$ such that $\left|\ca{N}^{(\ell,i)}\right|\ge \gamma\s{C}$. Consider two users $u,v\in \ca{M}^{(\ell,i)}$ having an edge in the graph $\ca{G}^{(\ell,i)}$. Conditioned on the events $\ca{E}^{(\ell)},\ca{E}_2^{(\ell)}$, we must have 
\begin{align*}
&\max_{x\in \ca{T}^{(\ell)}_u,y\in \ca{T}^{(\ell)}_v}\left|\fl{P}_{ux}-\fl{P}_{uy}\right|\le 16\Delta_{\ell+1} \\
&\text{ and } \max_{x\in \ca{T}^{(\ell)}_u,y\in \ca{T}^{(\ell)}_v}\left|\fl{P}_{vx}-\fl{P}_{vy}\right|\le 16\Delta_{\ell+1}   
\end{align*}
\end{lemma}

We can extend Lemma \ref{lem:interesting4} to the case when two users have a path joining them in the constructed graph:

\begin{lemma}\label{lem:interesting5}
Fix any $i\in [a_{\ell}]$ such that $\left|\ca{N}^{(\ell,i)}\right|\ge \gamma\s{C}$. Consider two users $u,v\in \ca{M}^{(\ell,i)}$ having a path in the graph $\ca{G}^{(\ell,i)}$. Conditioned on the events $\ca{E}^{(\ell)},\ca{E}_2^{(\ell)}$, we must have 
\begin{align*}
&\max_{x\in \ca{T}^{(\ell)}_u,y\in \ca{T}^{(\ell)}_v}\left|\fl{P}_{ux}-\fl{P}_{uy}\right|\le 32\s{C}\Delta_{\ell+1} \\
&\text{ and } \max_{x\in \ca{T}^{(\ell)}_u,y\in \ca{T}^{(\ell)}_v}\left|\fl{P}_{vx}-\fl{P}_{vy}\right|\le 32\s{C}\Delta_{\ell+1}.   
\end{align*}
\end{lemma}

Thus we can create the new sets of nice users as the connected components of the graph and the corresponding set of arms that are good for all is constructed by the union of the set of good arms (eq. \ref{eq:common_set}). If the number of active arms for a set of users become less than $\gamma\s{C}$, then we start a UCB algorithm for each of them until the end of the number of rounds. Therefore, conditioned on the events $\ca{E}^{(\ell)},\ca{E}^{(\ell)}_2$, the event $\ca{E}^{(\ell)}$ is also going to be true. Hence, conditioned on the event $\ca{E}^{(\ell)}$, we can bound the regret in each round of the $\ell^{\s{th}}$ phase by $\epsilon_{\ell}$; roughly speaking, the number of rounds in the $\ell^{\s{th}}$ phase is  $1/\epsilon^2_{\ell}$ and therefore the regret is $1/\epsilon_{\ell}$. By setting $\Delta_{\ell}$ as in Step 3 of Alg. \ref{algo:phased_elim} (and $\epsilon_{\ell}=\Delta_{\ell}/64\s{C}$), we can bound the regret of LATTICE and achieve the guarantee in Theorems \ref{thm:main_LBM} and \ref{thm:main_LBM2}.

\subsection{Detailed Proof}\label{sec:detailed_proof}

LATTICE is run in phases indexed by $\ell=1,2,\dots$. In the beginning of each phase $\ell$, we have the following set of desirable properties:
\begin{enumerate}
    \item Maintain a list of groups of users $\ca{M}^{(\ell)}\equiv\{\ca{M}^{(\ell,1)},\ca{M}^{(\ell,2)},\dots,\ca{M}^{(\ell,a_{\ell})}\}$ and arms $\ca{N}^{(\ell)}\equiv\{\ca{N}^{(\ell,1)},\ca{N}^{(\ell,2)},\dots,\ca{N}^{(\ell,a_{\ell})}\}$  where $a_{\ell}\le \s{C}$ such that  $\cup_{i \in [a_{\ell}]} \ca{M}^{(\ell,i)}= [\s{N}]$ and $\cup_{i \in [a_{\ell}]} \ca{N}^{(\ell,i)} \subseteq  [\s{M}]$.
    \item Moreover, for all $i\in [a_{\ell}]$, we will have $\ca{M}^{(\ell,i)}= \bigcup_{j \in \ca{G}^{(\ell,i)}} \ca{C}^{(j)}$ where the sets $\{\ca{G}^{(\ell,1)},\ca{G}^{(\ell,2)},\dots,\ca{G}^{(\ell,a_{\ell})}\}$ form a partition of the set $[\s{C}]$. This implies that every set of users in the family $\ca{M}^{(\ell)}$ is \textit{nice} and the sets of users in $\ca{M}^{(\ell)}$ form a partition of $[\s{N}]$.
    \item For each group $\ca{M}^{(\ell,i)}$ in the list $\ca{M}^{(\ell)}$, we will have an active set of arms denoted by $\ca{N}^{(\ell,i)}$ such that $\ca{N}^{(\ell,i)} \supseteq \{\s{argmax}_{j} \fl{P}_{uj} \mid u \in \ca{M}^{(\ell,i)}\}$ i.e. for each user $u$ in the set $\ca{M}^{(\ell,i)}$, their best arm must belong to the set $\ca{N}^{(\ell,i)}$. Furthermore, for all $i\in [a_{\ell}]$ such that $\left|\ca{N}^{(\ell,i)}\right| \ge \gamma\s{C}$,
the set $\ca{N}^{(\ell,i)}$ must also satisfy the following:
\begin{align}\label{eq:gap}
    \left|\max_{j\in \ca{N}^{(\ell,i)}}\fl{P}_{uj}-\min_{j\in \ca{N}^{(\ell,i)}}\fl{P}_{uj}\right| \le \epsilon_{\ell} \text{ for all }u\in \ca{M}^{(\ell,i)}
\end{align}
where $\epsilon_{\ell}$ is a fixed exponentially decreasing sequence in $\ell$ (in particular, we choose $\epsilon_1=\lr{\fl{P}}_{\infty}$ and $\epsilon_{\ell}=C'2^{-\ell}\min\Big(\|\fl{P}\|_{\infty},\frac{\sigma\sqrt{\mu}}{\log \s{N}}\Big)$ for $\ell>1$ for some constant $C'>0$).
\item Let $\ca{B}^{(\ell)}\subseteq [\s{N}]$ be a subset of users satisfying $\ca{B}^{(\ell)} =\bigcup_{i \in [a_{\ell}] \mid \left| \ca{N}^{(\ell,i)}\right| \ge \gamma\s{C}} \ca{M}^{(\ell,i)}$ i.e. $\ca{B}^{(\ell)}$ corresponds to the set of users which belong to a group $\ca{M}^{(\ell,i)}$ at the beginning of the $\ell^{\s{th}}$ phase having more than $\s{C}$ active arms. We will also maintain that $\ca{B}^{(i)} \supseteq  \ca{B}^{(j)}$ for any phase $i \le j$ i.e. the set of users with more than $\s{C}$ active arms goes on shrinking.
\end{enumerate}

Since LATTICE is random, we will say that our algorithm is $\epsilon_{\ell}-$\textit{good} at the beginning of the $\ell^{\s{th}}$ phase if the algorithm can maintain a list of users and arms satisfying the above properties at the start of phase $\ell$. Let us also define the event $\ca{E}^{(\ell)}$ to be true if properties $(1-4)$ are satisfied at the beginning of phase $\ell$ by the phased elimination algorithm. We are going to prove inductively that the phased elimination algorithm is $\epsilon_{\ell}$-good for all phases indexed by $\ell$ for our choice of $\{\epsilon_{\ell}\}$ with high probability as long as the number of phases are small.

\noindent \textbf{Base Case:} For $\ell=1$ (the first phase), we initialize  $\ca{M}^{(1,1)} = [\s{N}]$, $\ca{N}^{(1,1)}=[\s{M}]$ and therefore, we have $$\left|\max_{j\in \ca{N}^{(\ell,1)}}\fl{P}_{uj}-\min_{j\in \ca{N}^{(\ell,1)}}\fl{P}_{uj}\right| \le \lr{\fl{P}}_{\infty} \text{ for all }u\in [\s{M}].$$ 
Hence, we also have  $\ca{B}^{(1)}=[\s{N}]$.
Moreover, the set of users $[\s{N}]$ satisfies $[\s{N}]=\cup_{j \in \ca{G}^{(1,1)}}\ca{C}^{(j)}$ where $\ca{G}^{(1,1)}=[\s{C}]$ and finally for every user $u\in [\s{N}]$, the best arm $\s{argmax}_{j} \fl{P}_{uj}$ belongs to the entire set of arms. 
Thus for $\ell=1$, our initialization makes the algorithm $\lr{\fl{P}}_{\infty}$-good. 

\noindent \textbf{Inductive Argument:} Suppose, at the beginning of the phase $\ell$, we condition on the event $\ca{E}^{(\ell)}$ that Algorithm is $\epsilon_{\ell}-$\textit{good}. 
Next, our goal is to run Matrix completion in order to estimate each of the sub-matrices corresponding to $\{(\ca{M}^{(\ell,i)},\ca{N}^{(\ell,i)})\}_{i \in [\s{C}]}$. Fix the quantity $\Delta_{\ell+1}>0$.
We can show the following lemma:

\begin{lemma}\label{lem:interesting1}
Let us fix $\Delta_{\ell+1}>0$ and condition on the event $\ca{E}^{(\ell)}$. Suppose Assumptions \ref{assum:matrix} and \ref{assum:cluster_ratio} are satisfied. In that case, in phase $\ell$, by using
 $$m_{\ell}=O\Big(\frac{\sigma^2 \s{C}^2 (\s{C} \bigvee \mu\alpha^{-1})^3 \log \s{M}}{\Delta_{\ell+1}^2}\Big(\s{N}\bigvee\s{MC}\Big)\log^2 (\s{MNC}\delta^{-1})\Big)\Big)$$ rounds, we can compute an estimate $\widetilde{\fl{P}}^{(\ell)}\in \bb{R}^{\s{N}\times \s{M}}$ such that with probability $1-\delta$, we have for a nice subset of users $\ca{M}^{(\ell,i)}$
 \begin{align}
 \lr{\widetilde{\fl{P}}^{(\ell)}_{\ca{M}^{(\ell,i)},\ca{N}^{(\ell,i)}}-\fl{P}_{\ca{M}^{(\ell,i)},\ca{N}^{(\ell,i)}}}_{\infty} \le \Delta_{\ell+1} \text{ for all }i\in [a_{\ell}] \text{ satisfying } \left|\ca{N}^{(i,\ell)}\right| \ge \gamma\s{C}.  
 \end{align}
 \end{lemma}
 
 \begin{proof}[Proof of Lemma \ref{lem:interesting1}]
We are going to use Lemma \ref{lem:min_acc} in order to compute an estimate $\widetilde{\fl{P}}^{(\ell)}_{\ca{M}^{(\ell,i)},\ca{N}^{(\ell,i)}}$ of the sub-matrix $\fl{P}_{\ca{M}^{(\ell,i)},\ca{N}^{(\ell,i)}}$ satisfying $\lr{\widetilde{\fl{P}}^{(\ell)}_{\ca{M}^{(\ell,i)},\ca{N}^{(\ell,i)}}-\fl{P}_{\ca{M}^{(\ell,i)},\ca{N}^{(\ell,i)}}}_{\infty} \le \Delta_{\ell+1}$. From Lemma \ref{lem:min_acc}, we know that by using $m_{\ell}=O\Big(s\s{N}\log^2 (\s{MN}\delta^{-1})(\left|\ca{N}^{(\ell,i)}\right|p+\sqrt{\left|\ca{N}^{(\ell,i)}\right|p\log \s{N}\delta^{-1}})\Big)$ rounds (see Lemma \ref{lem:min_acc}) restricted to users in $\ca{M}^{(\ell,i)}$ such that with probability at least $1-\delta$,
\begin{align*}
    \lr{\widetilde{\fl{P}}^{(\ell)}_{\ca{M}^{(\ell,i)},\ca{N}^{(\ell,i)}}-\fl{P}_{\ca{M}^{(\ell,i)},\ca{N}^{(\ell,i)}}}_{\infty} \le O\Big(\frac{\sigma r }{\sqrt{sd_2}}\sqrt{\frac{\widetilde{\mu}^3\log d_2}{p}}\Big).
\end{align*}
where $d_2=\min(|\ca{M}^{(\ell,i)}|,|\ca{N}^{(\ell,i)}|)$, $\widetilde{\mu}$ is the incoherence factor of the matrix $\fl{P}_{\ca{M}^{(\ell,i)},\ca{N}^{(\ell,i)}}$ and $r$ is the rank of the matrix bounded from above by the number of clusters. In order for the right hand side to be less than $\Delta_{\ell+1}$, we can set $sp=O\Big(\frac{\sigma^2 r^2 \widetilde{\mu}^3 \log d_2}{\Delta_{\ell+1}^2d_2}\Big)$. Since the event $\ca{E}^{(\ell)}$ is true, we must have that $\left|\ca{M}^{(\ell,i)}\right|\ge \s{N}/(\kappa\s{C})$ ($\kappa=O(1)$ is the ratio of the sizes of maximum cluster and minimum cluster); hence $d_2 \ge \min\Big(\frac{\s{N}}{\kappa\s{C}},\left|\ca{N}^{(\ell,i)}\right|\Big)$. Therefore, we must have that 
\begin{align*}
    m_{\ell}=O\Big(\frac{\sigma^2 \s{C}^2 \widetilde{\mu}^3 \log \s{M}}{\Delta_{\ell+1}^2}\max\Big(\s{N},\s{MC}\Big)\log^2 (\s{MNC}\delta^{-1})\Big)\Big)
\end{align*}
where we take a union bound over all sets comprising the partition of the users $[\s{N}]$ (at most $\s{C}$ of them). Finally, from Lemma \ref{lem:incoherence}, we know that $\widetilde{\mu}$ can be bounded from above by $\max(\s{C},2\mu /\alpha)$ which we can use to say that 
\begin{align*}
    m_{\ell}=O\Big(\frac{\sigma^2 \s{C}^2 (\s{C} \bigvee \mu\alpha^{-1})^3 \log \s{M}}{\Delta_{\ell+1}^2}\Big(\s{N}\bigvee\s{MC}\Big)\log^2 (\s{MNC}\delta^{-1})\Big)\Big)
\end{align*}
to complete the proof of the lemma.
\end{proof}

In the following part of the analysis, we will repeatedly condition on the events $\ca{E}^{(\ell)}, \ca{E}_{2}^{(\ell)}$ which are described/reiterated below:

\begin{enumerate}
    \item The event $\ca{E}^{(\ell)}$ is true when the properties $(1-4)$ described at the beginning of Section \ref{app:detailed_lbm} are satisfied by the algorithm.
    \item We will denote the event described in Lemma  \ref{lem:interesting1} equation \ref{eq:infty} to be $\ca{E}^{(\ell)}_2$. In other words, if $\ca{E}^{(\ell)}_2$ is true, then the algorithm has successfully computed an estimate $\widetilde{\fl{P}}^{(\ell)}\in \bb{R}^{\s{N}\times \s{M}}$ such that 
 \begin{align*}
 \lr{\widetilde{\fl{P}}^{(\ell)}_{\ca{M}^{(\ell,i)},\ca{N}^{(\ell,i)}}-\fl{P}_{\ca{M}^{(\ell,i)},\ca{N}^{(\ell,i)}}}_{\infty} \le \Delta_{\ell+1} \text{ for all }i\in [a_{\ell}] \text{ satisfying } \left|\ca{N}^{(i,\ell)}\right| \ge \gamma\s{C}.  
 \end{align*} 
\end{enumerate}

Fix any $i \in [a_{\ell}]$. 
For each user $u\in \ca{M}^{(\ell,i)}$, let us denote a set of good arms for the user $u$ by $\ca{T}^{(\ell)}_u \equiv \{j \in \ca{N}^{(\ell,i)} \mid \max_{j'\in \ca{N}^{(\ell,i)}}\widetilde{\fl{P}}^{(\ell)}_{uj'}-\widetilde{\fl{P}}^{(\ell)}_{uj} \le 2\Delta_{\ell+1}\}$. If we condition on the event $\ca{E}^{(\ell)}_2$, then we can show the following statements to be true:

\begin{lemmau}[Restatement of Lemma \ref{lem:interesting2}]
Condition on the events $\ca{E}^{(\ell)},\ca{E}_2^{(\ell)}$ being true. In that case, for every user $u\in \ca{B}^{(\ell)}$, the arm with the highest reward $\fl{P}_{u\pi_u(1)}$ must belong to the set $\ca{T}_u^{(\ell)}$. Moreover, $\max_{s,s'\in \ca{T}_u^{(\ell)}}\left|\fl{P}_{us}-\fl{P}_{us'}\right|\le 4\Delta_{\ell+1}$.
\end{lemmau}

\begin{proof}
Let us fix a user $u\in \ca{M}^{(\ell,i)}$ with active set of arms $\ca{N}^{(\ell,i)}$.
Recall that $\pi_u(1)=\s{argmax}_{j}\fl{P}_{uj}$ and let us denote $t_1=\s{argmax}_{j\in \ca{N}^{(\ell,i)}}\widetilde{\fl{P}}^{(\ell)}_{uj}$for brevity of notation.
Now, we will have
\begin{align*}
    \widetilde{\fl{P}}_{ut_1}-\widetilde{\fl{P}}^{(\ell)}_{u\pi_u(1)} = \fl{\widetilde{P}}^{(\ell)}_{ut_1}-\fl{P}_{ut_1}+\fl{P}_{ut_1}-\fl{P}_{u\pi_u(1)}+\fl{P}_{u\pi_u(1)}-\widetilde{\fl{P}}^{(\ell)}_{u\pi_u(1)} \le 2\Delta_{\ell+1}
\end{align*}
which implies that $\pi_u(1)\in \ca{T}_u^{(\ell)}$.
Here we used the fact that $\fl{\widetilde{P}}^{(\ell)}_{ut_1}-\fl{P}_{ut_1} \le \Delta_{\ell+1}$, $\fl{P}_{u\pi_u(1)}-\widetilde{\fl{P}}^{(\ell)}_{u\pi_u(1)} \le \Delta_{\ell+1}$ and $\fl{P}_{ut_1}-\fl{P}_{u\pi_u(1)}\le 0$.


Next, notice that for any $s,s'\in \ca{T}_u^{(\ell)}$
\begin{align*}
    \fl{P}_{us}-\fl{P}_{us'} = \fl{P}_{us}-\fl{\widetilde{P}}^{(\ell)}_{us}+\fl{\widetilde{P}}^{(\ell)}_{us}-\fl{\widetilde{P}}^{(\ell)}_{ut_1} 
    +\fl{\widetilde{P}}^{(\ell)}_{ut_1}-\fl{\widetilde{P}}^{(\ell)}_{us'}+\fl{\widetilde{P}}^{(\ell)}_{us'}-\fl{P}_{us'} \le 4\Delta_{\ell+1}.
\end{align*}

\end{proof}

Again, fix any $i \in [a_{\ell}]$ such that $\left|\ca{N}^{(\ell,i)}\right|\ge \gamma\s{C}$. Consider a graph $\ca{G}^{(\ell,i)}$ whose nodes are given by the users in $\ca{M}^{(\ell,i)}$. Now, we draw an edge between two users $u,v \in \ca{M}^{(\ell,i)}$ if $\ca{T}^{(\ell)}_u \cap \ca{T}^{(\ell)}_v \neq \Phi$ and $\left|\fl{\widetilde{P}}^{(\ell)}_{ux}-\fl{\widetilde{P}}^{(\ell)}_{vx}\right|\le 2\Delta_{\ell+1}$ for all arms $x\in \ca{N}^{(\ell,i)}$.


\begin{lemmau}[Restatement of Lemma \ref{lem:interesting_clique}]
Fix any $i\in [a_{\ell}]$ such that $\left|\ca{N}^{(\ell,i)}\right|\ge \gamma\s{C}$. Conditioned on the events $\ca{E}^{(\ell)},\ca{E}_2^{(\ell)}$, nodes in $\ca{M}^{(\ell,i)}$ corresponding to the same cluster form a clique.
Also, users in each connected component of the graph $\ca{G}^{(\ell,i)}$ form a nice subset.
\end{lemmau}

\begin{proof}
For any two users $u,v\in \ca{M}^{(\ell,i)}$ belonging to the same cluster, consider an arm $x\in \ca{N}^{(\ell,i)}$. We must have
\begin{align*}
    \fl{\widetilde{P}}^{(\ell)}_{ux}-\fl{\widetilde{P}}^{(\ell)}_{vx} = \fl{\widetilde{P}}^{(\ell)}_{ux}-\fl{P}_{ux}+\fl{P}_{ux}-\fl{P}_{vx}+\fl{P}_{vx}-\fl{\widetilde{P}}^{(\ell)}_{vx}
    \le 2\Delta_{\ell+1}.
\end{align*}

In order to prove the next statement, 
consider two users $u,v\in \ca{M}^{(\ell,i)}$ that belongs to different clusters $\ca{P},\ca{Q}$
respectively. Note that since the event $\ca{E}^{(\ell)}$ is true, $\ca{M}^{(\ell,i)}$ is a union of clusters comprising $\ca{P},\ca{Q}$. Furthermore, we have already established that nodes in $\ca{G}^{(\ell,i)}$ (users in $\ca{M}^{(\ell,i)}$) restricted to the same cluster form a clique. There every connected component of the graph $\ca{G}^{(\ell,i)}$ can be represented as a union of a subset of clusters.

\end{proof}

\begin{lemmau}[Restatement of Lemma \ref{lem:interesting4}]
Fix any $i\in [a_{\ell}]$ such that $\left|\ca{N}^{(\ell,i)}\right|\ge \gamma\s{C}$. Consider two users $u,v\in \ca{M}^{(\ell,i)}$ having an edge in the graph $\ca{G}^{(\ell,i)}$. Conditioned on the events $\ca{E}^{(\ell)},\ca{E}_2^{(\ell)}$, we must have 
\begin{align*}
&\max_{x\in \ca{T}^{(\ell)}_u,y\in \ca{T}^{(\ell)}_v}\left|\fl{P}_{ux}-\fl{P}_{uy}\right|\le 16\Delta_{\ell+1} 
\text{ and }
\max_{x\in \ca{T}^{(\ell)}_u,y\in \ca{T}^{(\ell)}_v}\left|\fl{P}_{vx}-\fl{P}_{vy}\right|\le 16\Delta_{\ell+1}
\end{align*}
\end{lemmau}

\begin{proof}
From the construction of $\ca{G}^{(\ell,i)}$, we know that
users $u,v\in \ca{M}^{(\ell,i)}$ have an edge if $\left|\widetilde{\fl{P}}^{(\ell)}_{ux}-\widetilde{\fl{P}}^{(\ell)}_{vx}\right|\le 2\Delta_{\ell+1}$ (implying that $\left|\fl{P}_{ux}-\fl{P}_{vx}\right|\le 4\Delta_{\ell+1}$) for all $x \in \ca{N}^{(\ell,i)}$ and $\ca{T}_u\cap \ca{T}_v \neq \Phi$. Suppose $z\in \ca{T}_u\cap \ca{T}_v$. Now, for any pair of arms 
$x\in \ca{T}^{(\ell)}_u,y\in \ca{T}^{(\ell)}_v$; in that case, we have
\begin{align*}
    \fl{P}_{ux}-\fl{P}_{uy} = \fl{P}_{ux} - \fl{P}_{uz}+\fl{P}_{uz}-\fl{P}_{vz}+\fl{P}_{vz}-\fl{P}_{vy}+\fl{P}_{vy}-\fl{P}_{uy} \le 16\Delta_{\ell+1}
\end{align*}
where we used Lemma \ref{lem:interesting2}.

\end{proof}

\begin{lemmau}[Restatement of Lemma \ref{lem:interesting5}]
Fix any $i\in [a_{\ell}]$ such that $\left|\ca{N}^{(\ell,i)}\right|\ge \gamma\s{C}$. Consider two users $u,v\in \ca{M}^{(\ell,i)}$ having a path in the graph $\ca{G}^{(\ell,i)}$. Conditioned on the events $\ca{E}^{(\ell)},\ca{E}_2^{(\ell)}$, we must have 
\begin{align*}
\max_{x\in \ca{T}^{(\ell)}_u,y\in \ca{T}^{(\ell)}_v}\left|\fl{P}_{ux}-\fl{P}_{uy}\right|\le 32\s{C}\Delta_{\ell+1} \text{ and } \max_{x\in \ca{T}^{(\ell)}_u,y\in \ca{T}^{(\ell)}_v}\left|\fl{P}_{vx}-\fl{P}_{vy}\right|\le 32\s{C}\Delta_{\ell+1}.   
\end{align*}
\end{lemmau}

\begin{proof}
Consider $3$ users $u,v,w\in \ca{M}^{(\ell,i)}$ such that $u,v$ have an edge and similarly, $v,w$ have an edge. From 
Lemma \ref{lem:interesting2}, we have that 
\begin{align*}
&\ca{T}_u^{(\ell)} \cap \ca{T}_v^{(\ell)} \neq \Phi \text{ and }\left|\fl{P}_{ux}-\fl{P}_{vx}\right|\le 4\Delta_{\ell+1} \text{ for all }x \in \ca{N}^{(\ell,i)} \\
&\ca{T}_v^{(\ell)} \cap \ca{T}_w^{(\ell)} \neq \Phi \text{ and }\left|\fl{P}_{vx}-\fl{P}_{wx}\right|\le 4\Delta_{\ell+1} \text{ for all }x \in \ca{N}^{(\ell,i)}
\end{align*}
Let $z\in \ca{T}_u^{(\ell)} \cap \ca{T}_v^{(\ell)}$ and 
$z'\in \ca{T}_v^{(\ell)} \cap \ca{T}_w^{(\ell)}$. Therefore, for any $x\in \ca{T}^{(\ell)}_u,y\in \ca{T}^{(\ell)}_z$, we must have 
\begin{align*}
    &\fl{P}_{ux}-\fl{P}_{uy}=\fl{P}_{ux} - \fl{P}_{uz}+\fl{P}_{uz}-\fl{P}_{vz}+\fl{P}_{vz}-\fl{P}_{vz'}+\fl{P}_{vz'}-\fl{P}_{wz'}+\fl{P}_{wz'}-\fl{P}_{wy}+\fl{P}_{wy}-\fl{P}_{vy}+\fl{P}_{vy} -\fl{P}_{uy} \\
    &\le 32\Delta_{\ell+1}
\end{align*}

Note that the shortest path between the two users $u,v\in \ca{M}^{(\ell,i)}$ must be a sequence of at most  $2\s{C}$ nodes. Now, applying the above analysis at most $2\s{C}$ times, we get statement of the Lemma. In other words, consider a path connecting two users $u,v$ as denoted by $u=a_1,a_2,\dots,v=a_{\s{L}}$. Let us denote $z_i=\ca{T}_{a_i}^{(\ell)}\cap \ca{T}_{a_{i+1}}^{(\ell)}$ (definition of edge). For  $x\in \ca{T}^{(\ell)}_u,y\in \ca{T}^{(\ell)}_v$, we will have 
\begin{align*}
    \fl{P}_{ux}-\fl{P}_{uy}&= \fl{P}_{ux}-\fl{P}_{uz_1}+\sum_{i=1}^{\s{L}-1}\Big(\fl{P}_{a_iz_i}-\fl{P}_{a_{i+1}z_i}+\fl{P}_{a_{i+1}z_i}-\fl{P}_{a_{i+1}z_{i+1}}\Big)+\fl{P}_{vz_{\s{L}-1}}-\fl{P}_{vz'}+\sum_{j=\s{L}}^{j=2}\Big(\fl{P}_{a_{j}z'}-\fl{P}_{a_{j-1}z'}\Big) \\
    &\le 16\s{L}\Delta_{\ell+1}.
\end{align*}
Since the path connecting the two users can be of length at most $2\s{C}-1$ (conditioned on the events $\ca{E}^{(\ell)},\ca{E}^{(\ell)}_2$), the proof of our lemma is complete.
\end{proof}

For the subsequent iteration indexed by $\ell+1$, we compute the updated groups of users $\ca{M}^{(\ell+1)}$ in the following way: each set corresponds to the connected components of the graphs $\{\ca{G}^{(\ell,i)}\}$ for those indices $i\in [a_{\ell}]$ where $\left|\ca{N}^{(\ell,i)}\right|\ge \gamma\s{C}$ plus the groups of users $\ca{M}^{(\ell,i)}$ where $\left|\ca{N}^{(\ell,i)}\right|\le \s{C}$. More precisely, let $\ca{T}\subseteq [a_{\ell}]$ be the subset of indices for which 
$\left|\ca{N}^{(\ell,i)}\right|\ge \gamma\s{C}$; $\{\ca{G}^{(\ell,i,j)}\}$ be the connected components of the graph $\ca{G}^{(\ell,i)}$ for $i\in \ca{T}$. In that case, $\ca{M}^{(\ell+1)}=\{\ca{G}^{(\ell,i,j)} \mid i \in \ca{T}, \ca{G}^{(\ell,i,j)} \text{ is a connected component of graph }\ca{G}^{(\ell,i)}\}+\{\ca{M}^{(\ell,i)}\mid i \in [a_{\ell}]\setminus\ca{T}\}$.
Similarly, we update the family of sets of active arms as follows: for users corresponding to each connected component $\ca{M}^{(\ell+1,s)}=\ca{G}^{(\ell,i,j)}$ of some graph, we define the active set of arms $\ca{N}^{(\ell+1,s)}$ to be $\cup_{u\in \ca{G}^{(\ell,i,j)}}\ca{T}_u^{(\ell)}$ and for each group $\{\ca{M}^{(\ell,i)}\}_{i \in [a_{\ell}]\setminus\ca{T}}$, we keep the corresponding set of active  arms $\{\ca{N}^{(\ell,i)}\}_{i \in [a_{\ell}]\setminus\ca{T}}$ same. With $a_{\ell+1}=\left|\ca{M}^{(\ell+1)}\right|$, we will also update $\ca{B}^{(\ell+1)}=\bigcup_{i \in [a_{\ell+1}] \mid \left| \ca{N}^{(\ell+1,i)}\right| \ge \gamma\s{C}} \ca{M}^{(\ell+1,i)}$ to be the set of users with more than $\s{C}$ active arms. 

\begin{lemma}\label{lem:interesting6}
Condition on the events $\ca{E}^{(\ell)}$ being true. In that case, with probability $1-\s{T}^{-4}$, with the groups of users $\ca{M}^{(\ell+1)}$ and their respective group of arms given by $\ca{N}^{(\ell+1)}$ being updated as described above and $\Delta_{\ell+1}=\epsilon_{\ell}/64\s{C}$, the event $\ca{E}^{(\ell+1)}$ is also going to be true with $\epsilon_{\ell+1} \le \epsilon_{\ell}/2$.
\end{lemma}

\begin{proof}
Conditioned on the event $\ca{E}^{(\ell)}$ being true, the event $\ca{E}^{(\ell)}_2$ holds true with probability with $1-\s{T}^{-4}$ (by substituting $\delta=\s{T}^{-4}$ in Lemma \ref{lem:interesting1}). Now, conditioned on the event $\ca{E}^{(\ell)},\ca{E}_2^{(\ell)}$ being true, the properties $(1-4)$ hold true at the beginning of the $(\ell+1)^{\s{th}}$ phase as well with our construction of $\ca{M}^{(\ell+1)},\ca{N}^{(\ell+1)}$. 
For the $(\ell+1)^{\s{th}}$ phase
from Lemma \ref{lem:interesting5}, we know that for any pair of users $u,v$ in the same cluster $\ca{M}^{(\ell+1,i)}$ in the updated set of clusters $\ca{M}^{(\ell+1)}$, we must have $\max_{x\in \ca{T}^{(\ell)}_u,y\in \ca{T}^{(\ell)}_v}\left|\fl{P}_{ux}-\fl{P}_{vy}\right|\le 32\s{C}\Delta_{\ell+1}$. From Lemma \ref{lem:interesting2} we know that $\s{argmax}_j \fl{P}_{uj}\in \ca{T}_u^{(\ell)} \subseteq \ca{N}^{(\ell,i)}$ where $\ca{N}^{(\ell,i)}$ is the active set of arms for users in $\ca{M}^{(\ell,i)}$ in the updated set $\ca{M}$. For $\ell > 1$, we will set $\Delta_{\ell+1}=\epsilon_{\ell}/64\s{C}$ which would give us that $\epsilon_{\ell+1}= \epsilon_{\ell}/2$. Finally, also note that we maintain the set of users $\ca{B}^{(\ell+1)}$ as stipulated in Property 4 for the beginning of the $(\ell+1)^{\s{th}}$ phase.
\end{proof}

\begin{proof}[Proof of Theorem \ref{thm:main_LBM}]
 We condition on the events $\ca{E}^{(\ell)},\ca{E}_2^{(\ell)}$ being true for all $\ell$. The probability that there exists any $\ell$ such that the events $\ca{E}^{(\ell)},\ca{E}_2^{(\ell)}$ is false is $O(\s{T}^{-4})$ (by setting $\delta=\s{T}^{-4}$ in the proof of Lemma \ref{lem:interesting1}); hence the probability that $\ca{E}^{(\ell)},\ca{E}_2^{(\ell)}$ is true for all $\ell$ is at least $1-O(\s{T}^{-3})$ (the total number of iterations can be at most $\s{T}$).  
Let us also denote the set of rounds in phase $\ell$ by $\ca{T}_{\ell}\subseteq [\s{T}]$ (therefore $\left|\ca{T}_{\ell}\right|=m_{\ell}$). Let us compute the  regret $\sum_{t\in \ca{T}^{(\ell)}}\fl{P}_{u(t)\pi_{u(t)}(1)}- \sum_{t\in \ca{T}^{(\ell)}}\fl{P}_{u(t),\rho(t)}$ restricted to the rounds in $\ca{T}^{(\ell)}$ conditioned on the events $\ca{E}^{(\ell)},\ca{E}_2^{(\ell)}$ being true for all $\ell$. We can bound the regret quantity in the $\ell^{\s{th}}$ phase from above by the sum of two quantities: 1) the first quantity is the regret incurred by users in $\ca{B}^{(\ell)}$ with active arms more than $\s{C}$ 2) the second quantity is the regret incurred by the UCB algorithm played separately for each user  $u\in\s{N}\setminus [\ca{B}^{(\ell)}]$ with active arms less than $\gamma\s{C}$. Let us denote the regret incurred by such an user $u$ in the $\ell^{\s{th}}$ phase by $\s{Reg}_{\s{UCB}}(u,\ca{T}_u^{(\ell)})$ where $\ca{T}_u^{(\ell)}$ is the number of rounds in the $\ell^{\s{th}}$ phase where user $u$ pulled an arm.

The first quantity can be bounded from above by $m_{\ell}\epsilon_{\ell}$ while the second quantity can be bounded by using standard results in the literature. Substituting from Lemma \ref{lem:interesting1}, we have that \begin{align*}
  &\bb{E}\Big(\sum_{t\in \ca{T}^{(\ell)}}\fl{P}_{u(t)\pi_{u(t)}(1)}- \sum_{t\in \ca{T}^{(\ell)}}\fl{P}_{u(t),\rho(t)}\mid \bigcap_{\ell}\ca{E}^{(\ell)}\bigcap_{\ell}\ca{E}_2^{(\ell)}\Big) \\
  &= O\Big(\frac{\epsilon_{\ell}\sigma^2 \s{C}^2 (\s{C} \bigvee \mu\alpha^{-1})^3 \log \s{M}}{\Delta_{\ell+1}^2}\Big(\s{N}\bigvee\s{MC}\Big)\log^2 (\s{MNC}\delta^{-1})\Big)\Big)+\sum_{u\in [\s{N}]\setminus [\ca{B}^{(\ell)}]} \s{Reg}_{\s{UCB}}(u,\ca{T}_u^{(\ell)}) 
\end{align*} 
For simplicity, let us denote $\s{V}=\sigma^2 \s{C}^2 (\s{C} \bigvee \mu\alpha^{-1})^3 \Big(\s{N}\bigvee\s{MC}\Big)\log^3 (\s{ABCT})$. 
We can now bound the regret as follows (after removing the conditioning on the events $\bigcap_{\ell}\ca{E}^{(\ell)}\bigcap_{\ell}\ca{E}_2^{(\ell)}$):
\begin{align*}
    &\bb{E}\Big(\sum_{t\in [\s{T}]}\fl{P}_{u(t)\pi_{u(t)}(1)}- \sum_{t\in [\s{T}]}\fl{P}_{u(t),\rho(t)}\Big)
    =\bb{E}\Big(\sum_{\ell}\Big(\sum_{t\in \ca{T}^{(\ell)}}\fl{P}_{u(t)\pi_{u(t)}(1)}- \sum_{t\in \ca{T}^{(\ell)}}\fl{P}_{u(t),\rho(t)}\Big)\Big) \\
    &= \sum_{\ell}  O\Big(\epsilon_{\ell}m_{\ell}\mid \bigcap_{\ell}\ca{E}^{(\ell)}\bigcap_{\ell}\ca{E}_2^{(\ell)}\Big)+O(\s{T}^{-3}\|\fl{P}\|_{\infty})+\sum_{\ell}\sum_{u\in\s{N}\setminus [\ca{B}^{(\ell)}]} \s{Reg}_{\s{UCB}}(u,\ca{T}_u^{(\ell)}) 
\end{align*}
The last term of the regret can be bounded from above $\sum_{u \in [\s{N}]}\s{Reg}_{\s{UCB}}(u,\s{T}_b)$ where $\ca{T}_b$ is the number of rounds user $u$ pulled an arm according to the $\s{UCB}$ algorithm; hence, $\sum_{u \in [\s{N}]}\s{Reg}_{\s{UCB}}(u,\s{T}_b) \le \sum_{u \in [\s{N}]} \sqrt{\s{T}_b \gamma\s{C}\log \s{T}}\cdot\sigma \le \sqrt{\gamma\s{NCT}\log \s{T}}\cdot\sigma$ by using the Cauchy Schwartz inequality.
 Moving on, we can decompose the first term regret as follows (we use $\Delta_{\ell+1}=\epsilon_{\ell}/64\s{C}$): 
\begin{align*}
    O\Big(\sum_{\ell: \epsilon_{\ell}\le\Phi} \epsilon_{\ell}m_{\ell}\mid \ca{E}^{(\ell)},\ca{E}^{(\ell)}_2 \text{ is true for all } \ell\Big)
     +O\Big(\sum_{\ell: \epsilon_{\ell}>\Phi} \epsilon_{\ell}\s{V}\Delta_{\ell+1}^{-2}\mid \ca{E}^{(\ell)},\ca{E}^{(\ell)}_2 \text{ is true for all } \ell\Big) =
     \s{T}\Phi+O\Big(\sum_{\ell: \epsilon_{\ell}>\Phi} \s{C}^2\s{V}\epsilon_{\ell}^{-1}\Big) \\
\end{align*}

We choose $\epsilon_{\ell}=C'2^{-\ell}\min\Big(\|\fl{P}\|_{\infty},\frac{\sigma\sqrt{\mu}}{\log \s{N}}\Big)$ (so that the condition on $\sigma>0$ in Lemma \ref{lem:min_acc} is automatically satisfied for all $\ell$) for some constant $C'>0$, the maximum number of phases $\ell$ for which $\epsilon_{\ell}>\Phi$ can be bounded from above by $\s{J}=O\Big(\log \Big(\frac{1}{\Phi}\min\Big(\|\fl{P}\|_{\infty},\frac{\sigma\sqrt{\mu}}{\log \s{N}}\Big)\Big)\Big)$.
Moreover, the constraints on $p,\sigma$ present in Lemma \ref{lem:min_acc} can be satisfied with such a choice of $\epsilon_{\ell}$ for all $\ell$.
Hence, with $\s{V}=\sigma^2 \s{C}^2 (\s{C} \bigvee \mu\alpha^{-1})^3 \Big(\s{N}\bigvee\s{MC}\Big)\log^3 (\s{ABCT})$, we have
\begin{align*}
    \s{Reg}(\s{T}) &\le O(\s{T}^{-3}\|\fl{P}\|_{\infty})+O(\s{T}\Phi)+O\Big(\s{JV}\s{C}^2\Phi^{-1}\Big)+O(\sqrt{\s{NCT}\log \s{T}}\cdot\sigma) \\
     &= O(\s{T}^{-3}\|\fl{P}\|_{\infty})+O(\s{CJ}\sqrt{\s{TV}})+O(\sqrt{\s{NCT}\log \s{T}}\cdot\sigma)
\end{align*}
where we substituted $\Phi=\sqrt{\s{V}\s{C}^2\s{T}^{-1}}$ and hence $\s{J}=O\Big(\log \Big(\frac{1}{\sqrt{\s{V\s{C}^2T}^{-1}}}\min\Big(\|\fl{P}\|_{\infty},\frac{\sigma\sqrt{\mu}}{\log \s{N}}\Big)\Big)\Big)$ in the final step. 
\end{proof}

\begin{proof}[Proof of Theorem \ref{thm:main_LBM2}]

For every cluster $c\in[\s{C}]$, let us define the subset of arms  $\ca{G}_{c,\ell} \equiv \{j \in [\s{M}] \mid \epsilon_{\ell} \le \left|\fl{P}_{uj} - \fl{P}_{u\pi_u(1)} \right| \le \epsilon_{\ell-1} \; \forall \; u \in \ca{C}^{(c)}\}$ for all $\ell>1$ and $\ca{G}_{c,1} \equiv \{j \in [\s{M}] \mid \epsilon_{1} \le \left|\fl{P}_{uj} - \fl{P}_{u\pi_u(1)} \right|  \; \forall \; u \in \ca{C}^{(c)}\}$; $\ca{G}_{c,\ell}$ ($\ca{G}_{c,1}$) corresponds to the subset of arms having a sub-optimality gap that is between $\epsilon_{\ell-1}$ and $\epsilon_{\ell}$ (greater than $\epsilon_1$) for all users belonging to the cluster $\ca{C}^{(c)}$. There is no ambiguity in the definition since all users in the same cluster $\ca{C}^{(c)}$ have the same mean rewards over all arms. 

 Let us also define $\ca{H}_c \equiv \bigcup \limits_{\ell> 1} \s{argmin}_{j\in \ca{G}_{c,\ell}} \left|\fl{P}_{uj}-\fl{P}_{u\pi_u(1)}\right|$ with the understanding that whenever  $\ca{G}_{c,\ell} = \Phi$, there is no $\s{argmin}$ to be counted in the set.
For brevity of notation, let $\Psi_{c,a} \triangleq \fl{P}_{u\pi_u(1)}-\fl{P}_{ua}$ be the sub-optimality gap in the reward of arm $a$ for any user $u$ in cluster $c$,  and $\s{T}^{(1)}_{c,a}$ be the number of times arm $a$ has been pulled by the users in cluster $c$ during the phases indexed by $\ell$ when the users in cluster $c$ belonged to $\ca{B}^{(\ell)}$; $\s{T}^{(2)}_{u,a}$ is the number of times arm $a$ has been pulled by the user $u$ according to the UCB algorithm i.e. when users in cluster $c$ belonged to $[\s{N}]\setminus \ca{B}^{(\ell)}$.

Since the length of the phases increases exponentially with $\ell$, hence the total number of phases is $\widetilde{O}(1)$; hence the size of $\ca{H}_c$ is at most $\widetilde{O}(1)$. Again, for all users $u\in [\s{N}]$ who participated in the UCB algorithm, let us denote $\widehat{\ca{H}}_u$ to be the set of arms ($\left|\widehat{\ca{H}}_u\right|\le \s{C}$ and $\widehat{\ca{H}}_u\supseteq \pi_u(1)$; recall from Lemma \ref{lem:interesting2} that the best arm $\pi_u(1)$ always belongs to the active set of arms) that were used in the UCB algorithm; evidently, the regret incurred due to set $\widehat{\ca{H}}_u$ will be dominated by the set $\{\pi_u(i)\}_{i=1}^{\s{C}}$ which corresponds to the best $\s{C}$ arms for user $u$. We can decompose the regret by using the standard regret decomposition i.e. 
\begin{align*}
    \s{Reg}(\s{T}) & = \bb{E} \left[ \sum_{c\in [\s{C}],a \in [\s{M}]}\Psi_{c,a}  \s{T}^{(1)}_{c,a} \right] +\sum_{u\in [\s{N}],a \in \{\pi_u(s)\}_{s=1}^{\gamma|\s{C}|}}\Psi_{u,a} \bb{E} \s{T}^{(2)}_{u,a} \\
      & \leq \bb{E} \left[ \sum_{c\in [\s{C}],a \in [\s{M}]}\Psi_{c,a}  \s{T}^{(1)}_{c,a} \Big \lvert \bigcap \limits_{\ell} {\cal E}^{(\ell)} \right] + \Pr\Big(\bigcup_{\ell} ({\cal E}^{(\ell)})^c\Big)\bb{E} \left[ \sum_{c\in [\s{C}],a \in [\s{M}]}\Psi_{c,a}  \s{T}^{(1)}_{c,a} \Big \lvert \bigcup \limits_{\ell} ({\cal E}^{(\ell)})^c \right] \\
      &+ \sum_{u\in [\s{N}],a \in \{\pi_u(s)\}_{s=1}^{\gamma|\s{C}|}}\Psi_{u,a} \bb{E} \s{T}^{(2)}_{u,a} 
\end{align*}

We now show the following lemma:

\begin{lemma}
Fix any $\ell> 1$ and cluster $c\in [\s{C}]$.We must have $\bb{E}[\sum_{a\in \ca{G}_{c,\ell}}\s{T}^{(1)}_{c,a} \mid \cup_{\ell} (\ca{E}^{(\ell)})^c ] = \s{C}^{-1}O\Big(\s{T}^{-3}\Big)$ and $\bb{E}[\sum_{a\in \ca{G}_{c,\ell}}\s{T}^{(1)}_{c,a} \mid \cap_{\ell} \ca{E}^{(\ell)}] = \s{C}^{-1}O\Big(\frac{\s{C}^2\s{V}}{\epsilon_{\ell-1}^2}\Big)$ provided that $\ca{G}_{c,\ell}\neq \Phi$.
\end{lemma}

\begin{proof}
Fix cluster $c\in [\s{C}]$. From definition, we know that all arms $a$ in $\ca{G}_{c,\ell}$ satisfy the following for all users $u\in \ca{C}^{(c)}$: $\epsilon_{\ell} \le  \Psi_{c,a} =|\fl{P}_{ua}-\fl{P}_{u\pi_{u}(1)}| \le \epsilon_{\ell-1}$ for $\ell> 1$. In that case, with probability at least $1-\s{T}^{-3}$, the event $\ca{E}^{(j)}$ is true for all $j$ implying that the algorithm is $\epsilon_{\ell}$-good (see Lemma \ref{lem:interesting6}).
Hence, we must have $\Pr(\cup_{\ell} (\ca{E}^{(\ell)})^c)\bb{E}[\sum_{a\in \ca{G}_{c,\ell}}\s{T}^{(1)}_{c,a} \mid \cup_{\ell} (\ca{E}^{(\ell)})^c ] = \s{C}^{-1}O\Big(\s{T}^{-2}\Big)$.
Therefore, conditioning on all $\ca{E}^{(j)}$ being true, by definition using property ($3$), at the beginning of the $\ell^{\s{th}}$ phase, if $u\in \ca{C}^{(c)}\cap\ca{M}^{(\ell,i)}$, it must be the case that $a\not\in \ca{N}^{(\ell,i)}$ for all $a\in \ca{G}_{c,\ell}$. Hence, we must have (by plugging in the sample complexity in Lemma \ref{lem:interesting1} with
with $\s{V}=\sigma^2 \s{C}^2 (\s{C} \bigvee \mu\alpha^{-1})^3 \Big(\s{N}\bigvee\s{MC}\Big)\log^3 (\s{ABCT})$ and $\Delta_{\ell+1}=\epsilon_{\ell}/40\s{C}$)

\begin{align*}
    \bb{E}[\sum_{a\in \ca{G}_{c,\ell}}\s{T}^{(1)}_{c,a}\mid \cap \ca{E}^{(\ell)}] &\le \sum_{j=1}^{\ell} \frac{m_j}{\s{C}} = 
    \frac{1}{\s{C}}O\Big(\sum_{j=1}^{\ell-1} \frac{\s{V}}{\Delta^2_{j+1}}\Big) =  \s{C}^{-1}O\Big(\sum_{j=1}^{\ell-1} \frac{\s{C}^2\s{V}}{\epsilon^2_{j}}\Big)= \s{C}^{-1}O\Big( \frac{\s{C}^2\s{V}}{\epsilon^2_{\ell-1}}\Big)
\end{align*}


where we substituted the fact that  $\epsilon_{\ell}=C'2^{-\ell}\min\Big(\|\fl{P}\|_{\infty},\frac{\sigma\sqrt{\mu}}{\log \s{N}}\Big)$ for some constant $\s{C}'>0$. 
\end{proof}

Note that we will have the following set of equations; we use from definition that for every arm $a\in \ca{G}_{c,\ell}$, there exists a representative arm $\hat{a}$ of $\ca{G}_{c,\ell}$ in $\ca{H}_c$ such that $\Psi_{c,\hat{a}} \le \min(\epsilon_{\ell-1}, 2\Psi_{c,a})$ and $\epsilon_{\ell-1} \le 2\Psi_{c,\hat{a}}$.
\begin{align*}
  &\Pr(\cup_{\ell} (\ca{E}^{(\ell)})^c)\Big(\sum_{c\in [\s{C}]}\sum_{\ell}\bb{E}[\sum_{a\in \ca{G}_{c,\ell}}\Psi_{c,a}\s{T}^{(1)}_{c,a} \mid \cup_{\ell} (\ca{E}^{(\ell)})^c ]\Big) \\ 
  &\le \Pr(\cup_{\ell} (\ca{E}^{(\ell)})^c)\Big(\sum_{c\in [\s{C}]}\sum_{\ell} (\max_{a\in \ca{G}_{c,\ell}} \Psi_{c,a}) \sum_{a \in \ca{G}_{c,\ell}}  \bb{E}[\s{T}^{(1)}_{c,a} \mid \cup_{\ell} (\ca{E}^{(\ell)})^c ]\Big) \le \sum_{c \in [\s{C}],a \in \ca{H}_c} \Psi_{c,a} \s{C}^{-1}O\Big(\s{T}^{-2}\Big).  \end{align*}
 and similarly, we will also have 
  \begin{align*}
    &\sum_{c\in [\s{C}]}\sum_{\ell}\bb{E}[\sum_{a\in \ca{G}_{c,\ell}}\Psi_{c,a}\s{T}^{(1)}_{c,a} \mid \cap_{\ell} \ca{E}^{(\ell)}]\le \sum_{c\in [\s{C}]}\sum_{\ell} (\max_{a\in \ca{G}_{c,\ell}} \Psi_{c,a}) \sum_{a \in \ca{G}_{c,\ell}}  \bb{E}[\s{T}^{(1)}_{c,a} \mid \cap_{\ell} \ca{E}^{(\ell)}] \\
    &=  \sum_{c\in [\s{C}]} (\max_{a\in \ca{G}_{c,1}} \Psi_{c,a}) \sum_{a \in \ca{G}_{c,1}}  \bb{E}[\s{T}^{(1)}_{c,a} \mid \cap_{\ell} \ca{E}^{(\ell)}]+\sum_{c\in [\s{C}]}\sum_{\ell>1}\bb{E}[\sum_{a\in \ca{G}_{c,\ell}}\Psi_{c,a}\s{T}^{(1)}_{c,a} \mid \cap_{\ell} \ca{E}^{(\ell)}] \\
    &= \lr{\fl{P}}_{\infty}\cdot\frac{\s{V}}{\epsilon_1^2}\fl{1}[\ca{G}_{c,1} \neq \Phi]+ 
    \sum_{c \in [\s{C}]}\sum_{\ell>1:\ca{G}_{c,\ell}\neq \Phi} \s{C}^{-1}O\Big(\frac{\s{C}^2\s{V}}{\epsilon_{\ell}}\Big) \le \sum_{c \in [\s{C}],a \in \ca{H}_c} \s{C}^{-1}O\Big(\frac{\s{C}^2\s{V}}{\Psi_{c,a}}\Big). 
\end{align*}

Similarly, from well known analysis of $\s{UCB}$ algorithm, we know that (recall that $\{\pi_u(s)\}_{s=1}^{|\gamma\s{C}|}$ are the top $
\gamma\s{C}$ arms for the user $u$). 
\begin{align*}
    \sum_{u\in [\s{N}], a \in \{\pi_u(s)\}_{s=1}^{\gamma|\s{C}|}}\Psi_{u,a} \bb{E} \s{T}^{(2)}_{u,a} \le \s{N}^{-1}O\Big(\sum_{u\in [\s{N}], a \in \{\pi_u(s)\}_{s=1}^{|\gamma\s{C}|}}
    \Big(\frac{\sigma\log \s{T}}{\Psi_{u,a}}+3\Psi_{u,a}\Big)\Big) 
\end{align*}

Therefore, we can bound the regret from above as $$\s{Reg}(\s{T}) =\lr{\fl{P}}_{\infty}\cdot\frac{\s{V}}{\epsilon_1^2}\fl{1}[\ca{G}_{c,1} \neq \Phi]+ \s{C}^{-1}O\Big(\sum_{c\in [\s{C}], a\in \ca{H}_c} \Psi_{c,a}\s{T}^{-2}+\frac{\s{C}^2\s{V}}{\Psi_{c,a}}\Big)+\s{N}^{-1}O\Big(\sum_{u\in [\s{N}], a \in \{\pi_u(s)\}_{s=1}^{\gamma|\s{C}|}}
    \Big(\frac{\sigma\log \s{T}}{\Psi_{u,a}}+3\Psi_{u,a}\Big)\Big).$$ Loosely speaking, this bound translates as $\s{Reg}(\s{T}) = \widetilde{O}((\s{M}+\s{N})/\Psi)$ where $\Psi$ is the minimum sub-optimality gap; $\widetilde{O}(\cdot)$ hides factors in $\s{C}$ and other logarithmic terms.
\end{proof}

\section{Proofs of Theorems~\ref{thm:dist_free_lower_bounds},~\ref{thm:dist_dep_lower_bounds}}\label{app:lower_bounds}


\subsection{Proofs of Theorem~\ref{thm:dist_free_lower_bounds}, Corollary~\ref{cor:dist_free_lower_bounds_uniform}}
We first derive lower bounds for the following two settings: (a) known cluster assignment, and  (b) known cluster rewards. The lower bound for MAB-LC follows by taking the maximum of these two bounds.


\paragraph{Known Cluster Assignments.} Suppose we know the mapping between users and clusters. In this setting, it is easy to see that the optimal strategy is to treat users within a cluster as a single super-user, and reduce the problem to that of solving $\s{C}$ different multi-armed bandit problems (each corresponding to the $\s{C}$ clusters). 
One could rely on the regret lower bound of stochastic MAB~\citep{ lattimore2020bandit, Cesa-Bianchi:2006, Bubeck12} and provide the following informal proof for the regret lower-bound for our problem. We have $\s{C}$ MAB instances, where the $c^{th}$ instance has $\s{M}$ arms and occurs $\s{T}_c$ times (note that $T_c$ is a random variable).  Since the regret of the $c^{th}$ instance is lower bounded by $0.05\mathbb{E}\left[\min(\sqrt{\s{M}\s{T}_c}, \s{T}_c)\right],$  the overall regret is lower bounded $0.05\sum_{c\in[\s{C}]}\mathbb{E}\left[\min(\sqrt{\s{M}\s{T}_c}, \s{T}_c)\right]$.

We now make the above argument more formal.  At a high level, the proof involves designing problem instances that are hard to separate. We then argue that any algorithm should suffer large regret on at least one of the problems. 
\begin{itemize}
    \item Partition users into $\s{C}$ clusters. Let $\mathcal{N}_c$ be the set of users in cluster $c$. 
    \item Choose indices $(a_1, \dots, a_{\s{C}})$ such that  $a_i\in[\s{M}]$ for all $i$. Note that there are $\s{M}^{\s{C}}$ possible such choices. We are going to define $\s{M}^{\s{C}}$ problem instances each corresponding to a choice of $(a_1, \dots, a_{\s{C}})$. In these problems, each $a_i$ corresponds to the optimal arm in cluster $i$. Define the mean rewards of the $j^{th}$ arm in $i^{th}$ cluster as
    \[
    \fl{X}_{ij} = \begin{cases}
    \frac{1-\epsilon}{2},&\quad \text{if } j\neq a_i,\\
    \frac{1+\epsilon}{2},&\quad \text{otherwise}.
    \end{cases}
    \]
    Let's call this problem instance $\text{Prob}_{a_1, \dots a_{\s{C}}}$.
\end{itemize}
Define problem instance $\text{Prob}_{0, c,  a_1, \dots a_{\s{C}}}$ as follows. It is exactly equal to $\text{Prob}_{a_1, \dots a_{\s{C}}}$ except for one difference. The rewards of all the arms in cluster $c$ are set to $\frac{1-\epsilon}{2}$. 

In the proof, we first consider deterministic algorithms. Using Fubini's theorem, these results can be easily extended to randomized algorithms~\citep{Bubeck12}. Next, we assume $\s{T}_{c}$, the number of appearances of cluster $c$, is a fixed quantity. The final results can simply be obtained by taking expectation over $\{\s{T}_c\}_{c\in[\s{C}]}$. Let $T(c, a)$ be the number of times arm $a$  has been pulled during the appearances of cluster $c$. Then the regret of the algorithm under $\text{Prob}_{a_1,\dots a_{\s{C}}}$ can be written as
\[
\s{Reg}_{a_1,\dots a_{\s{C}}}(\s{T}) = \sum_{c\in [\s{C}]}\sum_{a\in [\s{M}]\setminus\{a_c\}} \epsilon\mathbb{E}_{a_1\dots a_{\s{C}}}[T(c, a)] = \sum_{c\in [\s{C}]} \epsilon\left(\s{T}_c-\mathbb{E}_{a_1,\dots a_{\s{C}}}[T(c, a_c)]\right).
\]
Let $J_{c}$ be a random variable that is drawn according to the discrete distribution $\left(\frac{T(c,0)}{\s{T}_c}, \dots \frac{T(c,\s{M}-1)}{\s{T}_c}\right).$ Then 
\[
\s{Reg}_{a_1,\dots a_{\s{C}}}(\s{T}) = \sum_{c\in [\s{C}]} \epsilon\s{T}_c\left(1-\mathbb{P}_{a_1,\dots a_{\s{C}}}[J_c = a_c]\right).
\]
So, we have
\begin{align}
\label{eqn:regret_lbound_known_cluster_assignments}
\frac{1}{\s{M}^{\s{C}}}\sum_{a_1,\dots a_{\s{C}}} \s{Reg}_{a_1,\dots a_{\s{C}}}(\s{T}) = \frac{1}{\s{M}^{\s{C}}}\sum_{a_1,\dots a_{\s{C}}}\sum_{c\in [\s{C}]} \epsilon\s{T}_c\left(1-\mathbb{P}_{a_1,\dots a_{\s{C}}}[J_c = a_c]\right).
\end{align}
Next, from Pinsker's inequality, we have
\begin{align*}
    \mathbb{P}_{a_1,\dots a_{\s{C}}}[J_c = a_c] &\leq \mathbb{P}_{0, c, a_1,\dots a_{\s{C}}}[J_c = a_c] + \sqrt{\frac{1}{2}KL(\mathbb{P}_{0, c, a_1,\dots a_{\s{C}}}, \mathbb{P}_{ a_1,\dots a_{\s{C}}})}\\
    & \stackrel{(a)}{\leq} \mathbb{P}_{0, c, a_1,\dots a_{\s{C}}}[J_c = a_c] + \sqrt{\frac{1}{2}KL\left(\frac{1-\epsilon}{2}, \frac{1+\epsilon}{2}\right)\mathbb{E}_{0, c, a_1,\dots a_{\s{C}}}[T(c, a_c)]}.
\end{align*}
Inequality $(a)$ simply follows from data processing inequality and the definition of KL divergence (also see Equation 3.20 of \cite{Bubeck12}).
Next, we take the average of the LHS and the average of the RHS in the above equation over all possible values of $a_c$. This gives us
\begin{align*}
    \frac{1}{\s{M}}\sum_{a'\in[\s{M}]}\mathbb{P}_{a_1,\dots, a_c=a',\dots a_{\s{C}}}[J_c = a'] & \leq \frac{1}{\s{M}} \sum_{a'\in[\s{M}]}\mathbb{P}_{0, c, a_1,\dots a_c=a',\dots a_{\s{C}}}[J_c = a'] \\
    &\quad \quad + \frac{1}{\s{M}}\sum_{a'\in[\s{M}]}\sqrt{\frac{1}{2}KL\left(\frac{1-\epsilon}{2}, \frac{1+\epsilon}{2}\right)\mathbb{E}_{0, c, a_1,\dots a_c=a',\dots a_{\s{C}}}[T(c, a')]}
\end{align*}
Now observe that  $\text{Prob}_{0, c,  a_1, \dots a_c, \dots a_{\s{C}}}$ doesn't depend on $a_c$ (that is, $\text{Prob}_{0, c,  a_1, \dots a_c, \dots a_{\s{C}}}$ are the same problem instances for all values of $a_c$). So
\[
\sum_{a'\in[\s{M}]}\mathbb{P}_{0, c, a_1,\dots a_c=a',\dots a_{\s{C}}}[J_c = a'] = 1.
\]
The second term in the RHS of the previous inequality can be bounded using Cauchy–Schwarz inequality
\begin{align*}
    \frac{1}{M}\sum_{a'\in[\s{M}]}\sqrt{\frac{1}{2}\mathbb{E}_{0, c, a_1,\dots a_c=a',\dots a_{\s{C}}}[T(c, a')]} &\leq \sqrt{\frac{1}{2\s{M}}\sum_{a'\in[\s{M}]}\mathbb{E}_{0, c, a_1,\dots a_c=a',\dots a_{\s{C}}}[T(c, a')]}\\
    & = \sqrt{\frac{1}{2\s{M}}\s{T}_c}
\end{align*}
Substituting this in the previous inequality gives us
\begin{align*}
    \frac{1}{\s{M}}\sum_{a'\in[\s{M}]}\mathbb{P}_{a_1,\dots, a_c=a',\dots a_{\s{C}}}[J_c = a'] & \leq \frac{1}{\s{M}}  + \sqrt{\frac{\s{T}_c}{2\s{M}}KL\left(\frac{1-\epsilon}{2}, \frac{1+\epsilon}{2}\right)}
\end{align*}
Substituting the above inequality in Equation~\eqref{eqn:regret_lbound_known_cluster_assignments}, we get
\begin{align*}
\max_{a_1,\dots a_{\s{C}}} \s{Reg}_{a_1,\dots a_{\s{C}}}(\s{T}) \geq \frac{1}{\s{M}^{\s{C}}}\sum_{a_1,\dots a_{\s{C}}} \s{Reg}_{a_1,\dots a_{\s{C}}}(\s{T}) \geq \sum_{c\in [\s{C}]} \epsilon\s{T}_c\left(1-\frac{1}{\s{M}} - \sqrt{\frac{\epsilon \s{T}_c}{2\s{M}}\log{\frac{1+\epsilon}{1-\epsilon}}}\right).
\end{align*}
Finally, taking expectation over $\s{T}_c$, and choosing best possible $\epsilon$, we get the following lower bound on the worst-case regret
\begin{align*}
\max_{a_1,\dots a_{\s{C}}} \s{Reg}_{a_1,\dots a_{\s{C}}}(\s{T}) &\geq  \max_{\epsilon}\sum_{c\in [\s{C}]} \epsilon\mathbb{E}\left[\s{T}_c\left(1-\frac{1}{\s{M}} - \sqrt{\frac{\epsilon \s{T}_c}{2\s{M}}\log{\frac{1+\epsilon}{1-\epsilon}}}\right)\right]\\
&\stackrel{(a)}{\geq} \max_{\epsilon}\sum_{c\in [\s{C}]} \epsilon\mathbb{E}\left[\s{T}_c\left(1-\frac{1}{\s{M}} - \sqrt{\frac{\epsilon^2 \s{T}_c}{\s{M}}}\right)\right],
\end{align*}
where $(a)$ follows from the  the fact that $\log{\frac{1+\epsilon}{1-\epsilon}} \leq 2\epsilon$. Rewriting the RHS in the above equation, we get
\begin{align*}
\max_{a_1,\dots a_{\s{C}}} \s{Reg}_{a_1,\dots a_{\s{C}}}(\s{T}) &\geq \max_{\epsilon}\sum_{c\in [\s{C}]} \epsilon\left(1-\frac{1}{\s{M}}\right)\mathbb{E}[\s{T}_c]-\frac{\epsilon^2}{\s{M}^{1/2}}\mathbb{E}[\s{T}_c^{3/2}]\\
&= \max_{\epsilon} \epsilon\left(1-\frac{1}{\s{M}}\right)\s{T}-\epsilon^2\frac{\sum_{c\in [\s{C}]}\mathbb{E}[\s{T}_c^{3/2}]}{\s{M}^{1/2}}.
\end{align*}
We now focus on maximizing the RHS of the above equation. Note that the objective is quadratic in $\epsilon$. So one can obtain an exact expression for the optimal value of $\epsilon$.
If $\left(\s{M}^{1/2}-\s{M}^{-1/2}\right) < 2\frac{\sum_{c\in [\s{C}]}\mathbb{E}[\s{T}_c^{3/2}]}{\s{T}}$, then the RHS is given by
\[
\left(\s{M}^{1/2}-\s{M}^{-1/2}\right)\frac{\s{T}^2}{4\sum_{c\in[\s{C}]}\mathbb{E}[\s{T}_c^{3/2}]}.
\]
On the other hand, if $\left(\s{M}^{1/2}-\s{M}^{-1/2}\right) > 2\frac{\sum_{c\in [\s{C}]}\mathbb{E}[\s{T}_c^{3/2}]}{\s{T}}$, the RHS is given by
\[
\left(1-\frac{1}{\s{M}}\right)\s{T}-\frac{\sum_{c\in [\s{C}]}\mathbb{E}[\s{T}_c^{3/2}]}{\s{M}^{1/2}}.
\]
For the special case where $\probuser$ is uniform and the clusters have the same size, we have $\mathbb{E}[\s{T}_c^{3/2}] = \Theta(\s{T}^{3/2}\s{C}^{-3/2})$. Substituting this in the above bounds, we get
\begin{align}
\label{eqn:minimax_lb_known_clusters}
\max_{a_1,\dots a_{\s{C}}} \s{Reg}_{a_1,\dots a_{\s{C}}}(\s{T}) \geq \begin{cases}
0.05\sqrt{\s{MCT}}, \quad &\text{if } \s{T}> 0.5\s{MC}\\
0.05\s{T} \quad &\text{otherwise}.
\end{cases}
\end{align}

\paragraph{Known Cluster Rewards.} 
We now consider the setting where we know the reward distributions of arms in each cluster. In this setting, it is easy to see that the optimal strategy is to solve a $\s{C}$-armed MAB problem for each user, where the $\s{C}$ arms correspond to the best arms in each of the $\s{C}$ clusters. That is, with the knowledge of cluster rewards, we can reduce the MAB-LC problem to that of solving $\s{N}$ multi-armed bandit problems each with $\s{C}$ arms. Here each MAB problem corresponds to a user. For this problem, we already derived lower bounds above. In particular, in the known cluster assignment setting, we derived lower bounds for solving $\s{C}$ MAB instances each with $\s{M}$ arms. So we could rely on the above lower bounds to derive lower bounds for the known cluster reward setting (we just replace $\s{C}$ with $\s{N}$ and $\s{M}$ with $\s{C}$ in the bounds).
When $\probuser$ is uniform and the clusters have the same size, we could rely on Equation~\eqref{eqn:minimax_lb_known_clusters} and obtain the following minimax lower bounds
\begin{align}
\label{eqn:minimax_lb_known_cluster_rewards}
\inf\sup \s{Reg}(\s{T}) \geq \begin{cases}
0.05\sqrt{\s{NCT}}, \quad &\text{if } \s{T}> 0.5\s{NC}\\
0.05\s{T} \quad &\text{otherwise}.
\end{cases}
\end{align}
For the more general case of non-uniform $\probuser$ and uneven cluster sizes, we obtain
\begin{align}
\label{eqn:minimax_lb_known_cluster_rewards_general}
\inf\sup \s{Reg}(\s{T}) \geq \begin{cases}
\left(\s{C}^{1/2}-\s{C}^{-1/2}\right)\frac{\s{T}^2}{4\sum_{b\in[\s{N}]}\mathbb{E}[\s{T}_b^{3/2}]}, \quad &\text{if } \left(\s{C}^{1/2}-\s{C}^{-1/2}\right) < 2\frac{\sum_{b\in [\s{N}]}\mathbb{E}[\s{T}_b^{3/2}]}{\s{T}}\\
\left(1-\frac{1}{\s{C}}\right)\s{T}-\frac{\sum_{b\in [\s{N}]}\mathbb{E}[\s{T}_b^{3/2}]}{\s{C}^{1/2}} \quad &\text{otherwise}.
\end{cases}
\end{align}

\subsection{Proof of Theorem~\ref{thm:dist_dep_lower_bounds}}

Here is a high level idea of the proof. We find two bandit instances that are close enough to each other but the behaviour of any uniformly efficient algorithm is totally different in the two instances. 

\paragraph{Background.} Similar to lower bounding techniques used in the MAB literature, our proof relies on data processing inequality~\citep{kaufmann2020contributions}.  Let $\mathbf{\mu}_1, \mathbf{\mu}_2$ be two stochastic $K$-armed bandit models. Let $\mathcal{F}_t$ be the $\sigma$-algebra generated by the observations available until round $t$. Let's suppose $\tau$ is the stopping time, and let $I_{\tau}$ be the information available until round $\tau$. Then for any event $\mathcal{E} \in \mathcal{F}_{\tau}$, the data processing inequality tells us
\[
KL(\mathbb{P}_{\mu_1}^{I_{\tau}}, \mathbb{P}_{\mu_2}^{I_{\tau}}) \geq KL(\mathbb{P}_{\mu_1}(\mathcal{E}), \mathbb{P}_{\mu_2}(\mathcal{E}))
\]
Moreover, by definition of KL divergence we have
\[
KL(\mathbb{P}_{\mu_1}^{I_{\tau}}, \mathbb{P}_{\mu_2}^{I_{\tau}}) \stackrel{(a)}{=} \mathbb{E}_{\mu_1}\left[L_{\tau}(\mu_1,\mu_2)\right] \stackrel{(b)}{=} \sum_{k=1}^K\mathbb{E}_{\mu_1}\left[T(k)\right]KL(\mu_{1,a}, \mu_{2,a}).
\]
$L_{\tau}$ in the above equation is the log-likelihood ratio, $T(k)$ is the number of pulls of arm $k$, and $\mu_{1,a}$ is the distribution of rewards of arm $a$ in bandit model $\mu_1$.
\paragraph{Main Proof.} Let $\pi:[\s{N}]\to[\s{C}]$ be the mapping from users to clusters and let $\fl{X}\in\mathbb{R}^{\s{C}\times \s{M}}$ be the mean rewards of arms in the clusters. Let $(\pi_1, \fl{X}_1), (\pi_2, \fl{X}_2)$ be two MAB-LC models. Let $a_{1,c}^*$ be the optimal arm in the $c^{th}$ cluster of the first model, and $a_{2,c}^*$ be the optimal arm in the second model. Let $\bar{\s{C}}$ be the set of clusters for which $a_{1,c}^* \neq  a_{2,c}^*$. Let $T(c, a)$ be the number of times arm $a$  has been pulled when cluster $c$ appeared during the course of the algorithm. Let $\mathcal{E}_{\s{T}}$ be the following event
\[
\mathcal{E}_{\s{T}} = \left\lbrace \sum_{c\in[\bar{\s{C}}]}T(c,a_{1,c}^*) \leq \s{T}/2\right\rbrace
\]
Intuitively, $\mathcal{E}_{\s{T}}$ has a small probability under $(\pi_1, \fl{X}_1)$ where the optimal arms should be selected a lot. Moreover, $\mathcal{E}_{\s{T}}$ has large probability under $(\pi_2, \fl{X}_2)$ as the event only contains sub-optimal arms. This can be formally proved using Markov's inequality as follows
\begin{align*}
    &\mathbb{P}_1(\mathcal{E}_{\s{T}}) = \mathbb{P}_1\left(\sum_{c\in[\bar{\s{C}}]}T(c,a_{1,c}^*) \leq \s{T}/2\right)  = \mathbb{P}_1\left(\sum_{c\in[\bar{\s{C}}]} \sum_{a\in [\s{M}]\setminus\{a_{1,c}^*\}}T(c,a) > \s{T}/2\right) \leq  \frac{2\sum_{c\in[\bar{\s{C}}]} \sum_{a\in [\s{M}]\setminus\{a_{1,c}^*\}}\mathbb{E}_1\left[T(c,a)\right]}{\s{T}}\\
    &\mathbb{P}_2(\bar{\mathcal{E}}_{\s{T}}) = \mathbb{P}_2\left(\sum_{c\in[\bar{\s{C}}]}T(c,a_{1,c}^*) > \s{T}/2\right) \leq  \frac{2\sum_{c\in[\bar{\s{C}}]}\mathbb{E}_2\left[T(c,a_{1,c}^*)\right]}{\s{T}}.
\end{align*}
Let $\zeta_1 = \sum_{c\in[\bar{\s{C}}]} \sum_{a\in [\s{M}\setminus\{a_{1,c}^*\}]}\mathbb{E}_1\left[T(c,a)\right]$ and $\zeta_2 = \sum_{c\in[\bar{\s{C}}]}\mathbb{E}_2\left[T(c,a_{1,c}^*)\right]$. Based on our assumptions of uniform efficiency of the algorithm, we know that $\zeta_1, \zeta_2 = o(T^{\alpha})$ for all $\alpha \in (0, 1]$. Next, observe that the KL divergence between two bernoulli distributions can be lower bounded as $KL(p,q) \geq (1-p)\log\left(1/(1-q)\right)-\log{2}$~\citep{garivier2021nonasymptotic}. Using this, we have
\begin{align*}
    KL(\mathbb{P}_{1}(\mathcal{E}_{\s{T}}), \mathbb{P}_{2}(\mathcal{E}_{\s{T}})) \geq \left(1-\frac{2\zeta_1}{\s{T}}\right)\log\left(\frac{\s{T}}{2\zeta_2}\right) - \log{2} \approx \log{T}.
\end{align*}
This shows that 
\begin{align}
\label{eqn:general_lb_instance_dep}
    \lim_{T\to\infty}\frac{KL(\mathbb{P}_{1}^{I_{\s{T}}}, \mathbb{P}_{2}^{I_{\s{T}}})}{\log{\s{T}}} \geq 1.
\end{align}

We now rely on this result to prove Theorem~\ref{thm:dist_dep_lower_bounds}. All we need to do is construct interesting bandit models that are hard to separate. The first set of models we construct is as follows. Let $(\pi_1, \fl{X}_1)$ be any MAB-LC model. Construct $(\pi_2, \fl{X}_2)$ from $(\pi_1, \fl{X}_1)$ as follows: $\pi_2 = \pi_1$, $\fl{X}_2$ is same as $\fl{X}_1$ for all cluster-arm pairs, except for one location. We take a sub-optimal  arm  $a'$ in cluster $c$ and make its mean reward to be slightly larger than the mean reward of the best arm in cluster $c$ (i.e., $\fl{X}_2[c, a'] = \fl{X}_1[c, a_{1,c}^*]+\epsilon$ for some $\epsilon \to 0$). Applying the above result on this model pair gives us
\[
\lim_{\s{T}\to\infty} \frac{\mathbb{E}_1\left[T(c, a')\right]KL(\fl{X}_{1}[c,a'], \fl{X}_1[c,a_{1,c}^*])}{\log{\s{T}}} \geq 1.
\]
We now rely on the following upper bound on the KL divergence between two Bernoulli distributions
\[
KL(p,q) \leq \frac{(p-q)^2}{q(1-q)}.
\]
Using this in the previous inequality, we get
\[
\lim_{\s{T}\to\infty} \frac{\mathbb{E}_1\left[T(c, a')\right] (\fl{X}_1[c,a_{1,c}^*] - \fl{X}_1[c,a'])}{\log{\s{T}}} \geq \frac{\fl{X}_1[c,a_{1,c}^*](1-\fl{X}_1[c,a_{1,c}^*])}{(\fl{X}_1[c,a_{1,c}^*] - \fl{X}_1[c,a'])}.
\]
Note that this result holds for any $c,a'$. Summing over all possible values of $c,a'$ gives us the required result in Theorem~\ref{thm:dist_dep_lower_bounds}
\begin{align}
\label{eqn:dist_dep_lower_bounds_ac}
\lim_{\s{T}\to\infty}\frac{\s{Reg}(\s{T})}{\log{\s{T}}} \geq \sum_{c\in[\s{C}]} \sum_{a \neq a_c^*}\frac{\fl{X}[c,a_{c}^*](1-\fl{X}[c,a_{c}^*])}{(\fl{X}[c,a_{c}^*] - \fl{X}[c,a'])}.
\end{align}

\paragraph{Tighter Bounds.} Note that the above lower bound didn't explicitly depend on the number of users  $\s{N}$. We now derive a different bound that depends on $\s{N}$. Let $(\pi_1, \fl{X}_1)$ be any MAB-LC model. Construct $(\pi_2, \fl{X}_2)$ from $(\pi_1, \fl{X}_1)$ as follows: $\fl{X}_2$ is exactly same as $\fl{X}_1$. Moreover, $\pi_2$ is same as $\pi_1$ for all users except for a particular user $b$. To be precise, $\pi_2$ places $b$ in a cluster that is different from $\pi_1(b)$.  Let's call $\pi_1(b)$ as $c$ and $\pi_2(b)$ as $c'$. One can show that Equation~\eqref{eqn:general_lb_instance_dep} holds for this setting. The proof of this uses similar arguments as those used to prove Equation~\eqref{eqn:general_lb_instance_dep} (the only thing that changes is our definition of the event $\mathcal{E}_{\s{T}}$ which now includes all the users $b$ that have different optimal arms across the two MAB-LC models). Applying Equation~\eqref{eqn:general_lb_instance_dep} to this setting gives us the following
\[
\lim_{\s{T}\to\infty} \frac{\sum_{a\in[\s{M}]}\mathbb{E}_1\left[T(b, a)\right]KL(\fl{X}_{1}[c,a], \fl{X}_1[c',a])}{\log{\s{T}}} \geq 1.
\]
Here $T(b,a)$ is the number of times arm $a$ has been pulled for user $b$.
From the above inequality, the regret of user $b$ can be lower bounded as (this follows from  Holder's inequality: $\sum_{i}|a_ib_i| \leq (\sum_{i}|a_i|)\max_i |b_i|$)
\begin{align*}
\lim_{\s{T}\to\infty} \frac{\sum_{a\in[\s{M}]}\mathbb{E}_1\left[T(b, a)\right](\fl{X}_1[c,a_{1,c}^*] - \fl{X}_1[c,a])}{\log{\s{T}}}  \geq \min_{a\in[\s{M}]} \frac{ (\fl{X}_1[c,a_{1,c}^*] - \fl{X}_1[c,a]) }{KL(\fl{X}_{1}[c,a], \fl{X}_1[c',a])}.
\end{align*}
Furthermore, optimizing the RHS over the choice of $c'$ gives us
\begin{align*}
\lim_{\s{T}\to\infty} \frac{\sum_{a\in[\s{M}]}\mathbb{E}_1\left[T(b, a)\right](\fl{X}_1[c,a_{1,c}^*] - \fl{X}_1[c,a])}{\log{\s{T}}}  \geq \max_{c'\neq c}\min_{a\in[\s{M}]} \frac{ (\fl{X}_1[c,a_{1,c}^*] - \fl{X}_1[c,a]) }{KL(\fl{X}_{1}[c,a], \fl{X}_1[c',a])}.
\end{align*}
This shows that the overall regret (for all the users) can be lower bounded as
\[
\lim_{\s{T}\to\infty}\frac{\s{Reg}(\s{T})}{\log{\s{T}}} \geq \sum_{b\in[\s{N}]} \max_{c'\neq \pi(b)}\min_{a\in[\s{M}]} \frac{ (\fl{X}[\pi(b),a_{\pi(b)}^*] - \fl{X}[\pi(b),a]) }{KL(\fl{X}[\pi(b),a], \fl{X}[c',a])}.
\]
So, a tighter lower bound for the regret can be obtained by taking a maximum of this regret lower bound and the lower bound in Equation~\eqref{eqn:dist_dep_lower_bounds_ac}.

\section{Missing Details in Section \ref{sec:GLBM}}\label{app:glbm}

\subsection{Preliminaries}

\begin{lemma}\label{lemma:cnd_bound}
(Conditional Number bounds) Let $\mathbf{P} \in \mathbb{R}^{\s{N} \times \s{M}}$ be a matrix with non-zero singular values $\sigma_1 > \sigma_2  \ldots ... > \sigma_r$ for some $\s{M} \ge r>0$. Consider a sub-matrix $\fl{P}_{\ca{S}}$ which is formed by taking all rows of matrix $\fl{P}$ and columns from a set $\ca{S} \subset [\s{M}]$ of indices. Suppose that the row-space of $\fl{P}_{\ca{S}}$ is a non-trivial vector space.  Then, $\frac{ \lambda_{\mathrm{max}} (\fl{P}_{\ca{S}}^T \fl{P}_{\ca{S}}) }{ \lambda_{\mathrm{min}} (\fl{P}_{\ca{S}}^T  \fl{P}_{\ca{S}}) } \leq \frac{ \lambda_{\mathrm{max}} (\fl{P}^T \fl{P})  } { \lambda_{\mathrm{min}} (\fl{P}^T \fl{P}) }$ where $\lambda_{\max}(\cdot),\lambda_{\min}(\cdot)$ corresponds to the largest and smallest non-zero eigenvalues of the corresponding matrix.
\end{lemma}

\begin{proof}
Let $\fl{P}_i$ be the $i$-th row of $\fl{P}$. 
\begin{align}
\lambda_{\mathrm{max}} (\fl{P}^T \fl{P}) & = \sup \limits_{\fl{x}:\lVert \fl{x} \rVert_2=1 } \fl{x}^T \fl{P}^T \fl{P} \fl{x} = \sup \limits_{\fl{x}:\lVert \fl{x} \rVert_2=1} \sum_i \fl{x}^T \fl{P}_i^T \fl{P}_i \fl{x}  \ge \sup \limits_{\fl{x}:\lVert \fl{x} \rVert_2=1,~\fl{x}_{[\s{M}]-\ca{S}}=\fl{0}} \sum_i \fl{x}^T \fl{P}_i^T \fl{P}_i \fl{x} \nonumber \\
& = \sup \limits_{\fl{x}:\lVert \fl{x} \rVert_2=1} \sum_i \fl{x}^T \fl{P}_{i|\ca{S}}^T \fl{P}_{i|\ca{S}} \fl{x} = \lambda_{\mathrm{max}} (\fl{P}_{\ca{S}}^T \fl{P}_{\ca{S}}).
\end{align}
Let $\ca{K}$ be the row-space of $\fl{P}$. Let $\ca{K}'$ be the row space of $\fl{P}_{\ca{S}}$.
\begin{align}
\lambda_{\mathrm{min}} (\fl{P}^T \fl{P}) & = \inf \limits_{\fl{x} \in \ca{K} :\lVert \fl{x} \rVert_2=1 } \fl{x}^T \fl{P}^T \fl{P} \fl{x} = \inf \limits_{\fl{x} \in \ca{K}:\lVert \fl{x} \rVert_2=1} \sum_i \fl{x}^T \fl{P}_i^T \fl{P}_i \fl{x}  \overset{(a)}{\le} \inf \limits_{\fl{x} \in \ca{K}:\lVert \fl{x} \rVert_2=1,~\fl{x}_{[\s{M}]-\ca{S}}=\fl{0}} \sum_i \fl{x}^T \fl{P}_i^T \fl{P}_i \fl{x} \nonumber \\
& \le \inf \limits_{\fl{x} \in \ca{K}':\lVert \fl{x} \rVert_2=1} \sum_i \fl{x}^T \fl{P}_{i|\ca{S}}^T \fl{P}_{i|\ca{S}} \fl{x} = \lambda_{\mathrm{min}} (\fl{P}_{\ca{S}}^T \fl{P}_{\ca{S}}).
\end{align}
(a) is due to the fact that there is at least one $\fl{x} \neq 0$ in the row-space of $\fl{P}$ with only non-zero entries in $\ca{S}$.
\end{proof}

\begin{lemmau}[Restatement of Lemma \ref{lem:condition_num2}]

Suppose Assumption \ref{assum:matrix2} is true.
  Consider a sub-matrix $\fl{P}_{\s{sub}}$ of $\fl{P}$ having non-zero singular values $\lambda'_1>\dots> \lambda'_{\s{C}'}$ (for $\s{C}'\le \s{C}$). Then, provided $\fl{P}_{\s{sub}}$ is non-zero, we have $\frac{\lambda'_1}{\lambda'_{\s{C}'}} \le \frac{\lambda_1}{\lambda_{\s{C}}}$.
 
\end{lemmau}
\begin{proof}

 Let $\ca{S}' \subset [\s{N}],~\ca{S} \subset [\s{M}]$. Let $\fl{P}_{\ca{S}',\ca{S}}$ be the sub-matrix formed by choosing column indices $\ca{S}'$ and $\ca{S}$ from matrix $\fl{P} \in \mathbb{R}^{\s{N} \times \s{M}}$. Let $\sigma_{\mathrm{max}}(\cdot)$ and $\sigma_{\mathrm{min}}(\cdot)$ be the largest and smallest non-zero singular values. Suppose, that $\mathrm{dim}(\mathrm{rowspace}(\fl{P}_{[\s{N}],\ca{S}})) \geq 1$. Let $\mathrm{dim}(\mathrm{colspace}(\fl{P}_{\ca{S}',\ca{S}})) \geq 1$. 
Applying Lemma \ref{lemma:cnd_bound} to $\fl{P}_{:,\ca{S}}$ and $\fl{P}$ using the fact that rowspace of $\fl{P}_{:,\ca{S}}$ is a non-trivial vector space, we get: 
$\frac{ \lambda_{\mathrm{max}} (\fl{P}_{:,\ca{S}}^T \fl{P}_{:,\ca{S}}) }{ \lambda_{\mathrm{min}} (\fl{P}_{:,\ca{S}}^T  \fl{P}_{:,\ca{S}}) } \leq \frac{ \lambda_{\mathrm{max}} (\fl{P}^T \fl{P})  } { \lambda_{\mathrm{min}} (\fl{P}^T \fl{P}) }$.

Again applying Lemma \ref{lemma:cnd_bound} to $\fl{P}^T_{\ca{S}',\ca{S}}$ and $\fl{P}^T_{:,\ca{S}}$ and using the fact that rowspace of $\fl{P}^T_{\ca{S}',\ca{S}}$ is a non-trivial vector space, we get:
 $\frac{ \lambda_{\mathrm{max}} (\fl{P}_{\ca{S}',\ca{S}} \fl{P}^T_{\ca{S}',\ca{S}}) }{ \lambda_{\mathrm{min}} (\fl{P}_{\ca{S}',\ca{S}}  \fl{P}^T_{\ca{S}',\ca{S}}) } \leq \frac{ \lambda_{\mathrm{max}} (\fl{P}_{:,\ca{S}} \fl{P}^T_{:,\ca{S}})  } { \lambda_{\mathrm{min}} (\fl{P}_{:,\ca{S}} \fl{P}^T_{:,\ca{S}}) } $. We observe that $\frac{ \lambda_{\mathrm{max}} (\fl{P}_{:,\ca{S}} \fl{P}^T_{:,\ca{S}})  } { \lambda_{\mathrm{min}} (\fl{P}_{:,\ca{S}} \fl{P}^T_{:,\ca{S}}) }=\frac{ \lambda_{\mathrm{max}} (\fl{P}_{:,\ca{S}}^T \fl{P}_{:,\ca{S}}) }{ \lambda_{\mathrm{min}} (\fl{P}_{:,\ca{S}}^T  \fl{P}_{:,\ca{S}}) }$. This implies:
 $\frac{ \lambda_{\mathrm{max}} (\fl{P}_{\ca{S}',\ca{S}} \fl{P}^T_{\ca{S}',\ca{S}}) }{ \lambda_{\mathrm{min}} (\fl{P}_{\ca{S}',\ca{S}}  \fl{P}^T_{\ca{S}',\ca{S}}) } \leq \frac{ \lambda_{\mathrm{max}} (\fl{P}^T \fl{P})  } { \lambda_{\mathrm{min}} (\fl{P}^T \fl{P}) } $. Taking square root on both sides yields the result.
\end{proof}

\begin{lemmau}[Restatement of Lemma \ref{lem:incoherence2}]
Suppose Assumption \ref{assum:matrix2} is true. Consider a sub-matrix $\fl{P}_{\s{sub}}\in \bb{R}^{\fl{B}'\times \fl{A}'}$ (with SVD decomposition $\fl{P}_{\s{sub}}=\widetilde{\fl{U}}\widetilde{\f{\Sigma}}\widetilde{\fl{V}}$) of $\fl{P}$ whose rows correspond to a nice subset of users. Then, provided the number of columns in $\fl{P}_{\s{sub}}$ is larger than $\gamma\s{C}$, we must have $\lr{\widetilde{\fl{U}}}_{2,\infty} \le \sqrt{\frac{\s{C}\tau}{\s{N}'}}$ and $\lr{\widetilde{\fl{V}}}_{2,\infty} \le \sqrt{\frac{\mu \s{C}}{\alpha\s{M}'}}$.  
\end{lemmau}

\begin{proof}[Proof of Lemma \ref{lem:incoherence2}]

Suppose the reward matrix $\fl{P}$ has the SVD decomposition $\fl{U}\f{\Sigma}\fl{V}^{\s{T}}$. We are looking at a sub-matrix of $\fl{P}$ denoted by $\fl{P}_{\s{sub}}\in \bb{R}^{\s{N}'\times \s{M}'}$ which can be represented as $\fl{U}_{\s{sub}}\f{\Sigma}\fl{V}_{\s{sub}}^{\s{T}}$ where $\fl{U}_{\s{sub}},\fl{V}_{\s{sub}}$ are sub-matrices of $\fl{U},\fl{V}$ respectively and are not necessarily orthogonal. Here, the rows in $\fl{P}_{\s{sub}}$ corresponds to a union of cluster of users $\cup_{j \in \ca{A}} \s{C}^{(j)}$ for $\ca{A}\in [\s{C}]$ and the columns in $\fl{P}_{\s{sub}}$ corresponds to the trimmed set of arms in $[\s{M}]$. Hence we can write (provided $\fl{U}_{\s{sub}}$,$\fl{V}_{\s{sub}}$ are invertible) 
\begin{align*}
    \fl{P}_{\s{sub}} &= \fl{U}_{\s{sub}}(\fl{U}_{\s{sub}}^{\s{T}}\fl{U}_{\s{sub}})^{-1/2}(\fl{U}_{\s{sub}}^{\s{T}}\fl{U}_{\s{sub}})^{1/2}\f{\Sigma}(\fl{V}_{\s{sub}}^{\s{T}}\fl{V}_{\s{sub}})^{1/2}(\fl{V}_{\s{sub}}^{\s{T}}\fl{V}_{\s{sub}})^{-1/2}\fl{V}_{\s{sub}}^{\s{T}} \\
    &= \fl{U}_{\s{sub}}(\fl{U}_{\s{sub}}^{\s{T}}\fl{U}_{\s{sub}})^{-1/2}\widehat{\fl{U}}\widehat{\fl{\Sigma}}\widehat{\fl{V}}^{\s{T}}(\fl{V}_{\s{sub}}^{\s{T}}\fl{V}_{\s{sub}})^{-1/2}\fl{V}_{\s{sub}}^{\s{T}} 
\end{align*}
where $\widehat{\fl{U}}\widehat{\fl{\Sigma}}\widehat{\fl{V}}^{\s{T}}$ is the SVD of the matrix $(\fl{U}_{\s{sub}}^{\s{T}}\fl{U}_{\s{sub}})^{1/2}\f{\Sigma}(\fl{V}_{\s{sub}}^{\s{T}}\fl{V}_{\s{sub}})^{1/2}$. 
Since $\widehat{\fl{U}}$ is orthogonal, $\fl{U}_{\s{sub}}(\fl{U}_{\s{sub}}^{\s{T}}\fl{U}_{\s{sub}})^{-1/2}\widehat{\fl{U}}$ is orthogonal. 
Similarly, $\widehat{\fl{V}}^{\s{T}}(\fl{V}_{\s{sub}}^{\s{T}}\fl{V}_{\s{sub}})^{-1/2}\fl{V}_{\s{sub}}^{\s{T}}$ is orthogonal as well whereas $\widehat{\f{\Sigma}}$ is diagonal. Hence $\fl{U}_{\s{sub}}(\fl{U}_{\s{sub}}^{\s{T}}\fl{U}_{\s{sub}})^{-1/2}\widehat{\fl{U}}\widehat{\fl{\Sigma}}\widehat{\fl{V}}^{\s{T}}(\fl{V}_{\s{sub}}^{\s{T}}\fl{V}_{\s{sub}})^{-1/2}\fl{V}_{\s{sub}}^{\s{T}}$ indeed corresponds to the SVD of $\fl{P}_{\s{sub}}$ and we only need to argue about the incoherence of the matrices $\fl{U}_{\s{sub}}(\fl{U}_{\s{sub}}^{\s{T}}\fl{U}_{\s{sub}})^{-1/2}\widehat{\fl{U}}$ and $(\widehat{\fl{V}}^{\s{T}}(\fl{V}_{\s{sub}}^{\s{T}}\fl{V}_{\s{sub}})^{-1/2}\fl{V}_{\s{sub}}^{\s{T}})^{\s{T}}$. 
We have
\begin{align*}
    &\max_i \|\widehat{\fl{U}}(\fl{U}_{\s{sub}}^{\s{T}}\fl{U}_{\s{sub}})^{-1/2}\fl{U}_{\s{sub}}^{\s{T}}\fl{e}_i\| \le \max_i \|(\fl{U}_{\s{sub}}^{\s{T}}\fl{U}_{\s{sub}})^{-1/2}\fl{U}_{\s{sub}}^{\s{T}}\fl{e}_i\| \\
    &\le \frac{\|\fl{U}_{\s{sub}}\|_{2,\infty}}{\sqrt{\lambda_{\min}(\fl{U}_{\s{sub}}^{\s{T}}\fl{U}_{\s{sub}})}} \le \frac{\lr{\fl{U}}_{2,\infty}}{\sqrt{\lambda_{\min}(\fl{U}_{\s{sub}}^{\s{T}}\fl{U}_{\s{sub}})}} \le \sqrt{\frac{\mu \s{C}}{\s{N}}} \cdot \frac{1}{\sqrt{\lambda_{\min}(\fl{U}_{\s{sub}}^{\s{T}}\fl{U}_{\s{sub}})}} \le \sqrt{\frac{\mu \s{C}^2}{\s{N}\beta\tau\left|\ca{A}\right|}} \le  \sqrt{\frac{\mu \s{C}}{\beta\s{N}'}}
\end{align*}
where we used the fact that $\s{N}\s{C}^{-1}\left|\ca{A}\right| \ge \s{N}'\tau^{-1}$ ($\s{N}\s{C}^{-1}$ is the average cluster size; $\s{N}'/|\ca{A}|$ is the average cluster size among the $\s{N}'$ users and $\tau^{-1}$ is the ratio of the sizes of smallest cluster and largest cluster).
Similarly, we will also have
\begin{align*}
    &\max_i \|\widehat{\fl{V}}(\fl{V}_{\s{sub}}^{\s{T}}\fl{V}_{\s{sub}})^{-1/2}\fl{V}_{\s{sub}}^{\s{T}}\fl{e}_i\| \le \max_i \|(\fl{V}_{\s{sub}}^{\s{T}}\fl{V}_{\s{sub}})^{-1/2}\fl{V}_{\s{sub}}^{\s{T}}\fl{e}_i\| \\
    &\le \frac{\|\fl{V}_{\s{sub}}\|_{2,\infty}}{\sqrt{\lambda_{\min}(\fl{V}_{\s{sub}}^{\s{T}}\fl{V}_{\s{sub}})}} \le \frac{\lr{\fl{V}}_{2,\infty}}{\sqrt{\lambda_{\min}(\fl{V}_{\s{sub}}^{\s{T}}\fl{V}_{\s{sub}})}} \le \sqrt{\frac{\mu \s{C}}{\alpha \s{M}'}}
\end{align*}
where the last line follows from the fact that $\min_{\fl{x}\in \bb{R}^{\s{C}}} \fl{x}^{\s{T}}\fl{V}_{\s{sub}}^{\s{T}}\fl{V}_{\s{sub}}\fl{x} = \min_{\fl{x}\in \bb{R}^{\s{C}}} \sum_{i \in \ca{S}} \fl{x}^{\s{T}}\fl{V}_{i}^{\s{T}}\fl{V}_{i}\fl{x}$ where $\ca{S},|\ca{S}|=\s{M}'$ is the set of rows in $\fl{V}$ represented in $\fl{V}_{\s{sub}}$. 
Here, again, we use Assumption \ref{assum:matrix2} to conclude directly that $\lambda_{\min}(\fl{V}_{\s{sub}}^{\s{T}}\fl{V}_{\s{sub}}) \ge \alpha \s{M}'/\s{M}$.

\begin{lemma}\label{lem:submatrix_complete}
Let us fix $\Delta_{\ell+1}>0$ and condition on the event $\ca{E}^{(\ell)}$. Suppose Assumptions \ref{assum:matrix2} and \ref{assum:cluster_ratio} are satisfied. In that case, in phase $\ell$, by using
 $$m_{\ell}=O\Big(\frac{\sigma^2 \s{C}^2 (\mu\beta^{-1} \bigvee \mu\alpha^{-1})^3 \log \s{M}}{\Delta_{\ell+1}^2}\Big(\s{N}\bigvee\s{MC}\Big)\log^2 (\s{MNC}\delta^{-1})\Big)\Big)$$ rounds, we can compute an estimate $\widetilde{\fl{P}}^{(\ell)}\in \bb{R}^{\s{N}\times \s{M}}$ such that with probability $1-\delta$, we have 
 \begin{align}\label{eq:infty}
 \lr{\widetilde{\fl{P}}^{(\ell)}_{\ca{M}^{(\ell,i)},\ca{N}^{(\ell,i)}}-\fl{P}_{\ca{M}^{(\ell,i)},\ca{N}^{(\ell,i)}}}_{\infty} \le \Delta_{\ell+1} \text{ for all }i\in [a_{\ell}] \text{ satisfying } \left|\ca{N}^{(i,\ell)}\right| \ge \gamma\s{C}.  
 \end{align}
 \end{lemma}
 
 \begin{proof}[Proof of Lemma \ref{lem:submatrix_complete}]
We are going to use Lemma \ref{lem:min_acc} in order to compute an estimate $\widetilde{\fl{P}}^{(\ell)}_{\ca{M}^{(\ell,i)},\ca{N}^{(\ell,i)}}$ of the sub-matrix $\fl{P}_{\ca{M}^{(\ell,i)},\ca{N}^{(\ell,i)}}$ satisfying $\lr{\widetilde{\fl{P}}^{(\ell)}_{\ca{M}^{(\ell,i)},\ca{N}^{(\ell,i)}}-\fl{P}_{\ca{M}^{(\ell,i)},\ca{N}^{(\ell,i)}}}_{\infty} \le \Delta_{\ell+1}$. From Lemma \ref{lem:min_acc}, we know that by using $m_{\ell}=O\Big(s\s{N}\log^2 (\s{MN}\delta^{-1})(\left|\ca{N}^{(\ell,i)}\right|p+\sqrt{\left|\ca{N}^{(\ell,i)}\right|p\log \s{N}\delta^{-1}})\Big)$ rounds (see Lemma \ref{lem:min_acc}) restricted to users in $\ca{M}^{(\ell,i)}$ such that with probability at least $1-\delta$,
\begin{align*}
    \lr{\widetilde{\fl{P}}^{(\ell)}_{\ca{M}^{(\ell,i)},\ca{N}^{(\ell,i)}}-\fl{P}_{\ca{M}^{(\ell,i)},\ca{N}^{(\ell,i)}}}_{\infty} \le O\Big(\frac{\sigma r }{\sqrt{sd_2}}\sqrt{\frac{\widetilde{\mu}^3\log d_2}{p}}\Big).
\end{align*}
where $d_2=\min(|\ca{M}^{(\ell,i)}|,|\ca{N}^{(\ell,i)}|)$ and $\widetilde{\mu}$ is the incoherence factor of the matrix $\fl{P}_{\ca{M}^{(\ell,i)},\ca{N}^{(\ell,i)}}$. In order for the right hand side to be less than $\Delta_{\ell+1}$, we can set $sp=O\Big(\frac{\sigma^2 r^2 \widetilde{\mu}^3 \log d_2}{\Delta_{\ell+1}^2d_2}\Big)$. Since the event $\ca{E}^{(\ell)}$ is true, we must have that $\left|\ca{M}^{(\ell,i)}\right|\ge \s{N}/\s{C}$; hence $d_2 \ge \min\Big(\frac{\s{N}}{\s{C}},\left|\ca{N}^{(\ell,i)}\right|\Big)$. Therefore, we must have that 
\begin{align*}
    m_{\ell}=O\Big(\frac{\sigma^2 \s{C}^2 \widetilde{\mu}^3 \log \s{M}}{\Delta_{\ell+1}^2}\max\Big(\s{N},\s{MC}\Big)\log^2 (\s{MNC}\delta^{-1})\Big)\Big)
\end{align*}
where we take a union bound over all sets comprising the partition of the users $[\s{N}]$ (at most $\s{C}$ of them). Finally, from Lemma \ref{lem:incoherence2}, we know that $\widetilde{\mu}$ can be bounded from above by $\max(\mu /\beta,\mu /\alpha)$ which we can use to say that 
\begin{align*}
    m_{\ell}=O\Big(\frac{\sigma^2 \s{C}^2 (\mu\beta^{-1} \bigvee \mu\alpha^{-1})^3 \log \s{M}}{\Delta_{\ell+1}^2}\Big(\s{N}\bigvee\s{MC}\Big)\log^2 (\s{MNC}\delta^{-1})\Big)\Big)
\end{align*}
to complete the proof of the lemma.
\end{proof}

 
 
 
 

\end{proof}

\subsection{Joint Arm Elimination:} 
The first part of the algorithm is run in phases indexed by $\ell=1,2,\dots$ in a similar way as for the $\s{CS}$ problem. As before, in the beginning of each phase $\ell$, our goal is to maintain properties $(1-4)$ proposed at the beginning of Section \ref{app:detailed_lbm}. Similarly, 
for a fixed $\epsilon_{\ell}$ that will be determined later, we will say that our phased elimination algorithm is $\epsilon_{\ell}-$\textit{good} at the beginning of the $\ell^{\s{th}}$ phase if the randomized algorithm LATTICE (for $\s{RCS}$) can maintain a list of users and arms satisfying the properties ($1-4$) at the start of phase $\ell$. Let us also define the event $\ca{E}^{(\ell)}$ to be true if properties $(1-4)$ are satisfied at the beginning of phase $\ell$ by LATTICE for $\s{RCS}$.

Let us fix $\Delta_{\ell+1}>0$ and condition on the event $\ca{E}^{(\ell)}$ being true at the beginning of phase $\ell$. Suppose Assumptions \ref{assum:matrix2} and \ref{assum:cluster_ratio} are satisfied. In that case, in phase $\ell$, recall from Lemma \ref{lem:submatrix_complete} that by using
 $$m_{\ell}=O\Big(\frac{\sigma^2 \s{C}^2 (\mu\beta^{-1} \bigvee \mu\alpha^{-1})^3 \log \s{M}}{\Delta_{\ell+1}^2}\Big(\s{N}\bigvee\s{MC}\Big)\log^2 (\s{ABCT})\Big)\Big)$$ rounds, we can compute an estimate $\widetilde{\fl{P}}^{(\ell)}\in \bb{R}^{\s{N}\times \s{M}}$ such that with probability $1-\s{T}^{-4}$, we have 
 \begin{align*}
 \lr{\widetilde{\fl{P}}^{(\ell)}_{\ca{M}^{(\ell,i)},\ca{N}^{(\ell,i)}}-\fl{P}_{\ca{M}^{(\ell,i)},\ca{N}^{(\ell,i)}}}_{\infty} \le \Delta_{\ell+1} \text{ for all }i\in [a_{\ell}] \text{ satisfying } \left|\ca{N}^{(i,\ell)}\right| \ge \gamma\s{C}.  
 \end{align*}
Furthermore, we denoted the aforementioned event in phase $\ell$ by $\ca{E}_2^{(\ell)}$.
Recall that in the $\ell^{\s{th}}$ phase, we have $\ca{T}^{(\ell)}_u \equiv \{j \in \ca{N}^{(\ell,i)} \mid \max_{j'\in \ca{N}^{(\ell,i)}}\widetilde{\fl{P}}^{(\ell)}_{uj'}-\widetilde{\fl{P}}^{(\ell)}_{uj} \le 2\Delta_{\ell+1}\}$. For every $i\in [a_{\ell}]$ such that $\left|\ca{N}^{(\ell,i)}\right|\ge \gamma\s{C}$, we consider a graph $\ca{G}^{(\ell,i)}$ whose nodes are given by the users in $\ca{M}^{(\ell,i)}$, an edge is drawn between two users $u,v\in \ca{M}^{(\ell,i)}$ if $\ca{T}_u^{(\ell)}\cap \ca{T}_v^{(\ell)} \neq \Phi$ and $\max_{x\in \ca{N}^{(\ell,i)}} \left|\widetilde{\fl{P}}^{(\ell)}_{ux}-\widetilde{\fl{P}}^{(\ell)}_{vx}\right|\le 3\Delta_{\ell+1}$.
Note that Lemma \ref{lem:interesting2} holds true implying that conditioned on the event $\ca{E}_2^{(\ell)}$, $\pi_u(1)\in \ca{T}^{(\ell)}_u$ for all users $u\in [\s{N}]$. Next, we show the following:

\begin{lemma}\label{lem:interesting_clique_general1}
Let $\Delta_{\ell+1} \ge \nu$. Fix any $i\in [a_{\ell}]$. Conditioned on the events $\ca{E}^{(\ell)},\ca{E}_2^{(\ell)}$, nodes in $\ca{M}^{(\ell,i)}$ corresponding to the same cluster form a clique.
\end{lemma}

\begin{proof}
For any two users $u,v\in \ca{M}^{(\ell,i)}$ belonging to the same cluster, consider arms $x\in \ca{T}^{(\ell)}_u,y\in \ca{T}^{(\ell)}_v$; in that case, we have
\begin{align*}
    \fl{\widetilde{P}}^{(\ell)}_{ux}-\fl{\widetilde{P}}^{(\ell)}_{vx} = \fl{\widetilde{P}}^{(\ell)}_{ux}-\fl{P}_{ux}+\fl{P}_{ux}-\fl{P}_{vx}+\fl{P}_{vx}-\fl{\widetilde{P}}^{(\ell)}_{vx}
    \le 2\Delta_{\ell+1}+\nu \le 3\Delta_{\ell+1}.
\end{align*}

 it is clear that the nodes in $\ca{M}^{(\ell,i)}$ corresponding to the same cluster form a clique.
\end{proof}

\begin{coro}
Condition on the events $\ca{E}^{(\ell)},\ca{E}_2^{(\ell)}$ and suppose $\Delta_{\ell+1} \ge \nu$. 1) Each connected component of the graph $\ca{G}^{(\ell,i)}$ can be represented as a nice set of users. 2) Consider two users $u,v\in \ca{M}^{(\ell,i)}$ having an edge in the graph $\ca{G}^{(\ell,i)}$. We must have $\max_{x\in \ca{T}^{(\ell)}_u,y\in \ca{T}^{(\ell)}_v}\left|\fl{P}_{ux}-\fl{P}_{vy}\right|\le 20\Delta_{\ell+1}$. 3) Consider two users $u,v\in \ca{M}^{(\ell,i)}$ having a path in the graph $\ca{G}^{(\ell,i)}$. Then, we must have $\max_{x\in \ca{T}^{(\ell)}_u,y\in \ca{T}^{(\ell)}_v}\left|\fl{P}_{ux}-\fl{P}_{vy}\right|\le 40\s{C}\Delta_{\ell+1}$. 
\end{coro}

\begin{proof}
\begin{enumerate}
    \item This follows from the statement that users in the same set belonging to the same cluster form a clique (see Lemma \ref{lem:interesting_clique_general1}). 
    \item From the construction of $\ca{G}^{(\ell,i)}$, we know that
users $u,v\in \ca{M}^{(\ell,i)}$ have an edge if $\left|\widetilde{\fl{P}}^{(\ell)}_{ux}-\widetilde{\fl{P}}^{(\ell)}_{vx}\right|\le 3\Delta_{\ell+1}$ (implying that $\left|\fl{P}_{ux}-\fl{P}_{vx}\right|\le 6\Delta_{\ell+1}$) for all $x \in \ca{N}^{(\ell,i)}$ and $\ca{T}_u\cap \ca{T}_v \neq \Phi$. Suppose $z\in \ca{T}_u\cap \ca{T}_v$. Now, for any pair of arms 
$x\in \ca{T}^{(\ell)}_u,y\in \ca{T}^{(\ell)}_v$; in that case, we have
\begin{align*}
    \fl{P}_{ux}-\fl{P}_{uy} = \fl{P}_{ux} - \fl{P}_{uz}+\fl{P}_{uz}-\fl{P}_{vz}+\fl{P}_{vz}-\fl{P}_{vy}+\fl{P}_{vy}-\fl{P}_{uy} \le 20\Delta_{\ell+1}
\end{align*}

\item The proof follows in exactly the same way as in Lemma \ref{lem:interesting5} with minor modification in constants.
\end{enumerate}

\end{proof}

\begin{lemma}\label{lem:interesting_clique_general2}
Let $\Delta_{\ell+1} \le 2\nu$. Conditioned on the events $\ca{E}^{(\ell)}$ and $\ca{E}_2^{(\ell)}$, for any two users $u,v$ belonging to two different clusters, we will have 
\begin{align*}
    \left|\widetilde{\fl{P}}^{(\ell)}_{ux}-\widetilde{\fl{P}}^{(\ell)}_{vy}\right| \ge 8\Delta_{\ell+1} \text{ or }\left|\widetilde{\fl{P}}^{(\ell)}_{ux}-\widetilde{\fl{P}}^{(\ell)}_{vy}\right| \ge 8\Delta_{\ell+1}.
\end{align*}
\end{lemma}

\begin{proof}
For any two users $u,v\in \ca{M}^{(\ell,i)}$ belonging to the same cluster, consider arms $\pi_u(1)\in \ca{T}^{(\ell)}_u,\pi_v(1)\in \ca{T}^{(\ell)}_v$ (we have proved in Lemma \ref{lem:interesting2} that $\pi_u(1)\in \ca{T}^{(\ell)}_u,\pi_v(1)\in \ca{T}^{(\ell)}_v$ conditioned on the event $\ca{E}_2^{(\ell)}$); in that case, recall that we have either $\left|\fl{P}_{u\pi_u(1)}-\fl{P}_{v\pi_u(1)}\right| \ge 20\nu$ or $\left|\fl{P}_{u\pi_v(1)}-\fl{P}_{v\pi_v(1)}\right| \ge 20\nu$ (without loss of generality suppose the former is true). We will have

\begin{align*}
    \fl{\widetilde{P}}^{(\ell)}_{u\pi_u(1)}-\fl{\widetilde{P}}^{(\ell)}_{v\pi_u(1)} = \fl{\widetilde{P}}^{(\ell)}_{u\pi_u(1)}-\fl{P}_{u\pi_u(1)}+\fl{P}_{u\pi_u(1)}-\fl{P}_{v\pi_u(1)}+\fl{P}_{v\pi_u(1)}-\fl{\widetilde{P}}^{(\ell)}_{v\pi_u(1)}
    \ge 20\nu-2\Delta_{\ell+1} \ge 16\nu \ge 8\Delta_{\ell+1}.
\end{align*}
 it is clear that the nodes in $\ca{M}^{(\ell,i)}$ corresponding to the same cluster form a clique.
\end{proof}

\begin{lemma}\label{lem:interesting_clique_general3}
Let $\nu \le \Delta_{\ell+1} \le 2\nu$. Conditioned on the events $\ca{E}^{(\ell)}$ and $\ca{E}_2^{(\ell)}$, the union of the graphs $\ca{G}^{(\ell,i)}$ can be represented as $\s{C}$ connected components if $\ca{B}^{(\ell)}=\Phi$. 
\end{lemma}

\begin{proof}
The proof follows from Lemma \ref{lem:interesting_clique_general1} and Lemma \ref{lem:interesting_clique_general2}.
\end{proof}

For the subsequent iteration indexed by $\ell+1$, we compute the updated groups of users $\ca{M}^{(\ell+1)}$ in the following way: each set corresponds to the connected components of the graphs $\{\ca{G}^{(\ell,i)}\}$ for those indices $i\in [a_{\ell}]$ where $\left|\ca{N}^{(\ell,i)}\right|\ge \gamma\s{C}$ plus the groups of users $\ca{M}^{(\ell,i)}$ where $\left|\ca{N}^{(\ell,i)}\right|\le \s{C}$. More precisely, let $\ca{T}\subseteq [a_{\ell}]$ be the subset of indices for which 
$\left|\ca{N}^{(\ell,i)}\right|\ge \gamma\s{C}$; $\{\ca{G}^{(\ell,i,j)}\}$ be the connected components of the graph $\ca{G}^{(\ell,i)}$ for $i\in \ca{T}$. In that case, $\ca{M}^{(\ell+1)}=\{\ca{G}^{(\ell,i,j)} \mid i \in \ca{T}, \ca{G}^{(\ell,i,j)} \text{ is a connected component of graph }\ca{G}^{(\ell,i)}\}+\{\ca{M}^{(\ell,i)}\mid i \in [a_{\ell}]\setminus\ca{T}\}$.
Similarly, we update the family of sets of active arms as follows: for each connected component $\ca{M}^{(\ell+1,s)}=\ca{G}^{(\ell,i,j)}$ of some graph, we define the active set of arms $\ca{N}^{(\ell+1,s)}$ to be $\cup_{u\in \ca{G}^{(\ell,i,j)}}\ca{T}_u^{(\ell)}$ and for each group $\{\ca{M}^{(\ell,i)}\}_{i \in [a_{\ell}]\setminus\ca{T}}$, we keep the corresponding set of active  arms $\{\ca{N}^{(\ell,i)}\}_{i \in [a_{\ell}]\setminus\ca{T}}$ same. With $a_{\ell+1}=\left|\ca{M}^{(\ell+1)}\right|$, we will also update $\ca{B}^{(\ell+1)}=\bigcup_{i \in [a_{\ell+1}] \mid \left| \ca{N}^{(\ell+1,i)}\right| \ge \gamma\s{C}} \ca{M}^{(\ell+1,i)}$ to be the set of users with more than $\s{C}$ active arms.

\begin{lemma}\label{lem:interest_gen}
Condition on the events $\ca{E}^{(\ell)}$ being true. In that case, with probability $1-\s{T}^{-4}$, with the groups of users $\ca{M}^{(\ell+1)}$ and their respective group of arms given by $\ca{N}^{(\ell+1)}$ being updated as described above and $\Delta_{\ell+1}=\epsilon_{\ell}/80\s{C}$, the event $\ca{E}^{(\ell+1)}$ is also going to be true with $\epsilon_{\ell+1} \le \epsilon_{\ell}/2$.
\end{lemma}

\begin{proof}
Conditioned on the event $\ca{E}^{(\ell)}$ being true, the event $\ca{E}^{(\ell)}_2$ holds true with probability with $1-\s{T}^{-4}$ (by substituting $\delta=\s{T}^{-4}$ in Lemma \ref{lem:interesting1}). Now, conditioned on the event $\ca{E}^{(\ell)},\ca{E}_2^{(\ell)}$ being true, the properties $(1-4)$ hold true at the beginning of the $(\ell+1)^{\s{th}}$ phase as well with our construction of $\ca{M}^{(\ell+1)},\ca{N}^{(\ell+1)}$. 
For the $(\ell+1)^{\s{th}}$ phase
from Lemma \ref{lem:interesting5}, we know that for any pair of users $u,v$ in the same cluster $\ca{M}^{(\ell+1,i)}$ in the updated set of clusters $\ca{M}^{(\ell+1)}$, we must have $\max_{x\in \ca{T}^{(\ell)}_u,y\in \ca{T}^{(\ell)}_v}\left|\fl{P}_{ux}-\fl{P}_{vy}\right|\le 40\s{C}\Delta_{\ell+1}$. From Lemma \ref{lem:interesting2} we know that $\s{argmax}_j \fl{P}_{uj}\in \ca{T}_u^{(\ell)} \subseteq \ca{N}^{(\ell,i)}$ where $\ca{N}^{(\ell,i)}$ is the active set of arms for users in $\ca{M}^{(\ell,i)}$ in the updated set $\ca{M}$. For $\ell > 1$, we will set $\Delta_{\ell+1}=\epsilon_{\ell}/80\s{C}$ which would give us that $\epsilon_{\ell+1}= \epsilon_{\ell}/2$. Finally, also note that we maintain the set of users $\ca{B}^{(\ell+1)}$ as stipulated in Property 4 for the beginning of the $(\ell+1)^{\s{th}}$ phase.
\end{proof}

We stop the joint phased elimination algorithm at the end of the phase $\ell^{\star}$ when we have $\s{C}$ connected components. From Lemma \ref{lem:interesting_clique_general3}, we know that users in the same connected component must belong to the same cluster.  We then move on to the second part where we run cluster-wise phased elimination:

\subsection{Cluster-wise Arm Elimination:} Suppose the total number of phases that the joint phased elimination algorithm was ran for is $\ell^{\star}$ ($\nu \le \Delta_{\ell^{\star}+1} \le 2\nu$). We condition on the event  that the joint phased elimination algorithm is $\epsilon_{\ell}$-good for all $\ell\in [\ell^{\star}]$ that it ran for (in other words the event $\ca{E}^{(\ell^{\star})}$ is true); note that the failure probability for this event is at most $\s{T}^{-3}$ (see Lemma \ref{lem:interest_gen}). At the end of the joint phased elimination algorithm, we will have $\s{C}$ connected components denoted by $\s{N}^{(\ell^{\star})}\equiv \{\ca{M}^{(\ell^{\star},1)},\ca{M}^{(\ell^{\star},2)},\dots,\ca{M}^{(\ell^{\star},\s{C})}\}$  and arms $\ca{N}^{\ell^{\star}}\equiv \{\ca{N}^{(\ell^{\star},1)},\ca{N}^{(\ell^{\star},2)},\dots,\ca{N}^{(\ell^{\star},\s{C})}\}$ such that $\ca{N}^{(\ell^{\star},i)} \supseteq \{\s{argmax}_{j} \fl{P}_{uj} \mid u \in \ca{M}^{(\ell^{\star},i)}\}$ i.e. for each user $u$ in the set $\ca{M}^{(\ell^{\star},i)}$, their best arm must belong to the set $\ca{N}^{(\ell^{\star},i)}$. Furthermore, the set $\ca{N}^{(\ell^{\star},i)}$ must also satisfy the following:
\begin{align}\label{eq:phased_elim_gap}
    \left|\max_{j\in \ca{N}^{(\ell^{\star},i)}}\fl{P}_{uj}-\min_{j\in \ca{N}^{(\ell^{\star},i)}}\fl{P}_{uj}\right| \le \epsilon_{\ell^{\star}+1} \text{ for all }u\in \ca{M}^{(\ell^{\star},i)}
\end{align}
This part of the algorithm is again run in phases but this time, we do not cluster anymore since the clustering operation is complete in the joint phased elimination. Hence, we can continue to define the events $\ca{E}^{(\ell')},\ca{E}_2^{(\ell')}$ for all phases $\ell'$ for which the Cluster-wise phased elimination algorithm is run indexed by $\ell'=\ell^{\star}+1,\ell^{\star}+2,\dots$ for continuity. Let us condition on the event $\ca{E}^{(\ell')}$ which means that at the beginning of the $\ell'^{\s{th}}$ phase, our algorithm is $\epsilon_{\ell'}$-good. We apply Lemma \ref{lem:interesting1} at the beginning of the $\ell'^{\s{th}}$ phase such that with probability at least $1-\s{T}^{-4}$, we have $\lr{\widetilde{\fl{P}}^{(\ell')}_{\ca{M}^{(\ell',i)},\ca{N}^{(\ell',i)}}-\fl{P}_{\ca{M}^{(\ell',i)},\ca{N}^{(\ell',i)}}}_{\infty} \le \Delta_{\ell'+1} \text{ for all }i\in [\s{C}]$ by using ($\mu^{\star}=\max(\mu\alpha^{-1},\mu\beta^{-1})$)
\begin{align*}
    O\Big(\frac{\sigma^2 \s{C}^2 (\mu^{\star})^3 \log \s{M}}{\Delta_{\ell'+1}^2}\Big(\s{N}\bigvee\s{MC}\Big)\log^2 (\s{ABCT})\Big)\Big)
\end{align*}
rounds.  Note that Lemma \ref{lem:interesting2} still holds true for the $\ell'^{\s{th}}$ phase. For each $i\in [\s{C}]$, we only update $\ca{N}^{(\ell,i)}=\cap_{u \in \ca{M}^{(\ell,i)}} \ca{T}_u$. We again choose $\Delta_{\ell'+1}=\epsilon_{\ell'}/32\s{C}$ implying that
 $\epsilon_{\ell'+1}= \epsilon_{\ell'}/2$. Since Lemma \ref{lem:interesting2} holds true, $\s{argmax}_j \fl{P}_{uj}$ will belong to the updated $\ca{N}^{(\ell,i)}$ for every $u\in \ca{M}^{(\ell,i)}$ and furthermore, $$\left|\max_{j\in \ca{N}^{(\ell,i)}}\fl{P}_{uj}-\min_{j\in \ca{N}^{(\ell,i)}}\fl{P}_{uj}\right| \le 4\Delta_{\ell'+1}\le \epsilon_{\ell'+1} \text{ for all }u\in \ca{M}^{(\ell,i)}$$. 
 
Now, we are ready to prove the main theorem 

\begin{proof}[Proof of Theorem \ref{thm:main_GLBM}]
We condition on the events $\ca{E}^{(\ell)},\ca{E}_2^{(\ell)}$ being true for all $\ell$ (including the joint phased elimination and the cluster-wise phased elimination algorithm). The probability that there exists any $\ell$ such that the events $\ca{E}^{(\ell)},\ca{E}_2^{(\ell)}$ is false is $O(\s{T}^{-4})$ (by setting $\delta=\s{T}^{-4}$ in the proof of Lemma \ref{lem:interesting1}); hence the probability that $\ca{E}^{(\ell)},\ca{E}_2^{(\ell)}$ is true for all $\ell$ is at least $1-O(\s{T}^{-3})$ (the total number of iterations can be at most $\s{T}$).  
Let us also denote the set of rounds in phase $\ell$ by $\ca{T}_{\ell}\subseteq [\s{T}]$ (therefore $\left|\ca{T}_{\ell}\right|=m_{\ell}$). Let us compute the  regret $\sum_{t\in \ca{T}^{(\ell)}}\fl{P}_{u(t)\pi_{u(t)}(1)}- \sum_{t\in \ca{T}^{(\ell)}}\fl{P}_{u(t),\rho(t)}$ restricted to the rounds in $\ca{T}^{(\ell)}$ conditioned on the events $\ca{E}^{(\ell)},\ca{E}_2^{(\ell)}$ being true for all $\ell$. We can bound the regret quantity in the $\ell^{\s{th}}$ phase from above by $m_{\ell}\epsilon_{\ell}$. Substituting from Lemma \ref{lem:interesting1} and using the fact that $\Delta^2_{\ell+1}= \epsilon_{\ell}^2/1024 \s{C}^2$, we have that \begin{align*}
  \sum_{t\in \ca{T}^{(\ell)}}\fl{P}_{u(t)\pi_{u(t)}(1)}- \sum_{t\in \ca{T}^{(\ell)}}\fl{P}_{u(t),\rho(t)} = O\Big(\frac{\sigma^2 \s{C}^2 (\mu^{\star})^3 \log \s{M}}{\Delta_{\ell+1}^2}\Big(\s{N}\bigvee\s{MC}\Big)\log^2 (\s{MNC}\delta^{-1})\Big)\Big) 
\end{align*} 
We can now bound the regret as follows (after removing the conditioning on the events $\bigcap_{\ell}\ca{E}^{(\ell)}\bigcap_{\ell}\ca{E}_2^{(\ell)}$):
\begin{align*}
    &\bb{E}\Big(\sum_{t\in [\s{T}]}\fl{P}_{u(t)\pi_{u(t)}(1)}- \sum_{t\in [\s{T}]}\fl{P}_{u(t),\rho(t)}\Big)
    =\bb{E}\Big(\sum_{\ell}\Big(\sum_{t\in \ca{T}^{(\ell)}}\fl{P}_{u(t)\pi_{u(t)}(1)}- \sum_{t\in \ca{T}^{(\ell)}}\fl{P}_{u(t),\rho(t)}\Big)\Big) \\
    &= \sum_{\ell}  O\Big(\epsilon_{\ell}m_{\ell}\mid \bigcap_{\ell}\ca{E}^{(\ell)}\bigcap_{\ell}\ca{E}_2^{(\ell)}\Big)+O(\s{T}^{-3}\|\fl{P}\|_{\infty})+\sum_{\ell}\sum_{u\in\s{N}\setminus [\ca{B}^{(\ell)}]} \s{Reg}_{\s{UCB}}(u,\ca{T}_u^{(\ell)}) 
\end{align*}
At this point, the analysis is similar to the proof of Theorem \ref{thm:main_LBM}. The contribution of the regret from the final term (analysis of the UCB algorithm for each user with at most $\gamma \s{C}$ arms) is a strictly lower order term as demonstrated in the proof of Theorem \ref{thm:main_GLBM} and we ignore it below. Moving on, we can decompose the regret as follows: 
\begin{align*}
     &\s{Reg}(\s{T})\le O(\s{T}^{-3}\|\fl{P}\|_{\infty})+O\Big(\sum_{\ell: \epsilon_{\ell}\le\Phi} \epsilon_{\ell}m_{\ell}\mid \ca{E}^{(\ell)},\ca{E}^{(\ell)}_2 \text{ is true for all } \ell\Big)
     +O\Big(\sum_{\ell: \Delta_{\ell}>\Phi} \epsilon_{\ell}\s{V}\Delta_{\ell+1}^{-2}\mid \ca{E}^{(\ell)},\ca{E}^{(\ell)}_2 \text{ is true for all } \ell\Big) \\
     &\le O(\s{T}^{-3}\|\fl{P}\|_{\infty}) +\s{T}\Phi+O\Big(\sum_{\ell: \epsilon_{\ell}>\Phi} \s{C}^2\s{V}\epsilon_{\ell}^{-1}\Big) \\
\end{align*}

Since we chose $\epsilon_{\ell}=C'2^{-\ell}\min\Big(\|\fl{P}\|_{\infty},\frac{\sigma\sqrt{\mu}}{\log \s{N}}\Big)$ for some constant $C'>0$, the maximum number of phases $\ell$ for which $\epsilon_{\ell}>\Phi$ can be bounded from above by $\s{J}=O\Big(\log \Big(\frac{1}{\Phi}\min\Big(\|\fl{P}\|_{\infty},\frac{\sigma\sqrt{\mu}}{\log \s{N}}\Big)\Big)\Big)$. Hence, with $\s{V}=\sigma^2 \s{C}^2 (\mu^{\star})^3 \log \s{M}\max\Big(\s{N},\s{MC}\Big)\log^2 (\s{ABCT})$, we have
\begin{align*}
    \s{Reg}(\s{T}) &\le O(\s{T}^{-3}\|\fl{P}\|_{\infty})+O(\s{T}\Phi)+O\Big(\s{JV}\s{C}^2\Phi^{-1}\Big) \\
     &= O(\s{T}^{-3}\|\fl{P}\|_{\infty})+O(\s{CJ}\sqrt{\s{TV}})
\end{align*}
where we substituted $\Phi=\sqrt{\s{V}\s{C}^2\s{T}^{-1}}$ and hence $\s{J}=O\Big(\log \Big(\frac{1}{\sqrt{\s{VT}^{-1}}}\min\Big(\|\fl{P}\|_{\infty},\frac{\sigma\sqrt{\mu}}{\log \s{N}}\Big)\Big)\Big)$ in the final step. 

\end{proof}

\section{Proof of a general version of Lemma \ref{lem:min_acc}}\label{app:general_proof}

For the sake of completeness, we reproduce the proof of a more general version of Lemma \ref{lem:min_acc}  here (Note that Lemma \ref{lem:min_acc} has been explicitly proved in \cite{jain2022online}).
We start with the following corollary:

\begin{lemma}[Theorem 2 in \cite{chen2019noisy}]\label{lem:source2}
 Let $\fl{P}=\fl{\bar{U}}\f{\Sigma}\fl{\bar{V}}^{\s{T}}\in \bb{R}^{d\times d}$ such that $\fl{\bar{U}}\in \bb{R}^{d\times r},\fl{\bar{V}}\in \bb{R}^{d \times r}$ and $\f{\Sigma} \triangleq \s{diag}(\lambda_1,\lambda_2,\dots,\lambda_r) \in \bb{R}^{r \times r}$ with $\fl{\bar{U}}^{\s{T}}\fl{\bar{U}}=\fl{\bar{V}}^{\s{T}}\fl{\bar{V}}=\fl{I}$ and $\|\fl{\bar{U}}\|_{2,\infty}\le \sqrt{\mu r/d}, \|\fl{\bar{V}}\|_{2,\infty}\le \sqrt{\mu r/d}$.  
 Let $1\ge p \ge C \kappa^4 \mu^2 d^{-1} \log^3 d$ for some sufficiently large constant $C>0$, $\sigma = O\Big(\sqrt{\frac{p}{d\kappa^4\mu r\log d}}\min_i\lambda_i\Big)$, rank $r$ and condition number $\kappa \triangleq \frac{\max_i \lambda_i}{\min_i \lambda_i}$. Then, with probability exceeding $ 1-O(d^{-3})$, we can recover a 
  matrix $\widehat{\fl{P}}$ s.t.,  
  \begin{align}\label{eq:guarantee_source}
    \|\widehat{\fl{P}}-\fl{P}\|_{\infty} \le O\Big(\frac{\sigma}{\min_i\lambda_{i}}\cdot \sqrt{\frac{\kappa^3 \mu r d\log d}{p}} \|\fl{P}\|_{\infty}\Big).
  \end{align}
\end{lemma}

\begin{lemma}\label{thm:randomsample2}
 Let $\fl{P}=\fl{\bar{U}}\f{\Sigma}\fl{\bar{V}}^{\s{T}}\in \bb{R}^{\s{N} \times \s{M}}$ such that $\fl{\bar{U}}\in \bb{R}^{\s{N}\times r},\fl{\bar{V}}\in \bb{R}^{\s{M} \times r}$ and $\f{\Sigma} \triangleq \s{diag}(\lambda_1,\lambda_2,\dots,\lambda_r) \in \bb{R}^{r \times r}$ with $\fl{\bar{U}}^{\s{T}}\fl{\bar{U}}=\fl{\bar{V}}^{\s{T}}\fl{\bar{V}}=\fl{I}$ and $\|\fl{\bar{U}}\|_{2,\infty}\le \sqrt{\mu r/\s{N}}, \|\fl{\bar{V}}\|_{2,\infty}\le \sqrt{\mu r/\s{M}}$. Let $d_1=\max(\s{N},\s{M})$ and $d_2=\min(\s{N},\s{M})$.   
 Let $1 \ge p \ge C  \kappa^4\mu^2 d_1d_2^{-2} \log^3 d_1$ for some sufficiently large constant $C>0$, $\sigma  = O\Big(\sqrt{\frac{p}{d\kappa^4\mu r\log d}}\min_i\lambda_i\Big)$, rank $r=O(1)$ and condition number $\kappa \triangleq \frac{\max_i \sigma_i}{\min_i \sigma_i} = O(1)$. Then, with probability exceeding $ 1-O(d_1^{-3})$, we can recover a 
  matrix $\widehat{\fl{P}}$ s.t.,  
  \begin{align}\label{eq:guarantee}
  \|\widehat{\fl{P}}-\fl{P}\|_{\infty} =  O\Big(\frac{\sigma }{\sqrt{d_2}}\Big(\frac{d_1}{d_2}\Big)^{1/2}\sqrt{\frac{\kappa^5\mu^3r^3 \log d_1}{p}}\Big).
  \end{align}

\end{lemma}

\begin{proof}[Proof of Lemma \ref{thm:randomsample2}]

Without loss of generality, let us assume that the matrix $\fl{P}$ is tall i.e. $\s{N} \ge \s{M}$. Now, 
let us construct the matrix  
\begin{align*}
    \fl{Q} =  
    \begin{bmatrix}
    \fl{P} & \fl{0}_{\s{N}\times \s{B-A}} 
    \end{bmatrix}
    =  \bar{\fl{U}} \Sigma [\bar{\fl{V}}^{\s{T}} \; \f{0}_{\s{A-B}}^{\s{T}}]
\end{align*}
where $\fl{Q}\in \bb{R}^{\s{N}\times \s{N}}$. 
Clearly, the decomposition $\fl{Q}=\bar{\fl{U}} \Sigma [\bar{\fl{V}}^{\s{T}} \; \f{0}_{\s{B-A}}^{\s{T}}]$ also coincides with the SVD of $\fl{Q}$ since both  $\bar{\fl{U}}$ and $[\bar{\fl{V}}^{\s{T}} \; \f{0}_{\s{A-B}}^{\s{T}}]^{\s{T}}$ are orthonormal matrices while $\Sigma$ remains unchanged. In case when $\s{M}> \s{N}$, we can construct $\fl{Q}$ similarly by vertically stacking $\fl{P}$ with a zero matrix of dimensions $\s{A-B}\times \s{N}$. Hence, generally speaking, let us denote $d_1=\max(\s{N},\s{M})$ and $d_2=\min(\s{N},\s{M})$.

The matrix $\fl{Q}$ is $\bar{\mu}$-incoherent where $\bar{\mu}r(d_1)^{-1}=\mu r d_2^{-1}$ implying that $\bar{\mu}=\mu d_1/d_2$. Moreover, we also have $\|\fl{Q}\|_{\infty}=\|\fl{P}\|_{\infty}$ implying that $\max_{ij}\left|\fl{P}_{ij}\right|=\max_{ij}\left|\fl{Q}_{ij}\right|$.
Therefore, by invoking Lemma \ref{lem:source2},
the sample size must obey 
\begin{align*}
    &p \ge C  \kappa^4\mu^2 d_1d_2^{-2} \log^3 d_1 \quad \text{and} \quad \sigma = O\Big(\sqrt{\frac{p}{d\kappa^4\mu r\log d}}\min_i\lambda_i\Big) 
\end{align*}
Then with probability at least $O(d_1^{-3})$, we can recover a  matrix $\widehat{\fl{Q}}$ such that 
\begin{align*}
    \|\widehat{\fl{Q}}-\fl{Q}\|_{\infty} \le
    O\Big(\frac{\sigma}{\min_i\lambda_{i}}\cdot \sqrt{\frac{\kappa^3 \mu \frac{d_1}{d_2} r d_1\log d_1}{p}} \|\fl{P}\|_{\infty}\Big).
\end{align*}
Using the fact that $\|\fl{P}\|_{\infty} \le \max_i \lambda_i  \|\bar{\fl{U}}\|_{2,\infty}\|\bar{\fl{V}}\|_{2,\infty} = \max_i \lambda_i  \mu r/\sqrt{d_1 d_2} \implies \|\fl{P}\|_{\infty}/\min_i \lambda_i = \kappa \cdot \mu r/\sqrt{d_1 d_2}  $, we obtain a matrix $\widehat{\fl{P}}$ such that 
\begin{align*}
    \|\widehat{\fl{P}}-\fl{P}\|_{\infty} \le O\Big(\frac{\sigma \mu  r}{\sqrt{d_1d_2}}\Big(\frac{d_1}{d_2}\Big)^{1/2}\sqrt{\frac{\kappa^5 \mu r d_1 \log d_1}{p}}\Big) = O\Big(\frac{\sigma }{\sqrt{d_2}}\Big(\frac{d_1}{d_2}\Big)^{1/2}\sqrt{\frac{\kappa^5\mu^3r^3 \log d_1}{p}}\Big).
\end{align*}

\end{proof}

\begin{lemma}\label{thm:randomsample3}
 Let the matrix $\fl{P}\in \bb{R}^{\s{N}\times \s{M}}$ satisfy the conditions as stated in Lemma \ref{thm:randomsample2}. Let $d_1=\max(\s{N},\s{M})$, $d_2 = \min(\s{N},\s{M})$, $1 \ge p \ge C  \kappa^4\mu^2 d_1d_2^{-2} \log^3 d_1$ for some sufficiently large constant $C>0$, $\sigma  = O\Big(\sqrt{\frac{p}{d\kappa^4\mu r\log d}}\min_i\lambda_i\Big)$. 
 For any pair of indices $(i,j)\in [\s{N}]\times[\s{M}]$,
 \begin{align}\label{eq:rectangular}
  |\widehat{\fl{P}}_{ij}-\fl{P}_{ij}|\le O\Big(\frac{\sigma }{\sqrt{d_2}}\Big(\frac{d_1}{d_2}\Big)^{1/2}\sqrt{\frac{\kappa^5\mu^3r^3 \log d_1}{p}}\Big)  
 \end{align}
\end{lemma}

\begin{proof}
Let us assume that the matrix $\fl{P}$ is tall i.e. $\s{N} \ge \s{M}$. Now, let us partition the set of rows into $\frac{\s{N}}{\s{M}}$ groups by assigning each group uniformly at random to each row. Notice that the expected number of rows in each group is $\s{M}$ and by using Chernoff bound, the number of rows in each group lies in the interval $[\frac{\s{M}}{2},\frac{3\s{M}}{2}]$ with probability at least $1-2\exp(-\s{M}/12)$. When $\s{N} \le \s{M}$, we partition the set of columns in a similar manner into $\s{M}/\s{N}$ groups so the number of columns in each group lies in the interval $[\frac{\s{N}}{2},\frac{3\s{N}}{2}]$ with probability at least $1-2\exp(-\s{N}/12)$.
 Hence, generally speaking, let us denote $d_1=\max(\s{N},\s{M})$ and $d_2=\min(\s{N},\s{M})$; we constructed $d_1/d_2$ sub-matrices of $\fl{P}$ denoted by $\fl{P}^{(1)},\fl{P}^{(2)},\dots,\fl{P}^{(d_1/d_2)}$.
 Let us analyze the guarantees on estimating  $\fl{P}^{(1)}$. The analysis for other matrices follow along similar lines. Note that $\fl{P}^{(1)}=\fl{U}\f{\Sigma}\fl{V}_{\s{sub}}^{\s{T}}$ where $\fl{V}_{\s{sub}}$ denotes the $\s{N}'\times r$ matrix where the rows in $\fl{V}_{\s{sub}}$ corresponds to the $\s{N}'$ rows in $\fl{V}$ assigned to $\fl{P}^{(1)}$. 
 
 First we will bound from below the minimum eigenvalue of the matrix $\fl{V}_{\s{sub}}^{\s{T}}\fl{V}_{\s{sub}}$. Note that every row of $\fl{V}$ is independently sampled  with probability $p \triangleq d_2/d_1$ for the matrix $\fl{P}^{(1)}$. Hence, we have 
 \begin{align*}
     \frac{1}{p}\fl{V}_{\s{sub}}^{\s{T}}\fl{V}_{\s{sub}} = \frac{1}{p}\sum_{i\in [d_1]} \delta_i\fl{V}_i\fl{V}_i^{\s{T}} = \sum_{i\in [d_1]} \fl{W}^{(i)}
 \end{align*}
 where $\delta_i$ denotes the indicator random variable which is true when $\fl{V}_i$ (the $i^{\s{th}}$ row of $\fl{V}$) is chosen for $\fl{P}^{(1)}$ and $\fl{W}^{(i)}=\frac{1}{p} \delta_i\fl{V}_i\fl{V}_i^{\s{T}}$. Notice that the random matrices $\fl{W}^{(i)}$ are independent with $\bb{E}\fl{W}^{(i)}=\fl{V}_i\fl{V}_i^{\s{T}}$. Hence we define $\fl{Z}^{(i)}=\fl{W}^{(i)}-\bb{E}\fl{W}^{(i)}$ satisfying $\bb{E}\fl{Z}^{(i)}=0$. Moreover, for all $i\in [d_1]$, we have $\|\fl{Z}^{(i)}\|_2\le \Big(1+\frac{1}{p}\Big)\|\fl{V}_i\fl{V}_i^{\s{T}}\|_2 \le \Big(1+\frac{1}{p}\Big) \max_i \|\fl{V}_i\|_2^2 \le \frac{2\mu r}{p d_1}$. Next, we can show the following:
 \begin{align*}
     &\|\sum_{i\in [d_1]}\fl{Z}^{(i)}(\fl{Z}^{(i)})^{\s{T}}\|_2 \le \|\Big(\frac{1}{p}-1\Big)\sum_{i\in [d_1]}(\fl{V}_i\fl{V}_i^{\s{T}})(\fl{V}_i\fl{V}_i^{\s{T}})^{\s{T}}\|_2 \le \|\Big(\frac{1}{p}-1\Big)\|\fl{V}_i\|^2\sum_{i\in [d_1]}(\fl{V}_i\fl{V}_i^{\s{T}})\|_2 \\
     &\le \frac{\mu r}{pd_1} \lambda_{\max}(\fl{V}^{\s{T}}\fl{V}).
 \end{align*}
 Similarly, we will also have 
 \begin{align*}
     \|\sum_{i\in [d_1]}(\fl{Z}^{(i)})^{\s{T}}\fl{Z}^{(i)}\|_2 \le \frac{\mu r}{pd_1} \lambda_{\max}(\fl{V}^{\s{T}}\fl{V}) 
     \le \frac{\mu r}{p d_1}
 \end{align*}
 where we used that $\fl{V}^{\s{T}}\fl{V}$ is orthogonal.
 Therefore, by using Bernstein's inequality for matrices, we have with probability at least $1-\delta$,
 \begin{align*}
     \|\frac{1}{p}\fl{V}_{\s{sub}}^{\s{T}}\fl{V}_{\s{sub}}-\fl{V}^{\s{T}}\fl{V}\| \le \frac{2\mu r}{3pd_1}\log \frac{2r}{\delta}+\sqrt{\frac{\mu r}{pd_1} \log \frac{2r}{\delta}}.
 \end{align*}
 Hence, by using Weyl's inequality, we will have with probability $1-\delta$
 \begin{align*}
     \lambda_{\min}(\fl{V}_{\s{sub}}^{\s{T}}\fl{V}_{\s{sub}}) \ge p-\frac{2\mu r}{3d_1}\log \frac{2r}{\delta}-\sqrt{\frac{p\mu r}{d_1} \log \frac{2r}{\delta}}.
 \end{align*}
 Hence, with probability at least $1-d_1^{-10}$, if $d_2=\Omega(\mu r \log rd_1)$, then $\lambda_{\min}(\fl{V}_{\s{sub}}^{\s{T}}\fl{V}_{\s{sub}}) \ge p/2$ implying that $\fl{V}_{\s{sub}}^{\s{T}}\fl{V}_{\s{sub}}$ is invertible. Also, under the same condition, note that we can show similarly that $\lambda_{\max}(\fl{V}_{\s{sub}}^{\s{T}}\fl{V}_{\s{sub}}) \le 3p/2$ implying that the condition number of each sub-matrix also stays $\kappa \cdot O(1)$ with high probability. 
 
 Clearly, $\fl{V}_{\s{sub}}$ is not orthogonal and therefore, we have
 \begin{align*}
     \fl{P}^{(1)}= \fl{U}\f{\Sigma}(\fl{V}_{\s{sub}}^{\s{T}}\fl{V}_{\s{sub}})^{1/2}(\fl{V}_{\s{sub}}^{\s{T}}\fl{V}_{\s{sub}})^{-1/2}\fl{V}_{\s{sub}}^{\s{T}} = \fl{U}\widehat{\fl{U}}\widehat{\f{\Sigma}}\widehat{\fl{V}}(\fl{V}_{\s{sub}}^{\s{T}}\fl{V}_{\s{sub}})^{-1/2}\fl{V}_{\s{sub}}^{\s{T}}
 \end{align*}
where $\widehat{\fl{U}}\widehat{\Sigma}\widehat{\fl{V}}$ is the SVD of the matrix $\f{\Sigma}(\fl{V}_{\s{sub}}^{\s{T}}\fl{V}_{\s{sub}})^{1/2}$.
Since $\widehat{\fl{U}}$ is orthogonal, $\widehat{\fl{U}}\fl{U}$ is orthogonal as well. Similarly, $(\widehat{\fl{V}}(\fl{V}_{\s{sub}}^{\s{T}}\fl{V}_{\s{sub}})^{-1/2}\fl{V}_{\s{sub}}^{\s{T}})^{\s{T}}$ is orthogonal as well whereas $\widehat{\f{\Sigma}}$ is diagonal. Hence $\fl{U}\widehat{\fl{U}}\widehat{\f{\Sigma}}\widehat{\fl{V}}(\fl{V}_{\s{sub}}^{\s{T}}\fl{V}_{\s{sub}})^{-1/2}\fl{V}_{\s{sub}}^{\s{T}}$ indeed corresponds to the SVD of $\fl{P}^{(1)}$ and we only need to argue about the incoherence of $\widehat{\fl{U}}\fl{U}$ and $\widehat{\fl{V}}(\fl{V}_{\s{sub}}^{\s{T}}\fl{V}_{\s{sub}})^{-1/2}\fl{V}_{\s{sub}}^{\s{T}}$. Notice that $\max_i \|(\fl{U}\widehat{\fl{U}})^{\s{T}}\fl{e}_i\|\le \|\fl{U}^{\s{T}}\fl{e}_i\|\le \sqrt{\frac{\mu r}{d_2}}$. On the other hand, 
\begin{align*}
    &\max_i \|\widehat{\fl{V}}(\fl{V}_{\s{sub}}^{\s{T}}\fl{V}_{\s{sub}})^{-1/2}\fl{V}_{\s{sub}}^{\s{T}}\fl{e}_i\| \le \max_i \|(\fl{V}_{\s{sub}}^{\s{T}}\fl{V}_{\s{sub}})^{-1/2}\fl{V}_{\s{sub}}^{\s{T}}\fl{e}_i\| \\
    &\le \frac{\|\fl{V}_{\s{sub}}\|_{2,\infty}}{\sqrt{\lambda_{\min}(\fl{V}_{\s{sub}}^{\s{T}}\fl{V}_{\s{sub}})}} \le \frac{\|\fl{V}\|_{2,\infty}}{\sqrt{p/2}} \le \sqrt{\frac{2\mu r}{d_2}}.
\end{align*}

Hence, with probability $1-2\exp(-d_2/12)-O(d_1^{-3})$, we can recover an estimate $\widehat{\fl{P}}^{(1)}$ and apply Lemma \ref{thm:randomsample2} to conclude that
\begin{align*}
   \| \fl{P}^{(1)} - \widehat{\fl{P}}^{(1)}\|_{\infty} = O\Big(\frac{\sigma }{\sqrt{d_2}}\Big(\frac{d_1}{d_2}\Big)^{1/2}\sqrt{\frac{\kappa^5\mu^3r^3 \log d_1}{p}}\Big) .
\end{align*}
as long as the conditions stated in the Lemma are satisfied.
Therefore, we can compute estimates of all the sub-matrices $\fl{P}^{(1)},\fl{P}^{(2)},\dots,\fl{P}^{(d_1/d_2)}$ that have similar guarantees as above with probability at least $1-2d_1d_2^{-1}\exp(-d_2/12)-O(d_1d_2^{-4})$. Hence, by combining all the estimates, we can obtain a final estimate $\widehat{\fl{P}}$ of the matrix $\fl{P}$ that satisfies
\begin{align*}
   \| \fl{P} - \widehat{\fl{P}}\|_{\infty} \le O\Big(\frac{\sigma }{\sqrt{d_2}}\Big(\frac{d_1}{d_2}\Big)^{1/2}\sqrt{\frac{\kappa^5\mu^3r^3 \log d_1}{p}}\Big) .
\end{align*}
\end{proof}

\begin{coro}\label{coro:obs3}
Consider an algorithm $\ca{A}$ that recommends random arms to users in each round (according to eq. \ref{eq:obs}) without recommending the same arm more than once to each user. Suppose the reward matrix $\fl{P}$ and parameters $p,\sigma$ satisfies the conditions stated in Lemma \ref{thm:randomsample3}.
In that case, using $m=O\Big(\s{M}p+\sqrt{\s{M}p\log \s{N}\delta^{-1}}\Big)$ recommendations per user, $\ca{A}$ is able to recover a matrix $\widehat{\fl{P}}$ such that for any $(i,j)\in [\s{N}]\times[\s{M}]$, we have equation \ref{eq:rectangular} with probability exceeding $ 1-\delta-O(d_2^{-3})$. 
\end{coro}

\begin{proof}[Proof of Corollary \ref{coro:obs3}]
Recall that $d_1=\max(\s{N},\s{M})$ and $d_2=\min(\s{N},\s{M})$.
Suppose we sample the elements in $[\s{N}]\times [\s{M}]$ each with some parameter $p=\Omega(\mu^2d_2^{-1}\log ^3 d_2)$ independently in order to sample a set $\Omega\in [\s{N}]\times [\s{M}]$ of indices. 
Let us define the event $\ca{F}_1$ which is true when the maximum number of indices observed in some row is more than $m=O\Big(\s{M}p+\sqrt{\s{M}p\log \s{N}\delta^{-1}}\Big)$ i.e. $\max_{i\in [\s{N}]}\left|(i,j)\in \Omega \mid j \in [\s{M}] \right| \ge m$. We will bound the probability of the event $\ca{F}_1$ from above by using Chernoff bound. Let us denote the number of arms observed for user $i\in [\s{N}]$ to be $Y_i$ i.e. $Y_i= \left|(i,j)\in \Omega \mid j \in [\s{M}]\right|$. Algorithm $\ca{A}$ can then obtain the noisy entries of $\fl{P}$ corresponding to the set $\Omega$ by doing the following: in each round, conditioned on the event that the user $i\in [\s{N}]$ is sampled, if there is an unobserved tuple of indices $(i,j)\in \Omega$, then $\ca{A}$ will recommend $j$ to user $i$ and obtain an noisy observation $\fl{P}_{ij}+\fl{E}_{ij}$; on the other hand, if there no unobserved entry, then $\ca{A}$ will simply recommend a random arm $j$ such that $(i,j) \not \in \Omega$. Notice that each of the random variables $Y_1,Y_2,\dots,Y_{\s{N}} \sim \s{Binomial}(\s{M},p)$ and are independent. By using Chernoff bound, we have that for each $i\in [\s{N}]$,
\begin{align*}
    &\Pr\Big(\cup_{i\in [\s{N}]}\left|Y_i-\s{M}p\right|\ge \s{M}p\epsilon\Big) \le 2\s{N}\exp\Big(-\frac{\epsilon^2\s{M}p}{3}\Big)\Big) \\
    &\implies Y_i \le \s{M}p+O\Big(\sqrt{\s{M}p\log \s{N}\delta^{-1}}\Big) \text{ for all }i\in[\s{N}] 
\end{align*}
with probability $1-\delta$ implying that $\Pr(\ca{F}_1)\le \delta$. Let $\ca{F}_2$ be the event when the recovered matrix $\widehat{\fl{P}}$ does not satisfy the guarantee on $\|\widehat{\fl{P}}-\fl{P}\|_{\infty}$ as stated in Lemma \ref{thm:randomsample3} equation \ref{eq:rectangular} given a set of observed indices $\Omega$ sampled according to the aforementioned process. From Lemma \ref{thm:randomsample3}, we know that $\Pr(\ca{F}_2)=O(d_2^{-3})$ where $d_1=\max(\s{N},\s{M})$. Hence we can conclude $\Pr(\ca{F}_1 \cup \ca{F}_2) \le  \Pr(\ca{F}_1)+\Pr(\ca{F}_2) = \delta+O(d_2^{-3})$. This completes the proof of the corollary. 
\end{proof}

\begin{rmk}\label{rmk:failure_prob}
Note that the failure probability in Corollary \ref{coro:obs3}, Lemmas \ref{thm:randomsample3}, \ref{thm:randomsample2} and \ref{lem:source2}  is $O(d_2^{-3})$. However the constant $3$ can be replaced by any arbitrary constant $c$ for example $c=100$ without any change in the guarantees on $\|\widehat{\fl{P}}-\fl{P}\|_{\infty}$. Hence, the guarantees presented in Lemma \ref{thm:randomsample3} and Corollary \ref{coro:obs3} hold with probability at least $1-O(d_2^{-c})$ for any arbitrary constant $c$.
\end{rmk}

\begin{lemma}[Generalized Restatement of Lemma \ref{lem:min_acc}]\label{lem:generalized_restatement}
Let $\fl{P}=\fl{\bar{U}}\f{\Sigma}\fl{\bar{V}}^{\s{T}}\in \bb{R}^{\s{N} \times \s{M}}$ such that $\fl{\bar{U}}\in \bb{R}^{\s{N}\times r},\fl{\bar{V}}\in \bb{R}^{\s{M} \times r}$ and $\f{\Sigma} \triangleq \s{diag}(\lambda_1,\lambda_2,\dots,\lambda_r) \in \bb{R}^{r \times r}$ with $\fl{\bar{U}}^{\s{T}}\fl{\bar{U}}=\fl{\bar{V}}^{\s{T}}\fl{\bar{V}}=\fl{I}$ and $\|\fl{\bar{U}}\|_{2,\infty}\le \sqrt{\mu r/\s{N}}, \|\fl{\bar{V}}\|_{2,\infty}\le \sqrt{\mu r/\s{M}}$. Let $d_1=\max(\s{N},\s{M})$ and $d_2=\min(\s{N},\s{M})$ such that 
 $1 \ge p \ge C  \kappa^4\mu^2 d_1d_2^{-2} \log^3 d_1$ for some sufficiently large constant $C>0$. 
In that case, for any positive integer $s>0$, there exists an algorithm $\ca{A}$ that uses $m=O\Big(s\log (\s{MN}\delta^{-1})(\s{M}p+\sqrt{\s{M}p\log \s{N}\delta^{-1}})\Big)$ recommendations per user such that $\ca{A}$ is able to recover a matrix $\widehat{\fl{P}}$ satisfying 
\begin{align}\label{eq:guar_gen}
   \| \fl{P} - \widehat{\fl{P}}\|_{\infty} = O\Big(\frac{\sigma }{\sqrt{d_2}}\Big(\frac{d_1}{d_2}\Big)^{1/2}\sqrt{\frac{\kappa^5\mu^3r^3 \log d_1}{sp}}\Big) .
\end{align}
with probability exceeding $ 1-O(\delta\log (\s{MN}\delta^{-1}))$ where $\frac{\sigma}{\sqrt{s}}=O\Big(\sqrt{\frac{p}{d\kappa^4\mu r\log d}}\min_i\lambda_i\Big)$.
\end{lemma}

\begin{proof}
Consider the proof of Lemma \ref{thm:randomsample3} where we sampled a set of indices $\Omega \in [\s{N}]\times [\s{M}]$ and observed $\fl{P}_{ij}+\fl{E}_{ij}$ corresponding to the indices $(i,j)\in\Omega$ where $\fl{E}_{ij}$'s are independent zero mean sub-gaussian random variables and have a variance proxy of $\sigma^2$ (along with other entries of the reward matrix $\fl{P}$). From Corollary \ref{coro:obs3}, we know that $m=O(\s{M}p+\sqrt{\s{M}p\log \s{N}\delta^{-1}})$ recommendations per user are sufficient to obtain the guarantees in Lemma \ref{thm:randomsample3} equation \ref{eq:rectangular} with probability $1-\delta$ if in each round $t$, we observe the reward corresponding to one arm for the sampled user $u(t)$ (see eq. \ref{eq:obs}). Thus, with $m$ recommendations per user, we observe the noisy entries of $\fl{P}$ corresponding to the indices in a superset $\Omega' \supseteq \Omega$.  

In our problem, an algorithm has the flexibility of recommending an arm more than once to the same user. Therefore, consider an algorithm $\ca{A}$ that uses $ms$ recommendations per user to obtain noisy observations corresponding to each index in $\Omega'$ $s$ times and uses the mean of observations
corresponding to each index. In that case, the algorithm $\ca{A}$ observes $\fl{P}_{ij}+\widetilde{\fl{E}}_{ij}$ for all $(i,j)\in \Omega' \supseteq \Omega$ where $\widetilde{\fl{E}}_{ij}$ are independent zero mean sub-gaussian random variables with zero mean and variance proxy $\sigma^2/s$. Hence the effective variance $\sigma^2/s$ should satisfy the upper bound on the noise variance implying that $\frac{\sigma}{\sqrt{s}}=O\Big(\sqrt{\frac{pd_2}{\mu^3\log d_2}}\|\fl{P}\|_{\infty}\Big)$.

At this point, algorithm $\ca{A}$ can use Lemma \ref{thm:randomsample3} to recover an estimate $\widehat{\fl{P}}$ of the matrix $\fl{P}$ satisfying the following: for any $(i,j)\in [\s{N}]\times[\s{M}]$, we have from equation \ref{eq:rectangular}
\begin{align}\label{eq:normalized}
   |\widehat{\fl{P}}_{ij} - \fl{P}_{ij}| \le O\Big(\frac{\sigma }{\sqrt{d_2}}\Big(\frac{d_1}{d_2}\Big)^{1/2}\sqrt{\frac{\kappa^5\mu^3r^3 \log d_1}{sp}}\Big).
\end{align}
with probability exceeding $1-\delta-O(d_2^{-c})$ for any arbitrary constant $c>0$. We can set $c$ such that the failure probability $\delta+O(d_2^{-c})<1/10$.

In order to boost the probability of success, the algorithm $\ca{A}$ can repeat the entire process $f=\log (\s{MN}\delta^{-1})$ times in order to obtain $f$ estimates $\widehat{\fl{P}}^{(1)},\widehat{\fl{P}}^{(2)},\dots,\widehat{\fl{P}}^{(f)}$. The total number of rounds needed to compute these estimates if we use the observation model in equation \ref{eq:obs} is at most $mf$ with probability at least $1-\delta f$. Furthermore, these estimates are independent and therefore, we can compute a final estimate $\widehat{\fl{P}}$ by computing the entry-wise median of the matrix estimates $\widehat{\fl{P}}^{(1)},\widehat{\fl{P}}^{(2)},\dots,\widehat{\fl{P}}^{(f)}$ i.e. for all $(i,j)\in [\s{N}]\times \s{M}$, we compute $\widehat{\fl{P}}_{ij}=\s{median}(\widehat{\fl{P}}^{(1)}_{ij},\dots,\widehat{\fl{P}}^{(f)}_{ij})$. Since, with probability $9/10$, each 
estimates satisfy the guarantee in \ref{eq:normalized}, we can again apply Chernoff bound that $\fl{P}$ satisfies the guarantee in \ref{eq:normalized} with probability $1-\delta/\s{MN}$. Now taking a union bound over all indices, we must have that 
\begin{align}
    \|\widehat{\fl{P}}-\fl{P}\|_{\infty} \le O\Big(\frac{\sigma }{\sqrt{d_2}}\Big(\frac{d_1}{d_2}\Big)^{1/2}\sqrt{\frac{\kappa^5\mu^3r^3 \log d_1}{sp}}\Big)
\end{align}
with probability at least $1-\delta$. 
Therefore the total failure probability is $1-O(\delta+\delta\log (\s{MN}\delta^{-1}))$. 

\end{proof}
We obtain the statement of Lemma \ref{lem:min_acc} by substituting $r=O(1),\kappa=O(1)$. Moreover, by another simple application of the Chernoff bound again, we can ensure with probability $1-\widetilde{O}(\delta)$ that the total number of recommendations (so that the number of recommendations per user is $m$ as described in Lemma \ref{lem:generalized_restatement}) is $\widetilde{O}(sp\s{MN})$ for which the guarantees in eq. \ref{eq:guar_gen} is satisfied. Now we characterize $m_{\ell}$ in this general setting:

\begin{lemma}\label{lem:interesting_gen}
Let us fix $\Delta_{\ell+1}>0$ and condition on the event $\ca{E}^{(\ell)}$. Suppose Assumptions \ref{assum:matrix} i
satisfied. In that case, in phase $\ell$, by using
 $$m_{\ell}=\widetilde{O}\Big(\frac{\sigma^2 \s{C}^2 (\s{C} \bigvee \mu\alpha^{-1})^3r^3 \kappa^5 \tau^{9/2} }{\Delta_{\ell+1}^2}\Big(\s{N}\bigvee\s{AC\tau}\Big)\Big)\Big)$$ rounds, we can compute an estimate $\widetilde{\fl{P}}^{(\ell)}\in \bb{R}^{\s{N}\times \s{M}}$ such that with probability $1-\delta$, we have 
 \begin{align}
 \lr{\widetilde{\fl{P}}^{(\ell)}_{\ca{M}^{(\ell,i)},\ca{N}^{(\ell,i)}}-\fl{P}_{\ca{M}^{(\ell,i)},\ca{N}^{(\ell,i)}}}_{\infty} \le \Delta_{\ell+1} \text{ for all }i\in [a_{\ell}] \text{ satisfying } \left|\ca{N}^{(i,\ell)}\right| \ge \gamma\s{C}.  
 \end{align}
 \end{lemma}
 
 \begin{proof}[Proof of Lemma \ref{lem:interesting_gen}]
We are going to use Lemma \ref{lem:min_acc} in order to compute an estimate $\widetilde{\fl{P}}^{(\ell)}_{\ca{M}^{(\ell,i)},\ca{N}^{(\ell,i)}}$ of the sub-matrix $\fl{P}_{\ca{M}^{(\ell,i)},\ca{N}^{(\ell,i)}}$ satisfying $\lr{\widetilde{\fl{P}}^{(\ell)}_{\ca{M}^{(\ell,i)},\ca{N}^{(\ell,i)}}-\fl{P}_{\ca{M}^{(\ell,i)},\ca{N}^{(\ell,i)}}}_{\infty} \le \Delta_{\ell+1}$. From Lemma \ref{lem:min_acc}, we know that by using $\widetilde{O}(sp\s{MN})$ rounds restricted to users in $\ca{M}^{(\ell,i)}$ such that with probability at least $1-\delta$,
\begin{align*}
    \lr{\widetilde{\fl{P}}^{(\ell)}_{\ca{M}^{(\ell,i)},\ca{N}^{(\ell,i)}}-\fl{P}_{\ca{M}^{(\ell,i)},\ca{N}^{(\ell,i)}}}_{\infty}  = O\Big(\frac{\sigma }{\sqrt{d_2}}\Big(\frac{d_1}{d_2}\Big)^{1/2}\sqrt{\frac{\widetilde{\kappa}^5\widetilde{\mu}^3r^3 \log d_1}{sp}}\Big) .
\end{align*}
where $d_2=\min(|\ca{M}^{(\ell,i)}|,|\ca{N}^{(\ell,i)}|)$ and $\widetilde{\mu},\widetilde{\kappa}$ is the incoherence factor and condition number of the matrix $\fl{P}_{\ca{M}^{(\ell,i)},\ca{N}^{(\ell,i)}}$. In order for the right hand side to be less than $\Delta_{\ell+1}$, we can set $sp=O\Big(\frac{\sigma^2 \widetilde{\kappa}^5\widetilde{\mu}^3r^3 \log d_1}{\Delta_{\ell+1}^2d_2}\Big)$. Since the event $\ca{E}^{(\ell)}$ is true, we must have that $\left|\ca{M}^{(\ell,i)}\right|\ge \s{N}/(\s{C}\tau)$; hence $d_2 \ge \min\Big(\frac{\s{N}}{\s{C}\tau},\left|\ca{N}^{(\ell,i)}\right|\Big)$. Therefore, we must have that 
\begin{align*}
    m_{\ell}=O\Big(\frac{\sigma^2 \s{C}^2\tau^2 \widetilde{\kappa}^5\widetilde{\mu}^3r^3 \log \s{M}}{\Delta_{\ell+1}^2}\max\Big(\s{N},\s{MC}\tau\Big)\log^2 (\s{MNC}\delta^{-1})\Big)\Big)
\end{align*}
where we take a union bound over all sets comprising the partition of the users $[\s{N}]$ (at most $\s{C}$ of them). Finally, from Lemma \ref{lem:incoherence}, we know that $\widetilde{\mu}$ can be bounded from above by $\max(\s{C},2\mu /\alpha)$; from lemma \ref{lem:condition_num}, we can say that $\widetilde{\kappa}$ can be bounded from above by $\kappa\sqrt{\tau}$ which we can use to say that 
\begin{align*}
    m_{\ell}=\widetilde{O}\Big(\frac{\sigma^2 \s{C}^2 (\s{C} \bigvee \mu\alpha^{-1})^3r^3 \kappa^5 \tau^{9/2} }{\Delta_{\ell+1}^2}\Big(\s{N}\bigvee\s{AC\tau}\Big)\Big)\Big)
\end{align*}
to complete the proof of the lemma.
\end{proof}

Plugging this $m_{\ell}$ in the proof of Theorem \ref{thm:main_LBM}, we can show that the Algorithm \ref{algo:phased_elim} guarantees the following regret
\begin{align}\label{eq:regret_gen}
    \s{Reg}(\s{T})  &= \widetilde{O}(\sqrt{\s{TV}}) +\sqrt{\s{BT}}\sigma)
\end{align}
where $\s{V}=\widetilde{O}\Big(\sigma^2 \s{C}^2 (\s{C}\bigvee\mu)^3 r^3 \kappa^5 \tau^{9/2} \Big(\s{N}+\s{M}\Big)\Big)$. This result follows by proceeding through the exact same steps as in the proof of Theorem \ref{thm:main_LBM}.

\end{document}